\pdfoutput=1

\documentclass[sigconf]{aamas} 
\usepackage{balance} % for balancing columns on the final page

\usepackage{algorithm, amsmath, amsthm, bm, bbm, booktabs, leftidx, lineno}
\usepackage[noend]{algpseudocode}
\usepackage[colorinlistoftodos]{todonotes}
\usepackage{tikz}
\usetikzlibrary{automata, angles}

\theoremstyle{plain}
\newtheorem{proposition}{Proposition}
\newtheorem{theorem}{Theorem}
\newtheorem{lemma}{Lemma}

\theoremstyle{definition}
\newtheorem{definition}{Definition}

\theoremstyle{remark}
\newtheorem{remark}{Remark}

\DeclareMathOperator*{\argmin}{argmin}
\DeclareMathOperator*{\argmax}{argmax}

\DeclareMathOperator*{\E}{\mathcal{E}}

\DeclareMathOperator*{\inff}{inf}
\DeclareMathOperator*{\finn}{fin}
\DeclareMathOperator*{\ltlU}{\sf U}
\DeclareMathOperator*{\ltlX}{\sf X}
\DeclareMathOperator*{\ltlF}{\sf F}
\DeclareMathOperator*{\ltlG}{\sf G}

\setcopyright{ifaamas}
\acmConference[AAMAS '21]{Proc.\@ of the 20th International Conference on Autonomous Agents and Multiagent Systems (AAMAS 2021)}{May 3--7, 2021}{Online}{U.~Endriss, A.~Now\'{e}, F.~Dignum, A.~Lomuscio (eds.)}
\copyrightyear{2021}
\acmYear{2021}
\acmDOI{}
\acmPrice{}
\acmISBN{}

\acmSubmissionID{325}

\title[AAMAS-2021 Formatting Instructions]{Multi-Agent Reinforcement Learning\\ with Temporal Logic Specifications}

\author{Lewis Hammond}
\affiliation{
\institution{University of Oxford}}
\email{lewis.hammond@cs.ox.ac.uk}

\author{Alessandro Abate}
\affiliation{
\institution{University of Oxford}}
\email{aabate@cs.ox.ac.uk}

\author{Julian Gutierrez}
\affiliation{
\institution{Monash University}}
\email{julian.gutierrez@monash.edu}

\author{Michael Wooldridge}
\affiliation{
 \institution{University of Oxford}}
\email{mjw@cs.ox.ac.uk}

\begin{abstract}
In this paper, we study the problem of learning to satisfy temporal logic specifications with a group of agents in an unknown environment, which may exhibit probabilistic behaviour. From a learning perspective these specifications provide a rich formal language with which to capture tasks or objectives, while from a logic and automated verification perspective the introduction of learning capabilities allows for practical applications in large, stochastic, unknown environments. The existing work in this area is, however, limited. Of the frameworks that consider full linear temporal logic or have correctness guarantees, all methods thus far consider only the case of a single temporal logic specification and a single agent. In order to overcome this limitation, we develop the first multi-agent reinforcement learning technique for temporal logic specifications, which is also novel in its ability to handle multiple specifications. We provide correctness and convergence guarantees for our main algorithm -- \textsc{Almanac} (Automaton/Logic Multi-Agent Natural Actor-Critic) -- even when using function approximation. Alongside our theoretical results, we further demonstrate the applicability of our technique via a set of preliminary experiments.
\end{abstract}


\newcommand{\BibTeX}{\rm B\kern-.05em{\sc i\kern-.025em b}\kern-.08em\TeX}

\begin{document}

%%% The following commands remove the headers in your paper. For final 
%%% papers, these will be inserted during the pagination process.

\pagestyle{fancy}
\fancyhead{}

%%% The next command prints the information defined in the preamble.

\maketitle 

%%%%%%%%%%%%%%%%%%%%%%%%%%%%%%%%%%%%%%%%%%%%%%%%%%%%%%%%%%%%%%%%%%%%%%%%

\section{Introduction}

Much recent work from the control and machine learning communities has considered the task of learning to satisfy temporal logic specifications in unknown environments \cite{Sadigh2014,Hasanbeig2019,Hahn2019,Bozkurt2019,Oura2020,Fu2014,Hasanbeig2020a,Li2018,Li2017,Wen2016,ToroIcarte2018,ToroIcarte2018a,Leon2020}. In these frameworks the agent is given a goal, typically specified using Linear Temporal Logic (LTL), and the dynamics of the agent’s environment are assumed to be captured by some unknown Markov Decision Process (MDP). The task of the agent is then to learn a policy that maximises the probability of satisfying the LTL specification. Importantly, the proposed Reinforcement Learning (RL) algorithms are \textit{model-free}, and so do not require evaluating the LTL specification against a model of the MDP (as is typically done in probabilistic model-checking, for instance), allowing for greater flexibility and scalability. These techniques have several advantages. From the perspective of RL, LTL forms an expressive and compact language with which to express infinite-horizon rewards that may be non-Markovian or exhibit special logical structure, and has provided a basis for new reward signal languages \cite{Jothimurugan2019, Littman2017}. From the perspective of logic, control and automated verification, the introduction of learning allows system designers to ensure that agents satisfy certain desirable properties in large, stochastic, unknown environments \cite{Hasanbeig2020a}.

The existing work in this area is, however, limited. Of the frameworks that consider full LTL or have correctness guarantees, all methods thus far consider only the case of a single specification and agent. Modern AI and control systems on the other hand are increasingly multi-agent and often multi-objective. 
Simply applying single-agent learning algorithms in a multi-agent setting can lead to poor performance and a lack of convergence \cite{Busoniu2008,Zinkevich2005}. Furthermore, even in the single-agent setting, no previous work has provided any correctness guarantees when using function approximation, which is crucial for scenarios that require both rigour \textit{and} scalability.

\subsection{Related Work}
\label{related}

Our contributions in this paper draw on many areas. The most closely related of these is a recent line of work investigating the problem of learning to satisfy temporal logic specifications in MDPs. These works can in turn be partitioned by whether they focus on full LTL or on a fragment of LTL. Within the former category, early approaches used model-based algorithms and encoded LTL specifications using Deterministic Rabin Automtata (DRAs) \cite{Fu2014,Sadigh2014}. To overcome scalability issues resulting from DRAs and model-based algorithms, later works employed model-free algorithms and Limit-Deterministic B{\"u}chi Automata (LDBAs) \cite{Hasanbeig2019,Hahn2019,Oura2020,Bozkurt2019} and in some cases function approximation \cite{Hasanbeig2020a}. However, these works only consider the case of a single specification and a single agent, and none have provided correctness guarantees when using function approximation. 

Other works have instead restricted their attention to \textit{fragments} of LTL \cite{Li2018,Li2017,Wen2016}. One strand of research uses `reward machines' (finite state transducers) to capture finite-horizon objectives to allow for a natural decomposition of tasks \cite{ToroIcarte2018}, and also for the introduction of multiple objectives \cite{ToroIcarte2018a}. Concurrently with this work, one recent effort has sought to generalise reward machines to the multi-agent and multi-objective case \cite{Leon2020}. However, this approach simply optimises the conjunction of all objectives via a single reward machine and independent Q-learning, which is well-known to suffer from convergence issues and sub-optimality in the multi-agent setting \cite{Busoniu2008}. Besides not supporting full LTL, the methods that use function approximation lack theoretical guarantees.

Similar problems to the one we tackle in this work have been considered by the verification community. Br\'{a}zdil et al. propose a Probably Approximate Correct (PAC) Q-learning algorithm for unbounded reachability properties in tabular settings with a single agent and single objective \cite{Brazdil2014}. Probabilistic or statistical model-checking algorithms have also been proposed for Markov Games (MGs), although so far these only handle known models and highly restricted forms of game, such as the turn-based two-player case \cite{Ashok2019}, or those that are composed of two coalitions of players and can thus be reduced to a two-player game \cite{Kwiatkowska2019}. 
Related paradigms such as \textit{rational verification} \cite{Wooldridge2016} and \textit{rational synthesis} \cite{Fisman2010} only consider non-stochastic games without learning agents.

Finally, our work can also be viewed in the context of the RL, Multi-Agent RL (MARL), and game theory literature \cite{Littman1994, Bowling2001,Conitzer2003,Wang2002,Prasad2015,Arslan2017}. The main algorithm we develop in this paper, 
\textsc{Almanac} (Automaton/Logic Multi-Agent Natural Actor-Critic), builds upon natural actor-critic algorithms \cite{Peters2008,Bhatnagar2009,Thomas2014} and generalises this to the multi-agent setting via the derivation of a multi-agent natural gradient. Multi-agent actor-critic algorithms enjoy state-of-the-art performance \cite{Lowe2017,Foerster2018} and have also been the focus of efforts to provide theoretical guarantees of convergence \cite{Zhang2018,Qu2020,Perolat2018}. We refer the reader to Zhang et al. for a recent survey of MARL \cite{Zhang2019} and to Now{\'e} et. al for a more game-theoretic perspective \cite{Nowe2012}. All of these works, however, use traditional scalar reward functions, whereas we focus on satisfying temporal logic formulae that provide a rich and rigorous language in which to express complex tasks and specifications over potentially infinite horizons.

\subsection{Contribution}

We overcome the limitations described above by proposing \textit{the first multi-agent reinforcement learning algorithm for temporal logic specifications} with correctness and convergence guarantees, even when using function approximation. Generalising from the single-objective, single-agent, non-approximate framework to the multi-objective, multi-agent, approximate setting is far from trivial and introduces several new challenges. 
We provide theoretical solutions to these challenges in the form of a new algorithm, \textsc{Almanac}, which provably converges to either locally or globally optimal joint policies with respect to multiple LTL specifications, depending on whether agents use local or global policies, respectively (the notions of local and global are made precise in later sections). We also evaluate our algorithm against ground-truth probabilities using PRISM, a state-of-the-art probabilistic model-checker \cite{Kwiatkowska2011}. 

We proceed as follows. 
In Section~\ref{preliminaries} we provide the requisite technical background on MARL and LTL and in Section~\ref{problem} we formalise our problem statement. We then introduce our full algorithm in Section~\ref{almanac} and report briefly on our experiments in Section~\ref{experiments}. 
Full proofs are relegated to Appendix \ref{proofs}, and we include examples and extra discussion in Appendix \ref{examples}.

\section{Preliminaries}
\label{preliminaries}

Unless otherwise indicated we use superscripts $i \in N$ to denote affiliation with a player $i$, or $j \in M$ to denote affiliation with a specification $\varphi^j$, and with subscripts $t \in \mathbb{N}$ to index variables through time. We denote true or optimal versions of functions or quantities using superscripts $^*$, and approximate versions using superscripts $~\hat{}$. 

\subsection{Multi-Agent Reinforcement Learning}
\label{marl}

Markov games (MGs), also known as (concurrent) stochastic games, are the \textit{lingua franca} of MARL, in much the same way that MDPs are for standard RL \cite{Littman1994}. In this setting the game proceeds, at each time step $t$, from a state $s_t$ by each player $i$ selecting an action $a_t^i$, after which a new state $s_{t+1}$ is reached and individual rewards $r_{t+1}^i$ are received. Formally, we have the following statement. 
\begin{definition}
    A (finite) \textbf{Markov Game (MG)} is a tuple $G = (N, S, A, T, \gamma, R)$ where $N = \{1, \ldots, n\}$ is a set of players, $S$ is a (finite) state space, $A = \{A^1, \ldots, A^n\}$ is a set of finite action spaces, $T : S \times A^1 \times \cdots \times A^n \times S \rightarrow [0,1]$ is a stochastic transition function, $\gamma \in (0,1)$ is an (optional) discount rate, and $R = \{R^1, \ldots, R^n\}$ is a set of reward functions defined as $R^i : S \times A^1 \times \cdots \times A^n \times S \rightarrow \mathbb{R}$. A (memoryless) \textbf{policy} $\pi^i : S \times A^i \rightarrow [0,1]$ maps states to a distribution over player $i$'s actions. If the range of $\pi^i$ is in fact $\{0,1\}$ then we say that $\pi$ is \textbf{deterministic}. A \textbf{joint policy} $\pi = (\pi^1, \ldots, \pi^n)$ is the combined policy of all players in $N$. We denote also by $\pi^{-i} = (\pi^1, \ldots, \pi^{i-1}, \pi^{i+1}, \ldots, \pi^n)$ the joint policy without player $i$. 
\end{definition}

Each player's objective in an MG is to maximise their cumulative discounted expected reward over time, given that the other players are playing some joint policy $\pi^{-i}$. Observe that given a starting state $s$ a joint policy $\pi$ induces a Markov chain $\Pr^\pi_G(\cdot \vert s)$ over the states of the MG. By taking the expectation over time of this Markov chain we define the \textit{value function} as $V^{i*}_{\pi}(s) \coloneqq \mathbb{E}_{\pi} [\sum^\infty_{t=0} \gamma^t r^i_{t+1}  ~\vert~  s ]$ where $R^i(s_t, a^1, \ldots, a^n, s_{t+1}) = r^i_{t+1}$. 
The core solution concept in MGs is that of a Markov Perfect Equilibrium (MPE) \cite{Maskin2001}. Informally, in the games we consider, an MPE is a set of memoryless strategies that forms a Nash Equilibrium when starting from any state.
\begin{definition}
    Consider an MG $G$. For each agent $i$ and joint policy $\pi^{-i}$, a policy $\pi^i$ is a \textbf{best response} to $\pi^{-i}$ if it is in the set $BR^i(\pi^{-i}) \coloneqq \{\pi^i : V^{i*}_{(\pi^i, \pi^{-i})}(s) = \max_{\bar{\pi}^i} V^{i*}_{(\bar{\pi}^i, \pi^{-i})}(s), \forall s \in S\}$. A joint policy $\pi = (\pi^1, \ldots, \pi^n)$ in $G$ is a \textbf{Markov Perfect Equilibrium (MPE)} if $\pi^i \in BR^i(\pi^{-i})$ for all $i \in N$. If, in addition, we have that $V^{i*}_{\pi}(s) = \max_{\bar{\pi}} V^{i*}_{\bar{\pi}}(s)$ for all $s \in S$ and for all $i \in N$, then we say that $\pi$ is \textbf{team-optimal}. If there exists a team-optimal joint policy in $G$, then we call $G$ a \textbf{common-interest game}.
\end{definition}
Intuitively, a common-interest game captures a setting in which there is a joint policy under which `everyone is happy'. In many games no such policy exists, and so a trade-off may be necessary. We may thus instead wish to maximise a weighted sum of rewards $V^*_\pi(s) = \sum_{i\in N} w[i] V^{i*}_{\pi}(s)$. In this general and popular setting (which generalises both team and common-interest games) that we adopt for the remainder of the paper we describe a joint policy $\pi$ as \textit{locally optimal} if it forms an MPE and \textit{globally optimal} if that MPE is team-optimal. Maximising a weighted sum of rewards means that there is always a \emph{deterministic} optimal joint policy \cite{Sutton2018} (as we can view a joint policy as a policy for a single agent in an MDP), and hence a deterministic MPE, in the games we consider.

\subsection{Linear Temporal Logic}

When defining specifications for a system (e.g., tasks for an agent), a natural idea is to introduce requirements on the possible traces that may arise as the system executes over time. LTL captures this idea and provides a logic for reasoning about the properties of such traces \cite{Pnueli1977}, which here we view as infinite paths $\rho$ through a state space $S$, where each $\rho[t] \in S$ for $t \in \mathbb{N}$ and $\rho[t..]$ denotes the path $\rho$ from time $t$ onwards. Additionally, we introduce a set of atomic propositions $AP$ and a labelling function $L : S \rightarrow \Sigma$ where $\Sigma = 2^{AP}$.
\begin{definition}
    The syntax of \textbf{Linear Temporal Logic (LTL)} formulae is defined recursively using the following operators:
    $$\varphi :\coloneqq \top \mid \alpha \mid \varphi \wedge \varphi \mid \neg \varphi \mid \ltlX \varphi \mid \varphi \ltlU \varphi$$
    where $\alpha \in AP$ is an atomic proposition and $\top$ is read as `true'. The semantics of said formulae are also defined recursively:
    \begin{equation*}
        \begin{aligned}
            \rho &\models \top\\
            \rho &\models \alpha &&\Leftrightarrow~ \alpha \in L(\rho[0])\\
            \rho &\models \psi_1 \wedge \psi_2 &&\Leftrightarrow~ \rho \models \psi_1 \text{ and } \rho \models \psi_2\\
            \rho &\models \neg \psi &&\Leftrightarrow~ \rho \not \models \psi\\
            \rho &\models \ltlX \psi &&\Leftrightarrow~ \rho[1..] \models \psi\\
            \rho &\models \psi_1 \ltlU \psi_2 &&\Leftrightarrow~ \exists t \in \mathbb{N} \text{ s.t. } \rho[t..] \models \psi_2, \forall t' \in [0, t), ~\rho[t'..] \models \psi_1 
        \end{aligned}
    \end{equation*}
\end{definition}
Alongside the standard operators from propositional logic, from which we may derive $\vee$, $\rightarrow$, and $\leftrightarrow$, we have the temporal operators $\ltlX$ (`next') and $\ltlU$ (`until'), from which we may derive $\ltlF \varphi \equiv \top \ltlU \varphi$ (`finally') and $\ltlG \varphi \equiv \neg \ltlF \neg \varphi$ (`globally'). An LTL formula $\varphi$ thus describes a set of infinite traces $\{ \rho \in S^\omega :\rho \models \varphi\}$ through $S$. Alternatively, one may encode such a set by using an \textit{automaton}.
\begin{definition}
    A \textbf{Non-deterministic B{\"u}chi Automaton (NBA)} is a tuple $B = (Q, q_0, \Sigma, F, \delta)$ where $Q$ is a finite set of states, $q_0 \in Q$ is the initial state, $\Sigma = 2^{AP}$ is a finite alphabet over a set of atomic propositions $AP$, $F \subseteq Q$ is a set of accepting states, and $\delta: Q \times \Sigma \rightarrow 2^Q$ is a (non-deterministic) transition function. We say that an \textbf{infinite word} $w \in \Sigma^\omega$ is \textbf{accepted} by $B$ if there exists an \textbf{infinite run} $\rho \in Q^\omega$ such that $\rho[0] = q_0$, $\rho[j+1] \in \delta(\rho[j], \omega[j])$ for all $j \in \mathbb{N}$, and we have $\inff(\rho) \cap F \neq \varnothing$, where $\inff(\rho)$ is the set of states in $Q$ that are visited infinitely often on run $\rho$. 
\end{definition}
In this work, we use a specific variant of NBAs, called Limit-Deterministic B{\"u}chi Automata (LDBAs). Intuitively, LDBAs relegate all non-determinism to a set of $\E$-transitions between two halves of the automaton, an initial component $Q_I$ and an accepting component $Q_A \supseteq F$. This level of non-determinism is, perhaps surprisingly, sufficient for encoding any LTL formula. We refer the reader to Sickert et al. for details of this LTL-to-LDBA conversion process \cite{Sickert2016}, which often yields smaller automata for formulas with deep nesting of modal operators compared to other approaches.
\begin{definition}
    A \textbf{Limit-Deterministic B{\"u}chi Automaton (LDBA)} is an NBA $B = (Q, q_0, \Sigma \cup \{\E\}, \delta, F)$ where $Q$ can be partitioned into two disjoint subsets $Q_I$ and $Q_A$ such that: $\vert\delta(q,\alpha)\vert=1$ for every $q \in Q$ and every $\alpha \in \Sigma$; $\delta(q,\E) = \varnothing$ for every $q \in Q_A$; $\delta(q, \alpha) \subseteq Q_A$ for every $q \in Q_A$ and every $\alpha \in \Sigma$; and $F \subseteq Q_A$.
\end{definition} 

\section{Problem Statement}
\label{problem}

We now combine MARL and LTL to consider the task of learning to satisfy temporal logic specifications with maximal probability in unknown multi-agent environments. The problem we seek to address in this work is: 
\begin{quote}
    Given an (unknown) environment with a team of $n$ agents characterised as an MG $G$, and a set of $m$ LTL specifications, compute (without first learning a model) a joint policy $\pi$ that maximises a weighted sum of the probabilities of satisfying each of the LTL specifications.
\end{quote}
To formalise this problem, we first define the satisfaction probability of $\pi$ in $G$ with respect to an LTL specification $\varphi$.
\begin{definition}
    \label{satprob}
    Given an MG $G$ and a joint policy $\pi$, denote by $\Pr^\pi_G(\cdot \vert s)$ the induced Markov chain over the states of $G$ starting from $s$. 
    Then, given an LTL formula $\varphi$, the \textbf{satisfaction probability} of $\pi$ in $G$ with respect to $\varphi$ starting from a state $s$ is given by $\Pr^\pi_G(s \models \varphi) \coloneqq \Pr^\pi_G(\{\rho : \rho \models \varphi\} ~\vert~ \rho[0] = s)$.
\end{definition}
Thus, our problem can be formally expressed as computing a policy $\pi^*$ in an unknown MG $G$, given a set of LTL specifications $\{\varphi^j\}_{0 \leq j \leq m}$ and vector of weights $w$ of length $m$, such that:
$$\pi^* \in \argmax_{\pi} \sum_j w[j] \Pr^{\pi}_G(s \models \varphi^j) \qquad \forall s \in S$$
This forms a natural extension of the single-agent single-objective case, in which one agent seeks to compute a policy that maximises the probability of satisfying a single LTL specification.

Our solution to this problem crucially relies on the definition of a \textit{product game} which, while never explicitly constructed, defines the full environment over which our agents learn. Note that in the following definition we consider an MG with a generic discount rate $\gamma$ and reward functions $R^\otimes$, though in our algorithm we redefine these to capture the original LTL specification, as in similar single-agent constructions \cite{Sickert2016}. The idea behind this construction is that by learning to act optimally in the (implicit) product game, agents learn to satisfy the LTL specification(s) in the original game.
\begin{definition}
    \label{prodmg}
    Given an LDBA $B = (Q, q_0, \Sigma \cup \{\E\}, \delta, F)$ associated with a set of agents $N^B \subseteq N$, where $Q = Q_I \cup Q_A$, a (finite) MG $G = (N, S, A, T, \gamma, R)$, and a labelling function $L : S \rightarrow \Sigma$, the resulting \textbf{Product MG} is a tuple $G \otimes B = G_B = (N, S^\otimes, A^\otimes, T^\otimes, \gamma, R^\otimes)$ where: 
    $S^\otimes = S \times Q$ is a product state space;
    $A^\otimes = \{A^1_\otimes, \ldots, A^n_\otimes\}$ where each $A^i_\otimes = A^i \cup \{\E_{q'} \vert \exists q \in Q_I, q' \in \delta(q, \E)\}$ for $i \in N^B$ and $A^i_\otimes = A^i$ otherwise;
    $T^\otimes: S^\otimes \times A^1_\otimes \times \cdots \times A^n_\otimes \times S^\otimes \rightarrow [0,1]$ is a stochastic transition function such that $T^\otimes((s,q), a^1,\ldots,a^n, (s',q')) =$
    \begin{align*}
        \begin{cases}
        T(s,a^1,\ldots,a^n,s') & \text{if } \forall i \in N^B, a^i \in A^i , q' \in \delta(q, L(s'))\\
        1         & \text{if } \exists i \in N^B \text{ s.t. } a^i = \E_{q'}, q' \in \delta(q, \E), s=s'\\
        0         & \text{otherwise}
    \end{cases}
    \end{align*}
    and $R^\otimes$ is a set of reward functions $\{R^1_\otimes, \ldots, R^n_\otimes\}$ such that $R^i_\otimes : S^\otimes \times A^1_\otimes \times \cdots \times A^n_\otimes \times S^\otimes \rightarrow \mathbb{R}$ for each $i \in N$.
    A (memoryless) \textbf{policy} $\pi^i : S^\otimes \times A^i_\otimes \rightarrow [0,1]$ for a player $i$ in the product MG is defined as before, using  $S^\otimes$ and $A^\otimes$. 
\end{definition}

We also extend the initial state distribution $\zeta$ to $\zeta^\otimes$ in the product game, where $\zeta^\otimes(s, q_0^1, \ldots, q_0^m) = \zeta(s)$ for all $s \in S$ and is equal to 0 for all other $s^\otimes \in S^\otimes$. We write $G \otimes B^1 \otimes \cdots \otimes B^m = G_{B^1,\ldots,B^m}$ for the product of $G$ with multiple automata $B^1,\ldots,B^m$, defined by sequentially taking individual products (as a product MG is simply another MG). In fact, given $G$, this operation can easily be seen to be associative (up to the ordering of elements forming a product state) if we assume that: $L^j(s,q^1,\ldots,q^{j-1}) = L^j(s) \subseteq \Sigma^j$ only depends on the state of $G$ for each labelling function $L^j$ (for automaton $B^j$); and that $\E$-transitions can be made for multiple automata at the same time step, i.e., there is no order in which $\E$-transitions are prioritised between groups $N^{B^j}$ when defining the new product transition function $T^\otimes$.
At each time step $t$ every agent in some set $N^{B^j}$ has the opportunity to make an $\E$-transition at which point their corresponding automaton state $q^j_t$ changes to $q^j_{t+1}$ with probability one and other elements of the product state remain the same. If no $\E$-transitions are made by any agent in any set $N^{B^j}$ then the transition probabilities are simply defined by the original transition function. Previous works have considered a similar multi-objective product construction, though only in the simpler case of a single agent \cite{Etessami2007}.

\section{Automaton/Logic Multi-Agent Natural Actor-Critic}
\label{almanac}

We now present our solution to the problem statement, in the form of our algorithm, \textsc{Almanac} (Automaton/Logic Multi-Agent Natural Actor-Critic). \textsc{Almanac} falls into a category of model-free RL algorithms known as \textit{actor-critic} methods \cite{Konda2000,Peters2008,Bhatnagar2009}, whereby a policy $\pi$ (the actor) is optimised via gradient descent using the value function $V_\pi$ (the critic) which is updated via bootstrapping. These two functions are typically learnt separately and simultaneously using a two-timescale approach in which the critic is updated faster than the actor in order to learn the value function with respect to the current policy. Such methods form a powerful, flexible, and highly scalable class of algorithms which can be applied to a wide range of (MA)RL problems and regularly achieve state-of-the-art performance \cite{Lowe2017,Foerster2018}. We begin by introducing a novel temporal difference (TD) algorithm with state-dependent discounts such that the learnt critics capture the LTL specifications. We then combine this with a natural policy gradient scheme to update the actors, forming our full algorithm. In the final subsection we provide proof sketches of correctness and convergence, with full proofs available in Appendix \ref{proofs}.

\subsection{Patient Temporal Difference Learning}
\label{ptd}

We wish to solve the problem of learning a critic $V_\pi$ given any fixed joint policy $\pi$ such that for any $s \in S$:
\begin{equation}
    \label{eq1}
    \pi^* \in \argmax_{\pi} V^*_{\pi}(s^\otimes) \Rightarrow \pi^* \in \argmax_{\pi} \sum_j w[j] \Pr^{\pi}_G(s \models \varphi^j), 
\end{equation}
where $s^\otimes = (s, q^1_0, \ldots, q^m_0)$.
The main idea is that by defining a new reward function $R^\otimes$ and discount rate $\Gamma$ we can learn a value function $V_\pi$ in the product MG $G_{B^1,\ldots,B^m}$ (where $B^j$ is the LDBA corresponding to $\varphi^j$) such that any policy $\pi$ that maximises $V^*_\pi$ in $G_{B^1,\ldots,B^m}$ is guaranteed to maximise $\sum_j w[j] \Pr^{\pi}_G(s \models \varphi^j)$ when projected down into the original game $G$. In this way, the states of the automata $B^1,\ldots,B^m$ 
can be thought of as a finite memory for $\pi$ in the original game.

\begin{remark}
    An MPE in our setting is simply a Subgame Perfect Equilibrium (SPE) in which all players use memoryless strategies, where a subgame in an MG is defined by a starting state \cite{Fudenberg1991}. If (\ref{eq1}) holds, then any joint policy $\pi^* \in \argmax_{\pi} V^*_{\pi}(s^\otimes)$ forms an MPE in the product game, but when viewed in terms of the original game, a policy $\pi^* \in \argmax_{\pi} \sum_j w[j] \Pr^{\pi}_G(s \models \varphi^j)$ for all $s$ is merely an SPE, as the policies of each agent are no longer memoryless.
\end{remark}

The problem defining $R^\otimes$ and $\Gamma$ such that the limit $V^*_\pi$ of the learnt value function $V_\pi$ satisfies (\ref{eq1}) is trickier than it might initially seem. Previous approaches for MDPs have either been open to counterexamples in which an agent learns to prioritise the length of the path taken to satisfy $\varphi$ over the probability of satisfying it \cite{Hasanbeig2019}, or involved constructions that hinder learning by increasing the state-space size \cite{Oura2020}, increasing reward sparsity \cite{Hahn2019}, or increasing learning rates \cite{Bozkurt2019}. We propose a novel solution that is far simpler and more natural. Given a state $s^\otimes = (s, q^1, \ldots q^m)$ the basic idea is, for each automaton, to issue a reward when $q^j \in F^j$ and to use a \textit{state-dependent} discount factor which is equal to $1$ when no reward is seen and equal to a constant $\gamma_V \in (0,1)$ otherwise. 
Formally, for each specification $\varphi^j$ we define $R^j_\otimes$ and $\Gamma^j$ as follows:
\begin{align}
    \label{eq2}
{R^j_\otimes(s^\otimes) \coloneqq 
\begin{cases}
                1 &\text{ if } q^j \in F^j\\
                0 &\text{ otherwise ,}
    \end{cases}
}
&&
{\Gamma^j(s^\otimes) \coloneqq \begin{cases}
    \gamma_V &\text{ if } R^j_\otimes(s^\otimes) = 1\\
    1 &\text{ otherwise .}
    \end{cases}}
\end{align}
We then define $V^{j*}_{\pi}(s^\otimes) \coloneqq \mathbb{E}_{\pi} \big[ \sum^\infty_{t=0} \Gamma^j_{1:t} R^j_\otimes(s^\otimes_{t+1}) ~\big\vert~ s^\otimes \big]$ where $\Gamma^j_{1:t} = \prod^{t}_{\tau=1} \Gamma^j(s^\otimes_\tau)$ for $t \geq 1$ and and $\Gamma^j_{1:0} = 1$. While earlier works have considered a similar solution in MDPs \cite{Hasanbeig2019, Hahn2020}, they fail to mention that the problem with this scheme is that if used to update a value function naively (such as in the vanilla Q-learning algorithms these works make use of) the update process can converge to the wrong values. This lack of convergence arises because of possible loops in the product MG that do not contain any rewarding states. To overcome this limitation we use a \emph{patient TD} scheme whereby agents update their estimates of the value function only once a reward is seen (or when the next state has value 0), meaning that value estimates for states on such loops cannot be artificially inflated. 

We begin by considering the standard TD(0) update rule with learning rate $\alpha = \{\alpha_t\}_{t \in \mathbb{N}}$ and fixed policy $\pi$ given by \cite{Sutton1988}:
$$V^j_\pi(s^\otimes_t) \leftarrow (1-\alpha_t)V^j_\pi(s^\otimes_t) + \alpha_t \big[R^j_\otimes(s^\otimes_{t+1}) + \Gamma^j_{t+1:t+1} V^j_\pi(s^\otimes_{t+1}) \big],$$
where $R^j_\otimes(s^\otimes_{t+1}) + \Gamma(s^\otimes_{t+1}) V^j_\pi(s^\otimes_{t+1}) =: G^j_{t:t+1}$ is a \textit{one-step target}, though one may also use a \textit{k-step target} instead, given by:
\begin{align*}
    G^j_{t:t+k} &\coloneqq R^j_\otimes(s^\otimes_{t+1}) + 
\Gamma^j_{t+1:t+1} R^j_\otimes(s^\otimes_{t+2}) + 
\cdots\\
&+ \Gamma^j_{t+1:t+k-1} R^j_\otimes(s^\otimes_{t+k}) + 
\Gamma^j_{t+1:t+k} V^j_\pi(s^\otimes_{t+k}) .
\end{align*}
It is well-known that using longer trajectories as targets can improve bootstrapping as much more can be learnt from a single episode \cite{Sutton2018}. Our motivation is different: by not immediately updating $V_\pi(s^\otimes_t)$ when we either do not see a reward, or when the value of the successor state $V_\pi(s^\otimes_{t+1})$ is non-zero, then we avoid increasing the values of zero-reward loop states in terms of themselves. Instead, we set $k$ `on the fly' to be the least $k$ such that either $R^j_\otimes(s^\otimes_{t+k}) > 0$ or $V^j_\pi(s^\otimes_{t+k}) = 0$, i.e.,we wait to update $V^j_\pi(s^\otimes_t)$ until we see a reward. Observe that for any $0 < l < k$ we have $R^j_\otimes(s^\otimes_{t+l}) = 0$, $\Gamma^j(s^\otimes_{t+l}) = 1$, and $\Gamma^j(s^\otimes_{t+k}) = \gamma_V$, hence:
$$G_{t:t+k} = R^j_\otimes(s^\otimes_{t+k}) + \Gamma^j(s^\otimes_{t+k})V^j_\pi(s^\otimes_{t+k}) = R^j_\otimes(s^\otimes_{t+k}) + \gamma_V V^j_\pi(s^\otimes_{t+k}).$$ 
\textsc{Almanac} implements an (approximate) patient TD scheme to learn each value function $V^j_\pi$ in the product MG under a joint policy $\pi$ by maintaining a temporary set of \textit{zero-reward states} $\{s^\otimes_t, \ldots, s^\otimes_{t+k-1}\}$ (with respect to $R^j_\otimes$) whose values it waits to update. The convergence of this update rule is proven in Theorem \ref{opt}. What remains to show here is that the resulting $V_\pi \coloneqq \sum_j w[j] V^j_\pi$ satisfies (\ref{eq1}), which gives us a critic for correctly \textit{capturing} the given LTL specifications. In the next subsection we show how this critic can be learnt synchronously alongside an actor (i.e., a joint policy) for optimally \textit{satisfying} said specifications.
\begin{proposition}
    \label{Scalar2LTL}
    Given an MG $G$ and LTL objectives $\{\varphi^j\}_{1 \leq j \leq m}$ (each equivalent to an LDBA $B^j$), let $G_B = G \otimes B^1 \otimes \cdots \otimes B^m$ be the resulting product MG with newly defined reward functions $R^j_\otimes$ and state-dependent discount functions $\Gamma^j$ given by (\ref{eq2}). Then there exists some $0 < \gamma_V < 1$ such that (\ref{eq1}) is satisfied by the patient value function $V_\pi \coloneqq \sum_j w[j] V^j_\pi$.
\end{proposition}
\begin{proof}[Proof (Sketch)]
    We begin by observing that:
    \begin{align*}
        V^{*}_{\pi}(s^\otimes) = \sum_j w[j] \sum_\rho \Big[Pr^\pi_{G_B}(\rho \vert s^\otimes) \sum^\infty_{t=0} \Gamma^j_{1:t} R^j_\otimes(s^\otimes_{t+1}) \Big] .
    \end{align*} 
    We denote the number of times a path $\rho$ in $G_B$ passes through the accepting set $F^j$ of automaton $B^j$ by $F^j(\rho)$, and let $F(\rho)= \sum_j F^j(\rho)$. Then when $F^j(\rho) = \infty$ we have that $\sum^\infty_{t=0} \Gamma^j_{1:t} R^j_\otimes(s^\otimes_{t+1}) = \frac{1}{1 - \gamma_V}$ and when $F^j(\rho) = f^j < \infty$  we have that $\sum^\infty_{t=0} \Gamma^j_{1:t} R^j_\otimes(s^\otimes_{t+1}) = (1 - \gamma_V^{f^j})\frac{1}{1 - \gamma_V}$. We show that if $\pi \notin \argmax_{\pi} \sum_j w[j] \Pr^{\pi}_G(s \models \varphi^j)$ then there exists $0 < \gamma_V < 1$ such that $\pi \notin \argmax_{\pi}V^*_{\pi}(s^\otimes)$ and for some $\pi' \in \argmax_{\pi} \sum_j w[j] \Pr^{\pi}_G(s \models \varphi^j)$, we have $V^*_{\pi'}(s^\otimes) > V^*_{\pi}(s^\otimes)$. Define the sets $\finn^j(s^\otimes) \coloneqq \{\rho : \rho[0] = s^\otimes \wedge F^j(\rho) \neq \infty\}$ and $\inff^j(s^\otimes) \coloneqq \{\rho : \rho[0] = s^\otimes \wedge F^j(\rho) = \infty\}$. Then we have:
    \begin{align*}
        V^*_{\pi'}(s^\otimes)
        \geq \sum_j w[j] \bigg[\sum_{\rho \in \inff^j(s^\otimes)} Pr^{\pi'}_{G_B}(\rho \vert s^\otimes) \frac{1}{1- \gamma_V} \bigg]
        = \frac{a^\top w}{1- \gamma_V},
    \end{align*}
    where $a$ is a vector such that $a[j] = \sum_{\rho \in \inff^j(s^\otimes)} Pr^{\pi'}_{G_B}(\rho \vert s^\otimes) = Pr^{\pi'}_{G_B}\big(\inff^j(s^\otimes)\big)$. Similarly, we have:
    \begin{align*}
        V^*_{\pi}(s^\otimes) \leq \frac{b^\top w}{1- \gamma_V} + (1 - b)^\top w \frac{1 - \gamma_V^{f}}{1 - \gamma_V},
    \end{align*}
    where $b[j] = \sum_{\rho \in \inff^j(s^\otimes)} Pr^{\pi}_{G_B}(\rho \vert s^\otimes) = Pr^{\pi}_{G_B}(\inff^j(s^\otimes))$ and $f = \max_j \max_{\rho \in \finn^j(s^\otimes)}F^j(\rho)$. Given that $\sum_j w[j] \Pr^{\pi'}_G(s \models \varphi^j) > \sum_j w[j] \Pr^{\pi}_G(s \models \varphi^j)$ by assumption, then by a straightforward extension of the result (see Appendix \ref{proofs} and \cite{Sickert2016}) that there exists a canonical extension of $\pi$ to $G_B$ such that $\Pr^\pi_G(s \models \varphi) = \Pr^{\pi}_{G_B}(\{\rho : F(\rho) = \infty\} \vert s^\otimes)$, we have $1 \geq a^\top w >  b^\top w \geq 0$. The proof is concluded by setting $\gamma_V > \sqrt[\uproot{5}f]{\frac{1-a^\top w}{1-b^\top w}}$.
\end{proof}

Finally, we remark that  as well as a `patient' value function $V$, we also learn a (standard) `hasty' value function $U$. This is because, if for some specification $\varphi$ and two possible joint policies $\pi$ and $\pi'$, we have that $\Pr^{\pi}_G(s \models \varphi) = \Pr^{\pi'}_G(s \models \varphi)$ then agents have no reason to use $\pi$ over $\pi'$, even if $\pi$ results in a much more efficient trajectory. In many ways this is a feature and not a bug; LTL has no emphasis on `hastiness' by design. In reality, however, whilst we may not want to forsake the satisfaction of a constraint for the sake of speed, we would like the agents to find the most efficient policy that maximises the probability of satisfying the constraint. We solve this problem by learning two value functions and then using \emph{lexicographic RL} \cite{Skalse2020a,Rentmeestersa1996} to maximise the hasty objective subject to maximising the patient objective, and thus satisfying (\ref{eq1}). Further details are provided in the following section.

\subsection{Multi-Agent Natural Actor-Critic}

For the remainder of the paper, we assume that each agent's policy $\pi^i$ is parameterised by $\theta^i \in \Theta^i$, potentially using a \textit{non-linear} function approximator giving $\pi^i(a\vert s;\theta^i)$, and that the value functions are linearly approximated using some state basis functions $\phi(s^\otimes)$ and parameters $v$ and $u$, such that $\hat{V}(s^\otimes) = \phi(s^\otimes)^\top v$ and $v = \sum_j w[j] v^j$ (and likewise for $U$). Note that these assumptions subsume the tabular setting which most prior work on RL with LTL specifications has focused on. We use $\theta = [{\theta^1}^\top, \ldots, {\theta^n}^\top]^\top$ to denote the joint set of parameters for all agents and write joint actions as $a = (a^1, \ldots, a^n)$. We also replace $\pi$ by $\theta$ in our notation for clarity where appropriate.

When formulated in terms of our parametrisation, and given that (\ref{eq1}) holds, our task can be viewed as optimising the objective function $J(\theta) \coloneqq \sum_{s^\otimes} \zeta^\otimes(s^\otimes) V_\theta(s^\otimes) = \sum_j w[j] J^j(\theta)$
where $J^j(\theta) \coloneqq \sum_{s^\otimes} \zeta^\otimes(s^\otimes) V^j_\theta(s^\otimes)$ is the objective for specification $\varphi^j$. Then, within $\argmax_\theta J(\theta)$, we also wish to select the parameters that maximise the objective function with respect to our `hasty' value function, which is given by $K(\theta) \coloneqq \sum_{s^\otimes} \zeta^\otimes(s^\otimes) U_\theta(s^\otimes)= \sum_j w[j] K^j(\theta)$.
The following derivations are provided for $J$, the case for $K$ is analogous. We can improve $\theta$ with respect to an objective function $J$ using the gradient $\nabla_\theta J(\theta)$ which, by the policy gradient theorem \cite{Sutton1999}, is:
\begin{align*}
    \nabla_\theta J(\theta) 
    &= \sum_j w[j] \sum_{s^\otimes} d^j_{\theta,\zeta}(s^\otimes) \sum_a Q^j_\theta(s^\otimes, a) \nabla_\theta \pi(a \vert s^\otimes; \theta)
\end{align*}
where $d^j_{\theta,\zeta}(s^\otimes)$ is the (patient) discounted state distribution in the product game for specification $\varphi^j$, given initial distribution $\zeta^\otimes$:
$$d^j_{\theta,\zeta}(s^\otimes) \coloneqq \mathbb{E}_{s_0^\otimes \sim \zeta^\otimes} \left[ \sum_{\rho} \Pr^\theta_{G_B} (\rho \vert s_0^\otimes) \Bigg( \frac{1}{\sum_{t=0}^\infty \Gamma^j_{0:t} } \sum_{t=0}^\infty \Gamma^j_{0:t} \mathbb{I}(\rho[t] = s^\otimes) \Bigg) \right]$$
where $\mathbb{I}$ denotes an indicator function. Where unambiguous we write simply $d^j$ instead of $d^j_{\theta,\zeta}$.
\begin{remark}
    \label{unbiased}
    Producing unbiased samples with respect to each $d^j(s^\otimes)$ in order to estimate the gradients used in policy evaluation and improvement raises several difficulties \cite{Nota2020,Thomas2014}. It is possible, however, to instead use trajectories from the \emph{undiscounted} MG distribution $d(s^\otimes)$ that are truncated after each transition to a state $s^\otimes$ with probability $1 - \Gamma^j(s^\otimes)$ \cite{Bertsekas1996}, or to re-weight updates as a function of $\Gamma^j(s^\otimes)$ \cite{Thomas2014}. We combine these two approaches in \textsc{Almanac}.
\end{remark}
A known problem with `vanilla' gradients is that they can sometimes be inefficient due to large plateaus in the optimisation space, leading to small gradients and thus incremental updates. A solution to this problem is instead to use the \textit{natural} policy gradient which is invariant to the parametrisation of the policy, and can be computed by applying the inverse Fisher matrix to the vanilla gradient \cite{Amari1998, Kakade2001}. 
For each specification $j$, the natural gradient of $J^j(\theta)$ can be shown \cite{Peters2008} to equal
$\tilde{\nabla}_\theta J^j(\theta) = G^j(\theta)^{-1} \nabla_\theta J^j(\theta) = x_{V^j}$,
where $G^j(\theta)$ is the Fisher information matrix and $x_{V^j}$ satisfies:
$$\psi_{\theta}(a\vert s^\otimes)^\top x_{V^j} = Q^j_\theta(s^\otimes, a) - V^j_\theta(s^\otimes) =: A^j_\theta(s^\otimes, a).$$
Here, $A^j_\theta$ denotes the advantage function for specification $j$ (for the hasty advantage function we use $Z^j_\theta$) and $\psi_{\theta}(a\vert s^\otimes) = \nabla_\theta \log\pi(a\vert s^\otimes;\theta)$ denotes the score function. 
Using a similar line of reasoning a derivation of a natural policy gradient for the multi-agent case is a simple exercise (omitted here due to space constraints).
\begin{lemma}
    Let $G_B$ be some (product) MG. Then for any set of parameters $\{\theta^i\}_{i \in N}$ and any player $i \in N$, the natural policy gradient for player $i$ with respect to each $J^j(\theta) = \sum_{s^\otimes} \zeta^\otimes(s^\otimes) V^j_\theta(s^\otimes)$ is given by $\tilde{\nabla}_{\theta^i} J^j(\theta) = x^i_{V^j}$, where $x^i_{V^j}$ is a parameter satisfying $\psi^i_{\theta^i}(a^i\vert s^\otimes)^\top x^i_{V^j} = \nabla_{\theta^i} \log\pi^i(a^i\vert s^\otimes;\theta)^\top x^i_{V^j} = A^j_\theta(s^\otimes, a)$. 
\end{lemma}
This implies that the gradient $x^i_V$ with respect to our weighted combination of objectives can be found by minimising the loss $L^i_V(x^i_V;\theta,\nu_{\theta,\zeta}) \coloneqq
    \sum_j w[j] L^i_{V^j}(x^i_V;\theta,\nu^j_{\theta,\zeta})$
where:
\begin{align*}
    L^i_{V^j}(x^i_V;\theta,\nu^j_{\theta,\zeta})
    \coloneqq \mathbb{E}_{(s^\otimes,a) \sim \nu^j_{\theta,\zeta}} \Big[\big\vert\psi^i_{\theta^i}(a^i\vert s^\otimes)^\top x^i_{V} - A^j_\theta(s^\otimes, a)\big\vert\Big]
\end{align*}
and $\nu^j_{\theta,\zeta}(s^\otimes,a) \coloneqq d^j_{\theta,\zeta}(s^\otimes)\pi(a\vert s^\otimes;\theta)$. Similarly we may define $\mu_{\theta,\zeta}(s^\otimes,a) = c_{\theta,\zeta}(s^\otimes)\pi(a\vert s^\otimes;\theta)$ as the hasty state-action distribution under a joint policy $\theta$ and with initial distribution $\zeta^\otimes$, where:
$$c_{\theta,\zeta}(s^\otimes) \coloneqq \mathbb{E}_{s_0^\otimes \sim \zeta^\otimes} \left[ \sum_{\rho} \Pr^\theta_{G_B} (\rho \vert s_0^\otimes) \Bigg( \frac{1}{1 - \gamma_U} \sum_{t=0}^\infty \gamma^t_U \mathbb{I}(\rho[t] = s^\otimes) \Bigg) \right],$$
which does not depend on the specification $\varphi^j$ because of the constant discount rate. As with $d^j$ we drop subscripts for $c$, $\nu^j$, and $\mu$ where unambiguous. Recall that our secondary objective is to optimise $\theta^i$ according to $K(\theta)$, given our lexicographic prioritisation of $J(\theta)$. In other words, we wish to follow the hasty natural gradient $x^i_U$ subject to following the patient natural gradient $x^i_V$. Formally, the gradient we seek is given by:
\begin{equation}
    \label{natgrad}
    x^i_* \in \argmin\hspace{0.01em}_{x^i_U \in \argmin_{x^i_V} L^i_V(x^i_V;\theta,\nu)}~L^i_U(x^i_U;\theta,\mu),
\end{equation}
where $L^i_U$ is defined analogously to $L^i_V$. Note that both $L^i_V$ and $L^i_U$ are convex, and so we can find some $x^i_*$ satisfying (\ref{natgrad}) by simply first following $\nabla_{x^i}L^i_V(x^i;\theta,\nu)$ until this gradient is zero, and then following $\nabla_{x^i}L^i_U(x^i;\theta,\mu)$ subject to the constraint that $\nabla_{x^i}L^i_V(x^i;\theta,\nu) = 0$.
We use a multi-timescale \emph{lexicographic} approach to perform this operation simultaneously and compute $x^i_*$, which we then use to update $\theta^i$. More specifically, we minimise $L^i_V(x^i;\theta,\nu)$ on a faster timescale and so guarantee its convergence to some $l^i$ before $L^i_U(x^i;\theta,\mu)$ has converged. On a slower timescale we solve the Lagrangian dual corresponding to the constrained optimisation problem of minimising $L^i_U(x^i;\theta,\mu)$ such that $L^i_V(x^i;\theta,\nu) - l^i \leq 0$:
\begin{align}
    \max_{\lambda^i \geq 0} \min_{x^i} ~~~~~ L^i_U(x^i;\theta,\mu) &+ \lambda^i \big[L^i_V(x^i;\theta,\nu) - l^i\big].
\end{align}
To form the gradients of $L^i_V$ and $L^i_U$ we use an unbiased estimate of each $A^j_\theta$ and $Z^j_\theta$ using samples of the TD error \cite{Bhatnagar2009} which can be trivially extended to the $k$-step version $\delta_{t:t+k}^{V^j} = G^{V^j}_{t:t+k} - V^j_\theta(s^\otimes_t)$ \cite{Sutton2018}. We compute $v$ by minimising the following loss for each $v^j$ (the case for the $u$ is analogous):
$$L^j_v(v^j;\theta,d^j) = \mathbb{E}_{s^\otimes \sim d^j}\Big[\big(V^j_\theta(s^\otimes_t) - \phi(s^\otimes)^\top v^j\big)^2\Big].$$
This can again be solved via gradient updates that use the temporal difference $\delta^{V^j}_{t:t+k}$, corresponding to the linear semi-gradient temporal difference algorithm \cite{Sutton2018}: 
\begin{align}
    \begin{aligned}
    \hat{\delta}^{V^j}_{t:t+k} &\leftarrow r^j_{t+k} + \gamma_V \phi(s^\otimes_{t+k})^\top v^j - \phi(s^\otimes_t)^\top v^j\\
    v^j &\leftarrow v^j + \alpha_t \hat{\delta}^{V^j}_{t:t+k} \phi(s_t^\otimes) 
    \label{eq:patient_update}
    \end{aligned}\\
    \begin{aligned}
    \hat{\delta}^{U^j}_{t:t+1} &\leftarrow r^j_t + \gamma_U \phi(s^\otimes_{t+1})^\top u^j - \phi(s^\otimes_t)^\top u^j\\
    u^j &\leftarrow u^j + \alpha_t \hat{\delta}^{U^j}_{t:t+1} \phi(s_t^\otimes).
    \label{eq:hasty_update}
    \end{aligned}
\end{align}
Using these quantities we then update the natural gradient $x^i$ and Lagrange multiplier $\lambda^i$:
\begin{align}
    \begin{aligned}
    x^i &\gets \Omega_{x^i} \bigg[ x^i + \Big( \sum_j w[j] (\beta^V_t + \beta^U_t \lambda^i) \chi^{V^j}_t +  \beta^U_t \chi^{U^j}_t \Big) \psi^i_{\theta^i} (a^i_t \vert s^\otimes_t) \bigg],\\
    \lambda^i &\gets \Omega_\lambda \bigg[\lambda^i + \eta_t \Big( \sum_j w[j] \Gamma^j_{1:t} \big\vert\psi^i_{\theta^i}(a^i_t\vert s_t^\otimes)^\top x^i_{V} - \hat{\delta}^{V^j}_{t:t+1}\big\vert - l^i \Big) \bigg],\\
    \end{aligned}
    \label{eq:nat_grad_updates}
\end{align}
where $\chi^{V^j}_t \coloneqq \Gamma^j_{1:t} \text{sgn}\big( \hat{\delta}^{V^j}_{t:t+1} - \psi^i_{\theta^i}(a^i_t\vert s^\otimes_t)^\top x^i\big)$ is used to form an estimate of $- \nabla_{x^i} L^i_{V^j}(x^i;\theta,\nu^j_{\theta,\zeta})$ in a piecewise fashion. Finally we update the policy parameters using:
\begin{align}
    \theta^i \gets \Omega_{\theta^i} \big[ \theta^i + \iota_t x^i \big].
    \label{eq:policy_updates}
\end{align}
The functions $\Omega_{x^i}, \Omega_\lambda, \Omega_{\theta^i}$ are projections  onto $X^i$, $[0,\infty)$, and $\Theta^i$ respectively, which are commonly used within stochastic approximation to ensure boundedness of iterates, and is also standard in the literature on natural actor-critic algorithms \cite{Bhatnagar2009, Borkar2008}. We assume that learning rates $\alpha, \beta^V, \beta^U, \eta, \iota$ obey the following relationship:
\begin{align}
\begin{aligned}
    &\sum^\infty_{t=0} \alpha_t = \sum^\infty_{t=0} \beta^V_t = \sum^\infty_{t=0} \beta^V_t = 
    \sum^\infty_{t=0} \eta_t = \sum^\infty_{t=0} \iota_t = \infty,\\
    &\sum^\infty_{t=0} \Big[ (\alpha_t)^2 + (\beta^V_t)^2 + (\beta^U_t)^2 + (\eta_t)^2  + (\iota_t)^2 \Big] < \infty,\\
    &\lim_{t \rightarrow \infty} \frac{\beta^V_t}{\alpha_t} = \lim_{t \rightarrow \infty} \frac{\beta^U_t}{\beta^V_t} = \lim_{t \rightarrow \infty} \frac{\eta_t}{\beta^U_t} = 0.
\end{aligned}
\label{eq:learningrates}
\end{align}
Intuitively, this means that critics update on the fastest timescale, followed by the patient updates to the natural gradient, the hasty updates to the natural gradient, and then the Lagrange multipliers. These updates occur in an inner loop, and the policy parameters themselves are updated on an outer loop, once the natural gradients have converged. The full procedure is shown in Algorithm \ref{alg2}. 
\begin{algorithm}
    \caption{\textsc{Almanac}}
    \label{alg2}
      \begin{algorithmic}[1]
        \Statex \textbf{Input:} specifications $\{\varphi^j\}_{0 \leq j \leq m}$, discount rates $\gamma_V, \gamma_U$, learning rates $\alpha, \beta^V, \beta^U, \eta, \iota$, reset probability $p$
        \Statex \textbf{Output:} policy $\pi^i_*$
        \State convert each $\varphi^j$ into an LDBA $B^j$
        \State initialise parameters $\theta^i$, $x^i$, $\{v^j\}_{1 \leq j \leq m}$, $\{u^j\}_{1 \leq j \leq m}$, $\lambda^i$
        \While{$\theta^i$ not converged}{}
            \While{$x^i$ not converged}
                \State initialise $t \leftarrow 0$, $end \leftarrow \bot$, and $Z^j \leftarrow \varnothing$ for each $j$
                \State sample $s^\otimes_0 \sim \zeta^\otimes$
                \While {$end = \bot$}
                    \State $Z^j \leftarrow Z \cup \{s_t^\otimes\}$ for each $j$
                    \State sample $a^i_t \sim \pi^i_\theta(\cdot \vert s_t^\otimes)$
                    \State observe $s_{t+1}^\otimes$ and $r^j_{t+1}$ for each $j$
                        \If{$r^j_{t+1} > 0$ \textbf{or} $\phi(s^\otimes_{t+1})^\top v^j = 0$}
                            \State \textbf{for} $s_k^\otimes \in Z^j$ \textbf{do} update $v^j$ using (\ref{eq:patient_update})
                            \State $Z^j \leftarrow \varnothing$
                        \EndIf
                        \State update $u^j$ using (\ref{eq:hasty_update}) for each $j$
                        \State update $x^i$ and $\lambda^i$ using (\ref{eq:nat_grad_updates})
                    \State \textbf{with probability} $p$ \textbf{set} $end \gets \top$
                \EndWhile
            \EndWhile
        \State update $\theta^i$ using (\ref{eq:policy_updates})
        \EndWhile
    \State \textbf{return} $\pi^i$
    \end{algorithmic}
\end{algorithm}

\subsection{Convergence and Correctness}

By making use of results from the stochastic approximation and RL literature we provide an asymptotic convergence guarantee to locally or globally optimal joint policies with respect to multiple LTL specifications, depending on whether agents use local or global policies respectively. We assume that the following conditions hold:
\begin{enumerate}
    \item $S$ and $A$ are finite, and all reward functions are bounded.
    \item The Markov chain 
    induced by any $\theta$ is irreducible over $S^\otimes$. 
    \item $\pi^i(a^i \vert s^\otimes ;\theta^i)$ is continuously differentiable $\forall i, s^\otimes, a^i$
    \item Let $\Phi$ be the $\vert S^\otimes \vert \times c$ matrix with rows $\phi(s^\otimes)$. Then $\Phi$ has full rank, $c \leq \vert S \vert$, and $\nexists w \in W$ such that $\Phi w = 1$.
    \item $\mathbb{E}_t \big[ L^i_{V^j}(x^{i*}_t; \theta_t, \nu^j_* ) \big] \leq e^j_{approx}$, where $e^j_{approx}$ is some constant, thus $\mathbb{E}_t \big[ L^i_{V}(x^{i*}_t; \theta_t, \nu_{*} ) \big] \leq e_{approx} \coloneqq \sum_j w[j] e^j_{approx}$.
    \item $\exists \sigma < \infty$ s.t. $\log \pi^i(a^i \vert s^\otimes;\theta)$ is a $\sigma$\emph{-smooth} in $\theta^i$  $\forall i, s^\otimes, a^i$.
    \item The \emph{relative condition number} is finite.
    \item $\pi^i(\cdot \vert s^\otimes ; \theta^i)$ is initialised as the uniform distribution $\forall i, s^\otimes$.
\end{enumerate}
Conditions 1--4 are standard within the literature on the convergence of actor critic algorithms \cite{Konda2000,Bhatnagar2009}. Conditions 5--8 are taken from recent work on the convergence of natural policy gradient methods by Agarwal et al. \cite{Agarwal2019}. Of particular note is condition 5, where $e_{approx} = 0$ when $\pi^i$ is a sufficiently rich class, such as an over-parametrised neural network. We recall that if $\log \pi^i(a^i \vert s^\otimes;\theta)$ is a $\sigma$\emph{-smooth function} of $\theta^i$ then for any $\theta^i_1, \theta^i_2 \in \Theta^i$ we have:
\begin{align*}
    &\big\Vert \nabla_{\theta^i} \log \pi^i(a\vert s^\otimes;\theta^i_1) - \nabla_{\theta^i} \log \pi^i(a\vert s^\otimes;\theta^i_2)  \big\Vert_2
    \leq \alpha \big\Vert \theta^i_1 - \theta^i_2 \big\Vert_2.
\end{align*}
Regarding 6 we define $\Sigma^\nu({\theta^i}) \coloneqq \mathbb{E}_{(s^\otimes,a) \sim \nu} \big[ \psi^i_{\theta^i}(a^i_t\vert s^\otimes_t) \psi^i_{\theta^i}(a^i_t\vert s^\otimes_t)^\top \big]$ where $\nu$ is some state-action distribution. Then the average \emph{relative condition number} \cite{Agarwal2019} is defined and bounded as follows for each player $i$ and each specification $\varphi^j$:
$$\mathbb{E} \Bigg[ \sup_{x^i} \frac{{x^i}^\top \Sigma_{\nu^j_*}(\theta^i_t) x^i}{{x^i}^\top \Sigma_{\xi}(\theta^i_t) x^i} \Bigg] \leq \kappa,$$
where $\xi$ is some initial state-action distribution and:
\begin{align*}
    \nu^j_* \coloneqq \nu^j_{\theta,\xi}(s^\otimes, a) &= \sum_{(s_0^\otimes, a_0) \in S^\otimes \times A} \xi^\otimes(s_0^\otimes, a_0) \sum_{\rho} \Pr^{\theta}_{G_B} (\rho \vert s_0^\otimes, a_0)\\
    &\cdot \Bigg[ \frac{1}{\sum_{t=0}^\infty \Gamma^j_{0:t} } \sum_{t=0}^\infty \Gamma^j_{0:t} \mathbb{I} \big( \rho[t, t + 0.5] = (s^\otimes, a) \big) \Bigg]
\end{align*}
and $\rho[t + 0.5]$ refers to the action taken along the trajectory $\rho$ at time $t$. Due to space limitations we refer the interested reader to the cited works above for further discussion of these conditions.

Our proof follows the recent work of Agarwal et al. \cite{Agarwal2019}. We begin with a variant of the well-known performance difference lemma \cite{Kakade2002}, using which we prove an analogue of the `no regret' lemma from Agarwal et al. which is in turn based on the mirror-descent approach of Even-Dar et al. \cite{Even-Dar2009}. The proofs are similar to the originals, and so we relegate them to Appendix \ref{proofs}. 

\begin{lemma}
    \label{performancedifference} 
    Suppose that $V_\theta(s^\otimes) \geq V_{\theta'}(s^\otimes)$ for some state $s^\otimes$ and two policies $\pi$ and $\pi'$ parametrised by $\theta$ and $\theta'$ respectively. Then:
    $$V_\theta(s^\otimes) - V_{\theta'}(s^\otimes) \leq \sum_j w[j] \bigg( \mathbb{E}_{\rho} \Big[\sum^\infty_{t=0} \Gamma^j_{0:t} A^j_{\theta'}(s^\otimes_{t},a_{t}) ~\Big\vert~ F^j(\rho) = \infty \Big] \bigg).$$
\end{lemma}
\begin{lemma}
    \label{noregret}
     Consider a sequence of natural gradient updates $\{x^i_t\}_{0\leq t \leq T}$ found by \textsc{Almanac} such that $\Vert x^i_t \Vert_2 \leq X$ for all $t$. Let us write $\iota_{0:T} = \sum^T_{t=0} \iota_t$, and recall that $F^j(\rho)$ is the number of times a path $\rho$ in $G_B$ passes through the accepting set $F^j$ of automaton $B^j$. Let us write $\mathbb{E}_{\rho^*}$ instead of $\mathbb{E}_{\rho \sim \Pr^{\theta_*}_{G_B} (\cdot\vert s^\otimes), s^\otimes \sim \zeta^\otimes}$ and define $e^j_t$ by:
    $$e^j_t \coloneqq \mathbb{E}_{\rho^*} \Big[ \sum^\infty_{\tau=0} \Gamma^j_{0:\tau} \Big( A^j_{\theta_t}(s^\otimes_\tau,a_\tau) - \psi^i_{\theta^i_t}(a^i_\tau\vert s^\otimes_\tau)^\top x^i_t \Big) ~\Big\vert~ F^j(\rho) = \infty \Big],$$
    where $\tau$ indexes $\rho$, i.e., $\rho[\tau] = s^\otimes_\tau$. Then we have:
    $$V_{\theta_*}(s^\otimes) - \lim_{T \rightarrow \infty} \mathbb{E}_{t \sim \iota_T} \big[ V_{\theta_t}(s^\otimes) \big] = \lim_{T \rightarrow \infty} \mathbb{E}_{t \sim \iota_T} \Big[ \sum_j w[j] e^j_t \Big],$$
    where we define the distribution $\iota_T$ over $t$ with $\iota_T(t) \coloneqq \frac{\iota_t}{\iota_{0:T}}$.
\end{lemma}
Finally, we use these results to prove that \textsc{Almanac} converges to either locally or globally optimal joint policies (i.e., either an SPE or a team-optimal SPE in the original MG) depending on whether agents use local or global policy parameters. By \textit{local} policy parameters we mean that the parameters $\theta^i$ stored and updated by agent $i$ only define $\pi^i$, and thus $\pi(a \vert s^\otimes ; \theta) = \prod_i \pi^i(a^i \vert s^\otimes ; \theta^i)$ is limited in its representational power due to its factorisation. If, instead, agents share a random seed and each $\theta^i = \theta$ is sufficient to parametrise the whole joint policy $\pi$ (hence \textit{global}) then at each timestep every agent $i$ can sample the same full joint action $a = (a^1,\ldots,a^n)$ and simply perform its own action $a^i$. As rewards are shared between agents then this means that updates to each agent's version of $v$, $u$, and $x^i$ will also be identical, and therefore so too will updates to $\theta^i = \theta$. Though more expensive in terms of computation and memory, the use of global parameters guarantees convergence to the globally optimal joint policy. 

\begin{theorem}
    \label{opt}
    Given an MG $G$ and LTL objectives $\{\varphi^j\}_{1 \leq j \leq m}$ (each equivalent to an LDBA $B^j$), let $G_B = G \otimes B^1 \otimes \cdots \otimes B^m$ be the resulting product MG with newly defined reward functions $R^j_\otimes$ and state-dependent discount functions $\Gamma^j$. Assume that $\gamma_V$ satisfies Proposition \ref{Scalar2LTL}, that the learning rates $\alpha, \beta^V, \beta^U, \eta, \iota$ are as in (\ref{eq:learningrates}) and that conditions 1--8 hold. Then if each agent $i$ uses \emph{local} (\emph{global}) parameters $\theta^i$ with \emph{local} policy $\pi^i_{\theta^i}$ (\emph{global} policy $\pi^i_{\theta^i} = \pi_\theta$) then as $T \rightarrow \infty$,  \textsc{Almanac} converges to within 
    $$\lim_{T \rightarrow \infty} \mathbb{E}_{t \sim \iota_T} \left[ \sum_j w[j] \sqrt{e^j_{approx}} \frac{M^j}{(1 - \gamma_V)P^j} \right]$$
    of a \emph{local} (\emph{global}) optimum of $\sum_j w[j] \Pr^{\pi}_G(s \models \varphi^j)$, where $P^j$ and $M^j$ are constants.
\end{theorem}
\begin{proof}[Proof (Sketch)]
The proof proceeds via a multi-timescale stochastic approximation analysis and is asymptotic in nature \cite{Borkar2008}. We consider convergence of the critics, natural gradients, and actor in three steps, dividing our attention between the local and global settings, where required.
\textbf{Step 1.} The convergence proof for the critics follows that of Tsitsiklis and Van Roy \cite{Tsitsiklis1997}. The hasty critic recursion is simply the classic linear semi-gradient temporal difference algorithm \cite{Sutton2018} which is known to converge to the unique TD fixed point with probability 1. A similar argument can be made for the patient critic. By waiting to update $v^j$ until seeing a reward, we ensure that a discount is applied and thus that the patient critic recursion forms a contraction. The proof follows immediately from previous work \cite{Tsitsiklis1997}, but using a $k$-step version of the relevant Bellman equation.
\textbf{Step 2.} Due to the learning rates chosen according to (\ref{eq:learningrates}) we may consider the more slowly updated parameters fixed for the purposes of analysing the convergence of more quickly updated parameters \cite{Borkar2008}. As the critic updates fastest we may consider it converged, and since the policy is only updated in the outer loop then it is fixed with respect to the natural gradient and Lagrange multiplier updates. We show that these updates form unbiased estimates of the relevant gradients and thus discrete approximations of the following ODEs:
\begin{align*}
    \dot{x^i_t} &= \Omega_{x^i} \big[ - \nabla_{x^i} L^i_V(x^i;\theta,\nu) \big]\\
    \dot{x^i_t} &= \Omega_{x^i} \Big[ - \nabla_{x^i} \big( L^i_U(x^i;\theta,\mu) + \lambda^i \big(L^i_V(x^i;\theta,\nu) - l^i\big) \big) \Big],\\
    \dot{\lambda^i_\tau} &= \Omega_{\lambda^i} \Big[ \nabla_{\lambda^i} \big( L^i_U(x^i(\lambda^i_\tau);\theta,\mu) + \lambda^i \big[L^i_V(x^i(\lambda^i_\tau);\theta,\nu) - l^i\big) \big) \Big],
\end{align*}
on timescales $\beta^V$, $\beta^U$, and $\eta$, respectively. Due to the convexity of $L^i_V$ and $L^i_U$ it can be shown that the recursions above lexicographically miminise $L^i_V$ and then $L^i_U$ and hence that the gradient $x^i_*$ satisfies (\ref{natgrad}) \cite{Rentmeestersa1996}. \textbf{Step 3.} Finally we use Lemma \ref{noregret} and bound each term $e^j_t$ by 
$\sqrt{e^j_{approx}}\frac{M^j}{(1 - \gamma_V)P^j}$ where $M^j$ and $P_j$ are constants. In particular, we have: $M^j \coloneqq \max_k \mathbb{E}_{\rho^\zeta_*} \big[ M^j_\rho(k) ~\big\vert~ F^j(\rho) = \infty \big]$ where $M^j_\rho(k)$ is the number of steps along trajectory $\rho$ between the $k^\text{th}$ reward and preceding reward, and $P^j \coloneqq \min\big( \sum_\rho \Pr^{\theta_*}_{G_B} \mathbb{I}(F^j(\rho) = \infty), 1\big)$. The proof structure follows that of Agarwal et al. \cite{Agarwal2019} with minor variations to handle our use of multiple agents and multiple state-dependent discount rates.
\end{proof}

\section{Experiments}
\label{experiments}

Evaluating our proposed algorithm is non-trivial for several reasons. The first is its novelty; it is designed specifically to satisfy the non-Markovian, infinite-horizon specifications that other MARL algorithms are unable to learn, making a direct comparison less meaningful. The second is that the satisfaction of the specifications we wish to evaluate our algorithm against cannot be estimated simply from samples. For example, $\psi$ may be true at every state in a set of samples despite $\ltlG \psi$ being false with probability 1. Using a probabilistic model-checker instead raises a third and final difficulty, as even state-of-the-art tools are unable to handle the size of games or number of specifications that \textsc{Almanac} is applicable to.

Despite this, we provide an initial set of results in which we benchmark an implementation\footnote{Our code can be found online at \href{https://github.com/lrhammond/almanac}{\texttt{https://github.com/lrhammond/almanac}}.} of our algorithm against ground-truth models exported to PRISM, a probabilistic model-checker \cite{Kwiatkowska2011}. These results serve to demonstrate \textsc{Almanac}'s empirical convergence properties, and how this performance varies as a function of the size of the state space, the number of actors, and the number of specifications (though, unfortunately, PRISM only supports multi-objective synthesis with two specifications). For each of these combinations, we randomly generated ten MGs and sample the specifications and weights. We then ran our algorithm for 5000 episodes and exported the resulting policy, game structure, and specifications to PRISM. The differences between the weighted sum of satisfaction probabilities resulting from \textsc{Almanac} and the ground-truth optimal quantities are displayed in Table \ref{tab:results}. We ran PRISM with a maximum of 16GB of memory, 100,000 value iteration steps, and twelve hours of computation, but for some combinations this was insufficient.

\setlength\tabcolsep{4 pt}
\begin{table}\centering
\begin{tabular}{c l l l l l l l l l l }
\toprule
 & $2^1$ & $2^2$ & $2^3$ & $2^4$ & $2^5$ & $2^6$ & $2^7$ & $2^8$ & $2^9$ & $2^{10}$\\
\midrule
1 & 0.13 & 0.20 & 0.16 & 0.21 & 0.14 & 0.13 & 0.17 & 0.22 & 0.19 & --\\
2 & 0.54 & 0.26 & 0.19 & 0.12 & 0.30 & 0.19 & 0.36 & 0.29 & 0.38 & --\\
3 & 0.48 & 0.23 & 0.25 & 0.10 & 0.20 & 0.15 & 0.10 & 0.31 & -- & --\\
4 & 0.44 & 0.21 & 0.02 & 0.16 & 0.22 & 0.26 & 0.23 & 0.34 & -- & --\\
5 & 0.17 & 0.30 & 0.10 & 0.13 & 0.22 & 0.06 & 0.30 & -- & -- & --\\
\cmidrule{1-11}
1  & 0.14 & 0.07 & 0.11 & 0.17 & 0.18 & 0.09 & 0.24 & 0.14 & 0.25 & --\\
2  & 0.15 & 0.07 & 0.15 & 0.34 & 0.20 & 0.15 & 0.17 & 0.06 & -- & --\\
3  & 0.15 & 0.14 & 0.12 & 0.25 & 0.23 & 0.52 & 0.28 & -- & -- & --\\
4  & 0.12 & 0.25 & 0.23 & 0.23 & 0.22 & -- & -- & -- & -- & --\\
5  & 0.23 & 0.21 & 0.28 & 0.45 & 0.01 & -- & -- & -- & -- & --\\
\bottomrule\\
\end{tabular}
\caption{Average errors across a number of states (columns), agents (rows), and specifications (top and bottom).}
\label{tab:results}
\end{table}

%%%%%%%%%%%%%%%%%%%%%%%%%%%%%%%%%%%%%%%%%%%%%%%%%%%%%%%%%%%%%%%%%%%%%%%%

%%% The acknowledgments section is defined using the "acks" environment
%%% (rather than an unnumbered section). The use of this environment 
%%% ensures the proper identification of the section in the article 
%%% metadata as well as the consistent spelling of the heading.

\begin{acks}
    The authors thank Hosein Hasanbeig, Joar Skalse, Alper Kamil Bozkurt, Kaiqing Zhang, Salomon Sickert, and the anonymous reviewers for helpful comments. Hammond acknowledges the support of an EPSRC Doctoral Training Partnership studentship (Reference: 2218880) and the University of Oxford ARC facility.\footnote{Details available at \href{http://dx.doi.org/10.5281/zenodo.22558}{\texttt{http://dx.doi.org/10.5281/zenodo.22558}}.} Wooldridge and Abate acknowledge the support of the Alan Turing Institute.
\end{acks}

%%%%%%%%%%%%%%%%%%%%%%%%%%%%%%%%%%%%%%%%%%%%%%%%%%%%%%%%%%%%%%%%%%%%%%%%

%%% The next two lines define, first, the bibliography style to be 
%%% applied, and, second, the bibliography file to be used.

\bibliographystyle{ACM-Reference-Format} 
\bibliography{ms}

%%% -*-BibTeX-*-
%%% Do NOT edit. File created by BibTeX with style
%%% ACM-Reference-Format-Journals [18-Jan-2012].

\begin{thebibliography}{60}

%%% ====================================================================
%%% NOTE TO THE USER: you can override these defaults by providing
%%% customized versions of any of these macros before the \bibliography
%%% command.  Each of them MUST provide its own final punctuation,
%%% except for \shownote{}, \showDOI{}, and \showURL{}.  The latter two
%%% do not use final punctuation, in order to avoid confusing it with
%%% the Web address.
%%%
%%% To suppress output of a particular field, define its macro to expand
%%% to an empty string, or better, \unskip, like this:
%%%
%%% \newcommand{\showDOI}[1]{\unskip}   % LaTeX syntax
%%%
%%% \def \showDOI #1{\unskip}           % plain TeX syntax
%%%
%%% ====================================================================

\ifx \showCODEN    \undefined \def \showCODEN     #1{\unskip}     \fi
\ifx \showDOI      \undefined \def \showDOI       #1{#1}\fi
\ifx \showISBNx    \undefined \def \showISBNx     #1{\unskip}     \fi
\ifx \showISBNxiii \undefined \def \showISBNxiii  #1{\unskip}     \fi
\ifx \showISSN     \undefined \def \showISSN      #1{\unskip}     \fi
\ifx \showLCCN     \undefined \def \showLCCN      #1{\unskip}     \fi
\ifx \shownote     \undefined \def \shownote      #1{#1}          \fi
\ifx \showarticletitle \undefined \def \showarticletitle #1{#1}   \fi
\ifx \showURL      \undefined \def \showURL       {\relax}        \fi
% The following commands are used for tagged output and should be
% invisible to TeX
\providecommand\bibfield[2]{#2}
\providecommand\bibinfo[2]{#2}
\providecommand\natexlab[1]{#1}
\providecommand\showeprint[2][]{arXiv:#2}

\bibitem[\protect\citeauthoryear{Agarwal, Kakade, Lee, and Mahajan}{Agarwal
  et~al\mbox{.}}{2020}]%
        {Agarwal2019}
\bibfield{author}{\bibinfo{person}{Alekh Agarwal}, \bibinfo{person}{Sham~M
  Kakade}, \bibinfo{person}{Jason~D Lee}, {and} \bibinfo{person}{Gaurav
  Mahajan}.} \bibinfo{year}{2020}\natexlab{}.
\newblock \showarticletitle{Optimality and Approximation with Policy Gradient
  Methods in Markov Decision Processes}. In
  \bibinfo{booktitle}{\emph{Proceedings of Thirty Third Conference on Learning
  Theory}} \emph{(\bibinfo{series}{Proceedings of Machine Learning Research},
  Vol.~\bibinfo{volume}{125})}, \bibfield{editor}{\bibinfo{person}{Jacob
  Abernethy} {and} \bibinfo{person}{Shivani Agarwal}} (Eds.).
  \bibinfo{publisher}{PMLR}, \bibinfo{pages}{64--66}.
\newblock


\bibitem[\protect\citeauthoryear{Amari}{Amari}{1998}]%
        {Amari1998}
\bibfield{author}{\bibinfo{person}{Shun'ichi Amari}.}
  \bibinfo{year}{1998}\natexlab{}.
\newblock \showarticletitle{Natural Gradient Works Efficiently in Learning}.
\newblock \bibinfo{journal}{\emph{Neural Computation}} \bibinfo{volume}{10},
  \bibinfo{number}{2} (\bibinfo{year}{1998}), \bibinfo{pages}{251–276}.
\newblock


\bibitem[\protect\citeauthoryear{Arslan and Yuksel}{Arslan and Yuksel}{2017}]%
        {Arslan2017}
\bibfield{author}{\bibinfo{person}{Gurdal Arslan} {and} \bibinfo{person}{Serdar
  Yuksel}.} \bibinfo{year}{2017}\natexlab{}.
\newblock \showarticletitle{Decentralized Q-learning for Stochastic Teams and
  Games}.
\newblock \bibinfo{journal}{\emph{IEEE Trans. Automat. Control}}
  \bibinfo{volume}{62}, \bibinfo{number}{4} (\bibinfo{year}{2017}),
  \bibinfo{pages}{1545–1558}.
\newblock


\bibitem[\protect\citeauthoryear{Ashok, Křetínský, and Weininger}{Ashok
  et~al\mbox{.}}{2019}]%
        {Ashok2019}
\bibfield{author}{\bibinfo{person}{Pranav Ashok}, \bibinfo{person}{Jan
  Křetínský}, {and} \bibinfo{person}{Maximilian Weininger}.}
  \bibinfo{year}{2019}\natexlab{}.
\newblock \showarticletitle{Pac Statistical Model Checking for Markov Decision
  Processes and Stochastic Games}.
\newblock In \bibinfo{booktitle}{\emph{Computer Aided Verification}}.
  \bibinfo{publisher}{Springer International Publishing},
  \bibinfo{pages}{497–519}.
\newblock


\bibitem[\protect\citeauthoryear{Bertsekas and Tsitsiklis}{Bertsekas and
  Tsitsiklis}{1996}]%
        {Bertsekas1996}
\bibfield{author}{\bibinfo{person}{Dimitri~P. Bertsekas} {and}
  \bibinfo{person}{John~N. Tsitsiklis}.} \bibinfo{year}{1996}\natexlab{}.
\newblock \bibinfo{booktitle}{\emph{Neuro-dynamic Programming}}.
\newblock \bibinfo{publisher}{Athena Scientific}.
\newblock
\showISBNx{1886529108}


\bibitem[\protect\citeauthoryear{Bhatnagar, Sutton, Ghavamzadeh, and
  Lee}{Bhatnagar et~al\mbox{.}}{2009}]%
        {Bhatnagar2009}
\bibfield{author}{\bibinfo{person}{Shalabh Bhatnagar},
  \bibinfo{person}{Richard~S. Sutton}, \bibinfo{person}{Mohammad Ghavamzadeh},
  {and} \bibinfo{person}{Mark Lee}.} \bibinfo{year}{2009}\natexlab{}.
\newblock \showarticletitle{Natural Actor–critic Algorithms}.
\newblock \bibinfo{journal}{\emph{Automatica}} \bibinfo{volume}{45},
  \bibinfo{number}{11} (\bibinfo{year}{2009}), \bibinfo{pages}{2471–2482}.
\newblock


\bibitem[\protect\citeauthoryear{Borkar}{Borkar}{2008}]%
        {Borkar2008}
\bibfield{author}{\bibinfo{person}{Vivek~S. Borkar}.}
  \bibinfo{year}{2008}\natexlab{}.
\newblock \bibinfo{booktitle}{\emph{Stochastic Approximation}}.
\newblock \bibinfo{publisher}{Hindustan Book Agency}.
\newblock


\bibitem[\protect\citeauthoryear{Bowling and Veloso}{Bowling and
  Veloso}{2001}]%
        {Bowling2001}
\bibfield{author}{\bibinfo{person}{Michael Bowling} {and}
  \bibinfo{person}{Manuela Veloso}.} \bibinfo{year}{2001}\natexlab{}.
\newblock \showarticletitle{Rational and Convergent Learning in Stochastic
  Games}. In \bibinfo{booktitle}{\emph{Proceedings of the 17th International
  Joint Conference on Artificial Intelligence - Volume 2}} (Seattle, WA, USA)
  \emph{(\bibinfo{series}{IJCAI'01})}. \bibinfo{publisher}{Morgan Kaufmann
  Publishers Inc.}, \bibinfo{address}{San Francisco, CA, USA},
  \bibinfo{pages}{1021–1026}.
\newblock
\showISBNx{1-55860-812-5, 978-1-558-60812-2}


\bibitem[\protect\citeauthoryear{Bozkurt, Wang, Zavlanos, and Pajic}{Bozkurt
  et~al\mbox{.}}{2019}]%
        {Bozkurt2019}
\bibfield{author}{\bibinfo{person}{Alper~Kamil Bozkurt}, \bibinfo{person}{Yu
  Wang}, \bibinfo{person}{Michael~M. Zavlanos}, {and} \bibinfo{person}{Miroslav
  Pajic}.} \bibinfo{year}{2019}\natexlab{}.
\newblock \showarticletitle{Control Synthesis from Linear Temporal Logic
  Specifications Using Model-free Reinforcement Learning}.
\newblock \bibinfo{journal}{\emph{arXiv:1909.07299}} (\bibinfo{year}{2019}).
\newblock


\bibitem[\protect\citeauthoryear{Brázdil, Chatterjee, Chmelík, Forejt,
  Křetiínský, Kwiatkowska, Parker, and Ujma}{Brázdil et~al\mbox{.}}{2014}]%
        {Brazdil2014}
\bibfield{author}{\bibinfo{person}{Tomáš Brázdil},
  \bibinfo{person}{Krishnendu Chatterjee}, \bibinfo{person}{Martin Chmelík},
  \bibinfo{person}{Vojtěch Forejt}, \bibinfo{person}{Jan Křetiínský},
  \bibinfo{person}{Marta Kwiatkowska}, \bibinfo{person}{David Parker}, {and}
  \bibinfo{person}{Mateusz Ujma}.} \bibinfo{year}{2014}\natexlab{}.
\newblock \showarticletitle{Verification of Markov Decision Processes Using
  Learning Algorithms}.
\newblock In \bibinfo{booktitle}{\emph{Automated Technology for Verification
  and Analysis}}. \bibinfo{publisher}{Springer International Publishing},
  \bibinfo{pages}{98–114}.
\newblock


\bibitem[\protect\citeauthoryear{Busoniu, Babuska, and Schutter}{Busoniu
  et~al\mbox{.}}{2008}]%
        {Busoniu2008}
\bibfield{author}{\bibinfo{person}{Lucian Busoniu}, \bibinfo{person}{Robert
  Babuska}, {and} \bibinfo{person}{Bart~De Schutter}.}
  \bibinfo{year}{2008}\natexlab{}.
\newblock \showarticletitle{A Comprehensive Survey of Multiagent Reinforcement
  Learning}.
\newblock \bibinfo{journal}{\emph{IEEE Transactions on Systems, Man, and
  Cybernetics, Part C (Applications and Reviews)}} \bibinfo{volume}{38},
  \bibinfo{number}{2} (\bibinfo{year}{2008}), \bibinfo{pages}{156–172}.
\newblock


\bibitem[\protect\citeauthoryear{Conitzer and Sandholm}{Conitzer and
  Sandholm}{2003}]%
        {Conitzer2003}
\bibfield{author}{\bibinfo{person}{Vincent Conitzer} {and}
  \bibinfo{person}{Tuomas Sandholm}.} \bibinfo{year}{2003}\natexlab{}.
\newblock \showarticletitle{Awesome: A General Multiagent Learning Algorithm
  That Converges in Self-play and Learns a Best Response against Stationary
  Opponents}. In \bibinfo{booktitle}{\emph{Proceedings of the Twentieth
  International Conference on International Conference on Machine Learning}}
  \emph{(\bibinfo{series}{ICML’03})}. \bibinfo{publisher}{AAAI Press},
  \bibinfo{address}{Washington, DC, USA}, \bibinfo{pages}{83–90}.
\newblock
\showISBNx{1577351894}


\bibitem[\protect\citeauthoryear{Etessami, Kwiatkowska, Vardi, and
  Yannakakis}{Etessami et~al\mbox{.}}{2007}]%
        {Etessami2007}
\bibfield{author}{\bibinfo{person}{Kousha Etessami}, \bibinfo{person}{Marta
  Kwiatkowska}, \bibinfo{person}{Moshe~Y. Vardi}, {and}
  \bibinfo{person}{Mihalis Yannakakis}.} \bibinfo{year}{2007}\natexlab{}.
\newblock \showarticletitle{Multi-objective Model Checking of Markov Decision
  Processes}.
\newblock In \bibinfo{booktitle}{\emph{Tools and Algorithms for the
  Construction and Analysis of Systems}}. \bibinfo{publisher}{Springer Berlin
  Heidelberg}, \bibinfo{pages}{50–65}.
\newblock


\bibitem[\protect\citeauthoryear{Even-Dar, Kakade, and Mansour}{Even-Dar
  et~al\mbox{.}}{2009}]%
        {Even-Dar2009}
\bibfield{author}{\bibinfo{person}{Eyal Even-Dar}, \bibinfo{person}{Sham.~M.
  Kakade}, {and} \bibinfo{person}{Yishay Mansour}.}
  \bibinfo{year}{2009}\natexlab{}.
\newblock \showarticletitle{Online Markov Decision Processes}.
\newblock \bibinfo{journal}{\emph{Mathematics of Operations Research}}
  \bibinfo{volume}{34}, \bibinfo{number}{3} (\bibinfo{year}{2009}),
  \bibinfo{pages}{726–736}.
\newblock


\bibitem[\protect\citeauthoryear{Fisman, Kupferman, and Lustig}{Fisman
  et~al\mbox{.}}{2010}]%
        {Fisman2010}
\bibfield{author}{\bibinfo{person}{Dana Fisman}, \bibinfo{person}{Orna
  Kupferman}, {and} \bibinfo{person}{Yoad Lustig}.}
  \bibinfo{year}{2010}\natexlab{}.
\newblock \showarticletitle{Rational Synthesis}.
\newblock In \bibinfo{booktitle}{\emph{Tools and Algorithms for the
  Construction and Analysis of Systems}}. \bibinfo{publisher}{Springer Berlin
  Heidelberg}, \bibinfo{pages}{190–204}.
\newblock


\bibitem[\protect\citeauthoryear{Foerster, Farquhar, Afouras, Nardelli, and
  Whiteson}{Foerster et~al\mbox{.}}{2018}]%
        {Foerster2018}
\bibfield{author}{\bibinfo{person}{Jakob~N. Foerster}, \bibinfo{person}{Gregory
  Farquhar}, \bibinfo{person}{Triantafyllos Afouras}, \bibinfo{person}{Nantas
  Nardelli}, {and} \bibinfo{person}{Shimon Whiteson}.}
  \bibinfo{year}{2018}\natexlab{}.
\newblock \showarticletitle{Counterfactual Multi-agent Policy Gradients}. In
  \bibinfo{booktitle}{\emph{Proceedings of the Thirty-Second AAAI Conference on
  Artificial Intelligence, (AAAI-18), the 30th innovative Applications of
  Artificial Intelligence (IAAI-18), and the 8th AAAI Symposium on Educational
  Advances in Artificial Intelligence (EAAI-18), New Orleans, Louisiana, USA,
  February 2-7, 2018}}, \bibfield{editor}{\bibinfo{person}{Sheila~A. McIlraith}
  {and} \bibinfo{person}{Kilian~Q. Weinberger}} (Eds.).
  \bibinfo{publisher}{AAAI Press}, \bibinfo{pages}{2974–2982}.
\newblock


\bibitem[\protect\citeauthoryear{Fu and Topcu}{Fu and Topcu}{2014}]%
        {Fu2014}
\bibfield{author}{\bibinfo{person}{Jie Fu} {and} \bibinfo{person}{Ufuk Topcu}.}
  \bibinfo{year}{2014}\natexlab{}.
\newblock \showarticletitle{Probably Approximately Correct Mdp Learning and
  Control with Temporal Logic Constraints}. In
  \bibinfo{booktitle}{\emph{Robotics: Science and Systems X, University of
  California, Berkeley, USA, July 12-16, 2014}},
  \bibfield{editor}{\bibinfo{person}{Dieter Fox}, \bibinfo{person}{Lydia~E.
  Kavraki}, {and} \bibinfo{person}{Hanna Kurniawati}} (Eds.).
\newblock


\bibitem[\protect\citeauthoryear{Fudenberg and Tirole}{Fudenberg and
  Tirole}{1991}]%
        {Fudenberg1991}
\bibfield{author}{\bibinfo{person}{Drew Fudenberg} {and} \bibinfo{person}{Jean
  Tirole}.} \bibinfo{year}{1991}\natexlab{}.
\newblock \bibinfo{booktitle}{\emph{Game Theory}}.
\newblock \bibinfo{publisher}{The MIT Press}.
\newblock
\showISBNx{9780262061414}


\bibitem[\protect\citeauthoryear{Hahn, Perez, Schewe, Somenzi, Trivedi, and
  Wojtczak}{Hahn et~al\mbox{.}}{2019}]%
        {Hahn2019}
\bibfield{author}{\bibinfo{person}{Ernst~M. Hahn}, \bibinfo{person}{Mateo
  Perez}, \bibinfo{person}{Sven Schewe}, \bibinfo{person}{Fabio Somenzi},
  \bibinfo{person}{Ashutosh Trivedi}, {and} \bibinfo{person}{Dominik
  Wojtczak}.} \bibinfo{year}{2019}\natexlab{}.
\newblock \showarticletitle{Omega-regular Objectives in Model-free
  Reinforcement Learning}.
\newblock In \bibinfo{booktitle}{\emph{Tools and Algorithms for the
  Construction and Analysis of Systems}}. \bibinfo{publisher}{Springer
  International Publishing}, \bibinfo{pages}{395–412}.
\newblock


\bibitem[\protect\citeauthoryear{Hahn, Perez, Schewe, Somenzi, Trivedi, and
  Wojtczak}{Hahn et~al\mbox{.}}{2020}]%
        {Hahn2020}
\bibfield{author}{\bibinfo{person}{Ernst~M. Hahn}, \bibinfo{person}{Mateo
  Perez}, \bibinfo{person}{Sven Schewe}, \bibinfo{person}{Fabio Somenzi},
  \bibinfo{person}{Ashutosh Trivedi}, {and} \bibinfo{person}{Dominik
  Wojtczak}.} \bibinfo{year}{2020}\natexlab{}.
\newblock \showarticletitle{Reward Shaping for Reinforcement Learning with
  Omega-regular Objectives}.
\newblock \bibinfo{journal}{\emph{arXiv:2001.05977}} (\bibinfo{year}{2020}).
\newblock


\bibitem[\protect\citeauthoryear{Hasanbeig, Abate, and Kroening}{Hasanbeig
  et~al\mbox{.}}{2019}]%
        {Hasanbeig2019}
\bibfield{author}{\bibinfo{person}{Mohammadhosein Hasanbeig},
  \bibinfo{person}{Alessandro Abate}, {and} \bibinfo{person}{Daniel Kroening}.}
  \bibinfo{year}{2019}\natexlab{}.
\newblock \showarticletitle{Certified Reinforcement Learning with Logic
  Guidance}.
\newblock \bibinfo{journal}{\emph{arXiv:1902.00778}} (\bibinfo{year}{2019}).
\newblock


\bibitem[\protect\citeauthoryear{Hasanbeig, Kroening, and Abate}{Hasanbeig
  et~al\mbox{.}}{2020}]%
        {Hasanbeig2020a}
\bibfield{author}{\bibinfo{person}{Mohammadhosein Hasanbeig},
  \bibinfo{person}{Daniel Kroening}, {and} \bibinfo{person}{Alessandro Abate}.}
  \bibinfo{year}{2020}\natexlab{}.
\newblock \showarticletitle{Deep Reinforcement Learning with Temporal Logics}.
\newblock In \bibinfo{booktitle}{\emph{Lecture Notes in Computer Science}}.
  \bibinfo{publisher}{Springer International Publishing},
  \bibinfo{pages}{1--22}.
\newblock


\bibitem[\protect\citeauthoryear{Jothimurugan, Alur, and Bastani}{Jothimurugan
  et~al\mbox{.}}{2019}]%
        {Jothimurugan2019}
\bibfield{author}{\bibinfo{person}{Kishor Jothimurugan},
  \bibinfo{person}{Rajeev Alur}, {and} \bibinfo{person}{Osbert Bastani}.}
  \bibinfo{year}{2019}\natexlab{}.
\newblock \showarticletitle{A Composable Specification Language for
  Reinforcement Learning Tasks}. In \bibinfo{booktitle}{\emph{Advances in
  Neural Information Processing Systems 32}},
  \bibfield{editor}{\bibinfo{person}{H.~Wallach},
  \bibinfo{person}{H.~Larochelle}, \bibinfo{person}{A.~Beygelzimer},
  \bibinfo{person}{F.~d~Alché-Buc}, \bibinfo{person}{E.~Fox}, {and}
  \bibinfo{person}{R.~Garnett}} (Eds.). \bibinfo{publisher}{Curran Associates,
  Inc.}, \bibinfo{pages}{13041–13051}.
\newblock


\bibitem[\protect\citeauthoryear{Kakade}{Kakade}{2001}]%
        {Kakade2001}
\bibfield{author}{\bibinfo{person}{Sham Kakade}.}
  \bibinfo{year}{2001}\natexlab{}.
\newblock \showarticletitle{A Natural Policy Gradient}. In
  \bibinfo{booktitle}{\emph{Proceedings of the 14th International Conference on
  Neural Information Processing Systems: Natural and Synthetic}} (Vancouver,
  British Columbia, Canada) \emph{(\bibinfo{series}{NIPS’01})}.
  \bibinfo{publisher}{MIT Press}, \bibinfo{address}{Cambridge, MA, USA},
  \bibinfo{pages}{1531–1538}.
\newblock


\bibitem[\protect\citeauthoryear{Kakade and Langford}{Kakade and
  Langford}{2002}]%
        {Kakade2002}
\bibfield{author}{\bibinfo{person}{Sham Kakade} {and} \bibinfo{person}{John
  Langford}.} \bibinfo{year}{2002}\natexlab{}.
\newblock \showarticletitle{Approximately Optimal Approximate Reinforcement
  Learning}. In \bibinfo{booktitle}{\emph{Proceedings of the Nineteenth
  International Conference on Machine Learning}} \emph{(\bibinfo{series}{ICML
  ’02})}. \bibinfo{publisher}{Morgan Kaufmann Publishers Inc.},
  \bibinfo{address}{San Francisco, CA, USA}, \bibinfo{pages}{267–274}.
\newblock
\showISBNx{1558608737}


\bibitem[\protect\citeauthoryear{Konda and Tsitsiklis}{Konda and
  Tsitsiklis}{2000}]%
        {Konda2000}
\bibfield{author}{\bibinfo{person}{Vijay~R. Konda} {and}
  \bibinfo{person}{John~N. Tsitsiklis}.} \bibinfo{year}{2000}\natexlab{}.
\newblock \showarticletitle{Actor-critic Algorithms}.
\newblock In \bibinfo{booktitle}{\emph{Advances in Neural Information
  Processing Systems 12}}, \bibfield{editor}{\bibinfo{person}{S.~A. Solla},
  \bibinfo{person}{T.~K. Leen}, {and} \bibinfo{person}{K.~Müller}} (Eds.).
  \bibinfo{publisher}{MIT Press}, \bibinfo{pages}{1008–1014}.
\newblock


\bibitem[\protect\citeauthoryear{Kwiatkowska, Norman, and Parker}{Kwiatkowska
  et~al\mbox{.}}{2011}]%
        {Kwiatkowska2011}
\bibfield{author}{\bibinfo{person}{Marta Kwiatkowska}, \bibinfo{person}{Gethin
  Norman}, {and} \bibinfo{person}{David Parker}.}
  \bibinfo{year}{2011}\natexlab{}.
\newblock \showarticletitle{Prism 4.0: Verification of Probabilistic Real-time
  Systems}. In \bibinfo{booktitle}{\emph{Proc. 23rd International Conference on
  Computer Aided Verification (CAV'11)}} \emph{(\bibinfo{series}{LNCS},
  Vol.~\bibinfo{volume}{6806})},
  \bibfield{editor}{\bibinfo{person}{G.~Gopalakrishnan} {and}
  \bibinfo{person}{S.~Qadeer}} (Eds.). \bibinfo{publisher}{Springer},
  \bibinfo{pages}{585–591}.
\newblock


\bibitem[\protect\citeauthoryear{Kwiatkowska, Norman, Parker, and
  Santos}{Kwiatkowska et~al\mbox{.}}{2019}]%
        {Kwiatkowska2019}
\bibfield{author}{\bibinfo{person}{Marta Kwiatkowska}, \bibinfo{person}{Gethin
  Norman}, \bibinfo{person}{David Parker}, {and} \bibinfo{person}{Gabriel
  Santos}.} \bibinfo{year}{2019}\natexlab{}.
\newblock \showarticletitle{Equilibria-based Probabilistic Model Checking for
  Concurrent Stochastic Games}. In \bibinfo{booktitle}{\emph{Lecture Notes in
  Computer Science}}. \bibinfo{publisher}{Springer International Publishing},
  \bibinfo{pages}{298–315}.
\newblock


\bibitem[\protect\citeauthoryear{León and Belardinelli}{León and
  Belardinelli}{2020}]%
        {Leon2020}
\bibfield{author}{\bibinfo{person}{Borja~G. León} {and}
  \bibinfo{person}{Francesco Belardinelli}.} \bibinfo{year}{2020}\natexlab{}.
\newblock \showarticletitle{Extended Markov Games to Learn Multiple Tasks in
  Multi-Agent Reinforcement Learning}. In \bibinfo{booktitle}{\emph{{ECAI} 2020
  - 24th European Conference on Artificial Intelligence, 29 August-8 September
  2020, Santiago de Compostela, Spain, August 29 - September 8, 2020 -
  Including 10th Conference on Prestigious Applications of Artificial
  Intelligence {(PAIS} 2020)}} \emph{(\bibinfo{series}{Frontiers in Artificial
  Intelligence and Applications}, Vol.~\bibinfo{volume}{325})},
  \bibfield{editor}{\bibinfo{person}{Giuseppe~De Giacomo},
  \bibinfo{person}{Alejandro Catal{\'{a}}}, \bibinfo{person}{Bistra Dilkina},
  \bibinfo{person}{Michela Milano}, \bibinfo{person}{Sen{\'{e}}n Barro},
  \bibinfo{person}{Alberto Bugarín}, {and} \bibinfo{person}{J{\'{e}}r{\^{o}}me
  Lang}} (Eds.). \bibinfo{publisher}{{IOS} Press}, \bibinfo{pages}{139--146}.
\newblock


\bibitem[\protect\citeauthoryear{Li, Ma, and Belta}{Li et~al\mbox{.}}{2018}]%
        {Li2018}
\bibfield{author}{\bibinfo{person}{Xiao Li}, \bibinfo{person}{Yao Ma}, {and}
  \bibinfo{person}{Calin Belta}.} \bibinfo{year}{2018}\natexlab{}.
\newblock \showarticletitle{Automata Guided Reinforcement Learning with
  Demonstrations}.
\newblock \bibinfo{journal}{\emph{arXiv:1809.06305}} (\bibinfo{year}{2018}).
\newblock


\bibitem[\protect\citeauthoryear{Li, Vasile, and Belta}{Li
  et~al\mbox{.}}{2017}]%
        {Li2017}
\bibfield{author}{\bibinfo{person}{Xiao Li}, \bibinfo{person}{Cristian-Ioan
  Vasile}, {and} \bibinfo{person}{Calin Belta}.}
  \bibinfo{year}{2017}\natexlab{}.
\newblock \showarticletitle{Reinforcement Learning with Temporal Logic
  Rewards}. In \bibinfo{booktitle}{\emph{IEEE/RSJ International Conference on
  Intelligent Robots and Systems (IROS)}}. \bibinfo{publisher}{IEEE}.
\newblock


\bibitem[\protect\citeauthoryear{Littman}{Littman}{1994}]%
        {Littman1994}
\bibfield{author}{\bibinfo{person}{Michael~L. Littman}.}
  \bibinfo{year}{1994}\natexlab{}.
\newblock \showarticletitle{Markov Games As a Framework for Multi-agent
  Reinforcement Learning}. In \bibinfo{booktitle}{\emph{Proceedings of the
  Eleventh International Conference on International Conference on Machine
  Learning}} (New Brunswick, NJ, USA) \emph{(\bibinfo{series}{ICML'94})}.
  \bibinfo{publisher}{Morgan Kaufmann Publishers Inc.}, \bibinfo{address}{San
  Francisco, CA, USA}, \bibinfo{pages}{157–163}.
\newblock
\showISBNx{1-55860-335-2}


\bibitem[\protect\citeauthoryear{Littman, Topcu, Fu, Isbell, Wen, and
  MacGlashan}{Littman et~al\mbox{.}}{2017}]%
        {Littman2017}
\bibfield{author}{\bibinfo{person}{Michael~L. Littman}, \bibinfo{person}{Ufuk
  Topcu}, \bibinfo{person}{Jie Fu}, \bibinfo{person}{Charles Isbell},
  \bibinfo{person}{Min Wen}, {and} \bibinfo{person}{James MacGlashan}.}
  \bibinfo{year}{2017}\natexlab{}.
\newblock \showarticletitle{Environment-independent Task Specifications Via
  Gltl}.
\newblock \bibinfo{journal}{\emph{arXiv:1704.04341}} (\bibinfo{year}{2017}).
\newblock


\bibitem[\protect\citeauthoryear{Lowe, Wu, Tamar, Harb, Abbeel, and
  Mordatch}{Lowe et~al\mbox{.}}{2017}]%
        {Lowe2017}
\bibfield{author}{\bibinfo{person}{Ryan Lowe}, \bibinfo{person}{Yi Wu},
  \bibinfo{person}{Aviv Tamar}, \bibinfo{person}{Jean Harb},
  \bibinfo{person}{Pieter Abbeel}, {and} \bibinfo{person}{Igor Mordatch}.}
  \bibinfo{year}{2017}\natexlab{}.
\newblock \showarticletitle{Multi-agent Actor-critic for Mixed
  Cooperative-competitive Environments}. In
  \bibinfo{booktitle}{\emph{Proceedings of the 31st International Conference on
  Neural Information Processing Systems}} (Long Beach, California, USA)
  \emph{(\bibinfo{series}{NIPS’17})}. \bibinfo{publisher}{Curran Associates
  Inc.}, \bibinfo{address}{Red Hook, NY, USA}, \bibinfo{pages}{6382–6393}.
\newblock
\showISBNx{9781510860964}


\bibitem[\protect\citeauthoryear{Maskin and Tirole}{Maskin and Tirole}{2001}]%
        {Maskin2001}
\bibfield{author}{\bibinfo{person}{Eric Maskin} {and} \bibinfo{person}{Jean
  Tirole}.} \bibinfo{year}{2001}\natexlab{}.
\newblock \showarticletitle{Markov Perfect Equilibrium}.
\newblock \bibinfo{journal}{\emph{Journal of Economic Theory}}
  \bibinfo{volume}{100}, \bibinfo{number}{2} (\bibinfo{year}{2001}),
  \bibinfo{pages}{191–219}.
\newblock


\bibitem[\protect\citeauthoryear{Nota and Thomas}{Nota and Thomas}{2020}]%
        {Nota2020}
\bibfield{author}{\bibinfo{person}{Chris Nota} {and} \bibinfo{person}{Philip~S.
  Thomas}.} \bibinfo{year}{2020}\natexlab{}.
\newblock \showarticletitle{Is the Policy Gradient a Gradient?}. In
  \bibinfo{booktitle}{\emph{Proceedings of the 19th International Conference on
  Autonomous Agents and Multi-Agent Systems}} (Auckland, New Zealand)
  \emph{(\bibinfo{series}{AAMAS ’20})}. \bibinfo{publisher}{International
  Foundation for Autonomous Agents and Multiagent Systems},
  \bibinfo{address}{Richland, SC}, \bibinfo{pages}{939–947}.
\newblock
\showISBNx{9781450375184}


\bibitem[\protect\citeauthoryear{Nowé, Vrancx, and Hauwere}{Nowé
  et~al\mbox{.}}{2012}]%
        {Nowe2012}
\bibfield{author}{\bibinfo{person}{Ann Nowé}, \bibinfo{person}{Peter Vrancx},
  {and} \bibinfo{person}{Yann-Michaël~De Hauwere}.}
  \bibinfo{year}{2012}\natexlab{}.
\newblock \showarticletitle{Game Theory and Multi-agent Reinforcement
  Learning}.
\newblock In \bibinfo{booktitle}{\emph{Adaptation, Learning, and
  Optimization}}. \bibinfo{publisher}{Springer Berlin Heidelberg},
  \bibinfo{pages}{441–470}.
\newblock


\bibitem[\protect\citeauthoryear{Oura, Sakakibara, and Ushio}{Oura
  et~al\mbox{.}}{2020}]%
        {Oura2020}
\bibfield{author}{\bibinfo{person}{Ryohei Oura}, \bibinfo{person}{Ami
  Sakakibara}, {and} \bibinfo{person}{Toshimitsu Ushio}.}
  \bibinfo{year}{2020}\natexlab{}.
\newblock \showarticletitle{Reinforcement Learning of Control Policy for Linear
  Temporal Logic Specifications Using Limit-deterministic Generalized Büchi
  Automata}.
\newblock \bibinfo{journal}{\emph{arXiv:2001.04669}} (\bibinfo{year}{2020}).
\newblock


\bibitem[\protect\citeauthoryear{Perolat, Piot, and Pietquin}{Perolat
  et~al\mbox{.}}{2018}]%
        {Perolat2018}
\bibfield{author}{\bibinfo{person}{Julien Perolat}, \bibinfo{person}{Bilal
  Piot}, {and} \bibinfo{person}{Olivier Pietquin}.}
  \bibinfo{year}{2018}\natexlab{}.
\newblock \showarticletitle{Actor-critic Fictitious Play in Simultaneous Move
  Multistage Games}. In \bibinfo{booktitle}{\emph{Proceedings of the
  Twenty-First International Conference on Artificial Intelligence and
  Statistics}} \emph{(\bibinfo{series}{Proceedings of Machine Learning
  Research}, Vol.~\bibinfo{volume}{84})},
  \bibfield{editor}{\bibinfo{person}{Amos Storkey} {and}
  \bibinfo{person}{Fernando Perez-Cruz}} (Eds.). \bibinfo{publisher}{PMLR},
  \bibinfo{address}{Playa Blanca, Lanzarote, Canary Islands},
  \bibinfo{pages}{919–928}.
\newblock


\bibitem[\protect\citeauthoryear{Peters and Schaal}{Peters and Schaal}{2008}]%
        {Peters2008}
\bibfield{author}{\bibinfo{person}{Jan Peters} {and} \bibinfo{person}{Stefan
  Schaal}.} \bibinfo{year}{2008}\natexlab{}.
\newblock \showarticletitle{Natural Actor-critic}.
\newblock \bibinfo{journal}{\emph{Neurocomputing}} \bibinfo{volume}{71},
  \bibinfo{number}{7-9} (\bibinfo{year}{2008}), \bibinfo{pages}{1180–1190}.
\newblock


\bibitem[\protect\citeauthoryear{Pnueli}{Pnueli}{1977}]%
        {Pnueli1977}
\bibfield{author}{\bibinfo{person}{Amir Pnueli}.}
  \bibinfo{year}{1977}\natexlab{}.
\newblock \showarticletitle{The Temporal Logic of Programs}. In
  \bibinfo{booktitle}{\emph{Proceedings of the 18th Annual Symposium on
  Foundations of Computer Science}} \emph{(\bibinfo{series}{FOCS ’77})}.
  \bibinfo{publisher}{IEEE Computer Society}, \bibinfo{address}{USA},
  \bibinfo{pages}{46–57}.
\newblock


\bibitem[\protect\citeauthoryear{Prasad, L.A., and Bhatnagar}{Prasad
  et~al\mbox{.}}{2015}]%
        {Prasad2015}
\bibfield{author}{\bibinfo{person}{H.L. Prasad}, \bibinfo{person}{Prashanth
  L.A.}, {and} \bibinfo{person}{Shalabh Bhatnagar}.}
  \bibinfo{year}{2015}\natexlab{}.
\newblock In \bibinfo{booktitle}{\emph{Proceedings of the 2015 International
  Conference on Autonomous Agents and Multiagent Systems}} (Istanbul, Turkey)
  \emph{(\bibinfo{series}{AAMAS ’15})}. \bibinfo{publisher}{International
  Foundation for Autonomous Agents and Multiagent Systems},
  \bibinfo{address}{Richland, SC}, \bibinfo{pages}{1371–1379}.
\newblock
\showISBNx{9781450334136}


\bibitem[\protect\citeauthoryear{Qu, Wierman, and Li}{Qu et~al\mbox{.}}{2020}]%
        {Qu2020}
\bibfield{author}{\bibinfo{person}{Guannan Qu}, \bibinfo{person}{Adam Wierman},
  {and} \bibinfo{person}{Na Li}.} \bibinfo{year}{2020}\natexlab{}.
\newblock \showarticletitle{Scalable Reinforcement Learning of Localized
  Policies for Multi-agent Networked Systems}.
\newblock \bibinfo{journal}{\emph{arXiv:1912.02906}} (\bibinfo{year}{2020}).
\newblock


\bibitem[\protect\citeauthoryear{Rentmeesters, Tsai, and Lin}{Rentmeesters
  et~al\mbox{.}}{1996}]%
        {Rentmeestersa1996}
\bibfield{author}{\bibinfo{person}{Mark~J. Rentmeesters},
  \bibinfo{person}{Wei~K. Tsai}, {and} \bibinfo{person}{Kwei-Jay Lin}.}
  \bibinfo{year}{1996}\natexlab{}.
\newblock \showarticletitle{A Theory of Lexicographic Multi-criteria
  Optimization}. In \bibinfo{booktitle}{\emph{Proceedings of ICECCS 1996: 2nd
  IEEE International Conference on Engineering of Complex Computer Systems}}.
  \bibinfo{publisher}{IEEE Comput. Soc. Press}.
\newblock


\bibitem[\protect\citeauthoryear{Sadigh, Kim, Coogan, Sastry, and
  Seshia}{Sadigh et~al\mbox{.}}{2014}]%
        {Sadigh2014}
\bibfield{author}{\bibinfo{person}{Dorsa Sadigh}, \bibinfo{person}{Eric~S.
  Kim}, \bibinfo{person}{Samuel Coogan}, \bibinfo{person}{S.~Shankar Sastry},
  {and} \bibinfo{person}{Sanjit~A. Seshia}.} \bibinfo{year}{2014}\natexlab{}.
\newblock \showarticletitle{A Learning Based Approach to Control Synthesis of
  Markov Decision Processes for Linear Temporal Logic Specifications}. In
  \bibinfo{booktitle}{\emph{53rd IEEE Conference on Decision and Control, CDC
  2014, Los Angeles, CA, USA, December 15-17, 2014}}.
  \bibinfo{pages}{1091–1096}.
\newblock


\bibitem[\protect\citeauthoryear{Sickert, Esparza, Jaax, and
  Křetínský}{Sickert et~al\mbox{.}}{2016}]%
        {Sickert2016}
\bibfield{author}{\bibinfo{person}{Salomon Sickert}, \bibinfo{person}{Javier
  Esparza}, \bibinfo{person}{Stefan Jaax}, {and} \bibinfo{person}{Jan
  Křetínský}.} \bibinfo{year}{2016}\natexlab{}.
\newblock \showarticletitle{Limit-deterministic Büchi Automata for Linear
  Temporal Logic}.
\newblock In \bibinfo{booktitle}{\emph{Computer Aided Verification}}.
  \bibinfo{publisher}{Springer International Publishing},
  \bibinfo{pages}{312–332}.
\newblock


\bibitem[\protect\citeauthoryear{Skalse, Hammond, and Abate}{Skalse
  et~al\mbox{.}}{2021}]%
        {Skalse2020a}
\bibfield{author}{\bibinfo{person}{Joar Skalse}, \bibinfo{person}{Lewis
  Hammond}, {and} \bibinfo{person}{Alessandro Abate}.}
  \bibinfo{year}{2021}\natexlab{}.
\newblock \showarticletitle{Lexicographic Multi-objective Reinforcement
  Learning}.
\newblock \bibinfo{journal}{\emph{Forthcoming}} (\bibinfo{year}{2021}).
\newblock


\bibitem[\protect\citeauthoryear{Sutton}{Sutton}{1988}]%
        {Sutton1988}
\bibfield{author}{\bibinfo{person}{Richard~S. Sutton}.}
  \bibinfo{year}{1988}\natexlab{}.
\newblock \showarticletitle{Learning to Predict by the Methods of Temporal
  Differences}.
\newblock \bibinfo{journal}{\emph{Machine Learning}} \bibinfo{volume}{3},
  \bibinfo{number}{1} (\bibinfo{year}{1988}), \bibinfo{pages}{9–44}.
\newblock


\bibitem[\protect\citeauthoryear{Sutton and Barto}{Sutton and Barto}{2018}]%
        {Sutton2018}
\bibfield{author}{\bibinfo{person}{Richard~S. Sutton} {and}
  \bibinfo{person}{Andrew~G. Barto}.} \bibinfo{year}{2018}\natexlab{}.
\newblock \bibinfo{booktitle}{\emph{Reinforcement Learning}}.
\newblock \bibinfo{publisher}{The MIT Press}.
\newblock
\showISBNx{0262039249}


\bibitem[\protect\citeauthoryear{Sutton, McAllester, Singh, and Mansour}{Sutton
  et~al\mbox{.}}{1999}]%
        {Sutton1999}
\bibfield{author}{\bibinfo{person}{Richard~S. Sutton}, \bibinfo{person}{David
  McAllester}, \bibinfo{person}{Satinder Singh}, {and} \bibinfo{person}{Yishay
  Mansour}.} \bibinfo{year}{1999}\natexlab{}.
\newblock \showarticletitle{Policy Gradient Methods for Reinforcement Learning
  with Function Approximation}. In \bibinfo{booktitle}{\emph{Proceedings of the
  12th International Conference on Neural Information Processing Systems}}
  (Denver, CO) \emph{(\bibinfo{series}{NIPS’99})}. \bibinfo{publisher}{MIT
  Press}, \bibinfo{address}{Cambridge, MA, USA}, \bibinfo{pages}{1057–1063}.
\newblock


\bibitem[\protect\citeauthoryear{Thomas}{Thomas}{2014}]%
        {Thomas2014}
\bibfield{author}{\bibinfo{person}{Philip~S. Thomas}.}
  \bibinfo{year}{2014}\natexlab{}.
\newblock \showarticletitle{Bias in Natural Actor-critic Algorithms}. In
  \bibinfo{booktitle}{\emph{Proceedings of the 31st International Conference on
  International Conference on Machine Learning - Volume 32}}
  \emph{(\bibinfo{series}{ICML’14})}. \bibinfo{publisher}{JMLR.org},
  \bibinfo{address}{Beijing, China}, \bibinfo{pages}{I–441–I–448}.
\newblock


\bibitem[\protect\citeauthoryear{Toro~Icarte, Klassen, Valenzano, and
  McIlraith}{Toro~Icarte et~al\mbox{.}}{2018a}]%
        {ToroIcarte2018}
\bibfield{author}{\bibinfo{person}{Rodrigo Toro~Icarte}, \bibinfo{person}{Toryn
  Klassen}, \bibinfo{person}{Richard Valenzano}, {and} \bibinfo{person}{Sheila
  McIlraith}.} \bibinfo{year}{2018}\natexlab{a}.
\newblock \showarticletitle{Using Reward Machines for High-level Task
  Specification and Decomposition in Reinforcement Learning}. In
  \bibinfo{booktitle}{\emph{Proceedings of the 35th International Conference on
  Machine Learning}} \emph{(\bibinfo{series}{Proceedings of Machine Learning
  Research}, Vol.~\bibinfo{volume}{80})},
  \bibfield{editor}{\bibinfo{person}{Jennifer Dy} {and}
  \bibinfo{person}{Andreas Krause}} (Eds.). \bibinfo{publisher}{PMLR},
  \bibinfo{address}{Stockholmsmässan, Stockholm Sweden},
  \bibinfo{pages}{2107–2116}.
\newblock


\bibitem[\protect\citeauthoryear{Toro~Icarte, Klassen, Valenzano, and
  McIlraith}{Toro~Icarte et~al\mbox{.}}{2018b}]%
        {ToroIcarte2018a}
\bibfield{author}{\bibinfo{person}{Rodrigo Toro~Icarte},
  \bibinfo{person}{Toryn~Q. Klassen}, \bibinfo{person}{Richard Valenzano},
  {and} \bibinfo{person}{Sheila~A. McIlraith}.}
  \bibinfo{year}{2018}\natexlab{b}.
\newblock \showarticletitle{Teaching Multiple Tasks to an Rl Agent Using Ltl}.
  In \bibinfo{booktitle}{\emph{Proceedings of the 17th International Conference
  on Autonomous Agents and MultiAgent Systems}} (Stockholm, Sweden)
  \emph{(\bibinfo{series}{AAMAS ’18})}. \bibinfo{publisher}{International
  Foundation for Autonomous Agents and Multiagent Systems},
  \bibinfo{address}{Richland, SC}, \bibinfo{pages}{452–461}.
\newblock

\vfill\eject

\bibitem[\protect\citeauthoryear{Tsitsiklis and Roy}{Tsitsiklis and
  Roy}{1997}]%
        {Tsitsiklis1997}
\bibfield{author}{\bibinfo{person}{John~N. Tsitsiklis} {and}
  \bibinfo{person}{Benjamin~Van Roy}.} \bibinfo{year}{1997}\natexlab{}.
\newblock \showarticletitle{An Analysis of Temporal-difference Learning with
  Function Approximation}.
\newblock \bibinfo{journal}{\emph{IEEE Trans. Automat. Control}}
  \bibinfo{volume}{42}, \bibinfo{number}{5} (\bibinfo{year}{1997}),
  \bibinfo{pages}{674–690}.
\newblock


\bibitem[\protect\citeauthoryear{Wang and Sandholm}{Wang and Sandholm}{2002}]%
        {Wang2002}
\bibfield{author}{\bibinfo{person}{Xiaofeng Wang} {and} \bibinfo{person}{Tuomas
  Sandholm}.} \bibinfo{year}{2002}\natexlab{}.
\newblock \showarticletitle{Reinforcement Learning to Play an Optimal Nash
  Equilibrium in Team Markov Games}. In \bibinfo{booktitle}{\emph{Proceedings
  of the 15th International Conference on Neural Information Processing
  Systems}} \emph{(\bibinfo{series}{NIPS'02})}. \bibinfo{publisher}{MIT Press},
  \bibinfo{address}{Cambridge, MA, USA}, \bibinfo{pages}{1603–1610}.
\newblock


\bibitem[\protect\citeauthoryear{Wen and Topcu}{Wen and Topcu}{2016}]%
        {Wen2016}
\bibfield{author}{\bibinfo{person}{Min Wen} {and} \bibinfo{person}{Ufuk
  Topcu}.} \bibinfo{year}{2016}\natexlab{}.
\newblock \showarticletitle{Probably Approximately Correct Learning in
  Stochastic Games with Temporal Logic Specifications}. In
  \bibinfo{booktitle}{\emph{Proceedings of the Twenty-Fifth International Joint
  Conference on Artificial Intelligence}}
  \emph{(\bibinfo{series}{IJCAI’16})}. \bibinfo{publisher}{AAAI Press},
  \bibinfo{address}{New York, New York, USA}, \bibinfo{pages}{3630–3636}.
\newblock
\showISBNx{9781577357704}


\bibitem[\protect\citeauthoryear{Wooldridge, Gutierrez, Harrenstein, Marchioni,
  Perelli, and Toumi}{Wooldridge et~al\mbox{.}}{2016}]%
        {Wooldridge2016}
\bibfield{author}{\bibinfo{person}{Michael Wooldridge}, \bibinfo{person}{Julian
  Gutierrez}, \bibinfo{person}{Paul Harrenstein}, \bibinfo{person}{Enrico
  Marchioni}, \bibinfo{person}{Giuseppe Perelli}, {and} \bibinfo{person}{Alexis
  Toumi}.} \bibinfo{year}{2016}\natexlab{}.
\newblock \showarticletitle{Rational Verification: From Model Checking to
  Equilibrium Checking}. In \bibinfo{booktitle}{\emph{Proceedings of the
  Thirtieth AAAI Conference on Artificial Intelligence}}
  \emph{(\bibinfo{series}{AAAI’16})}. \bibinfo{publisher}{AAAI Press},
  \bibinfo{address}{Phoenix, Arizona}, \bibinfo{pages}{4184–4190}.
\newblock


\bibitem[\protect\citeauthoryear{Zhang, Yang, and Başar}{Zhang
  et~al\mbox{.}}{2019}]%
        {Zhang2019}
\bibfield{author}{\bibinfo{person}{Kaiqing Zhang}, \bibinfo{person}{Zhuoran
  Yang}, {and} \bibinfo{person}{Tamer Başar}.}
  \bibinfo{year}{2019}\natexlab{}.
\newblock \showarticletitle{Multi-agent Reinforcement Learning: A Selective
  Overview of Theories and Algorithms}.
\newblock \bibinfo{journal}{\emph{arXiv:1911.10635}} (\bibinfo{year}{2019}).
\newblock


\bibitem[\protect\citeauthoryear{Zhang, Yang, Liu, Zhang, and Basar}{Zhang
  et~al\mbox{.}}{2018}]%
        {Zhang2018}
\bibfield{author}{\bibinfo{person}{Kaiqing Zhang}, \bibinfo{person}{Zhuoran
  Yang}, \bibinfo{person}{Han Liu}, \bibinfo{person}{Tong Zhang}, {and}
  \bibinfo{person}{Tamer Basar}.} \bibinfo{year}{2018}\natexlab{}.
\newblock \showarticletitle{Fully Decentralized Multi-agent Reinforcement
  Learning with Networked Agents}. In \bibinfo{booktitle}{\emph{Proceedings of
  the 35th International Conference on Machine Learning}}
  \emph{(\bibinfo{series}{Proceedings of Machine Learning Research},
  Vol.~\bibinfo{volume}{80})}, \bibfield{editor}{\bibinfo{person}{Jennifer Dy}
  {and} \bibinfo{person}{Andreas Krause}} (Eds.). \bibinfo{publisher}{PMLR},
  \bibinfo{address}{Stockholmsmässan, Stockholm Sweden},
  \bibinfo{pages}{5872–5881}.
\newblock


\bibitem[\protect\citeauthoryear{Zinkevich, Greenwald, and Littman}{Zinkevich
  et~al\mbox{.}}{2005}]%
        {Zinkevich2005}
\bibfield{author}{\bibinfo{person}{Martin Zinkevich}, \bibinfo{person}{Amy
  Greenwald}, {and} \bibinfo{person}{Michael~L. Littman}.}
  \bibinfo{year}{2005}\natexlab{}.
\newblock \showarticletitle{Cyclic Equilibria in Markov Games}. In
  \bibinfo{booktitle}{\emph{Proceedings of the 18th International Conference on
  Neural Information Processing Systems}} (Vancouver, British Columbia, Canada)
  \emph{(\bibinfo{series}{NIPS’05})}. \bibinfo{publisher}{MIT Press},
  \bibinfo{address}{Cambridge, MA, USA}, \bibinfo{pages}{1641–1648}.
\newblock


\end{thebibliography}

%%%%%%%%%%%%%%%%%%%%%%%%%%%%%%%%%%%%%%%%%%%%%%%%%%%%%%%%%%%%%%%%%%%%%%%%

\renewcommand\thesection{\Alph{section}}
\setcounter{section}{0}
\setcounter{lemma}{0}
\setcounter{proposition}{0}
\setcounter{theorem}{0}
\setcounter{corollary}{0}

\onecolumn

\section{Proofs}
\label{proofs}

We begin by briefly restating our assumptions from the main paper, which we refer back to in our proofs:
\begin{enumerate}
    \item $S$ and $A$ are finite, and all reward functions are bounded.
    \item The Markov chain 
    induced by any $\theta$ is irreducible over $S^\otimes$. 
    \item $\pi^i(a^i \vert s^\otimes ;\theta^i)$ is continuously differentiable $\forall i, s^\otimes, a^i$
    \item Let $\Phi$ be the $\vert S^\otimes \vert \times c$ matrix with rows $\phi(s^\otimes)$. Then $\Phi$ has full rank, $c \leq \vert S \vert$, and $\nexists w \in W$ such that $\Phi w = 1$.
    \item $\mathbb{E}_t \big[ L^i_{V^j}(x^{i*}_t; \theta_t, \nu^j_* ) \big] \leq e^j_{approx}$, where $e^j_{approx}$ is some constant, thus $\mathbb{E}_t \big[ L^i_{V}(x^{i*}_t; \theta_t, \nu_{*} ) \big] \leq e_{approx} \coloneqq \sum_j w[j] e^j_{approx}$.
    \item $\exists \sigma < \infty$ s.t. $\log \pi^i(a^i \vert s^\otimes;\theta)$ is a $\sigma$\emph{-smooth} in $\theta^i$  $\forall i, s^\otimes, a^i$.
    \item The \emph{relative condition number} is finite.
    \item $\pi^i(\cdot \vert s^\otimes ; \theta^i)$ is initialised as the uniform distribution $\forall i, s^\otimes$.
\end{enumerate}
At the end of this section a small supporting lemma that was referred to in the main manuscript (in the proof sketch of Proposition \ref{Scalar2LTL}) but omitted due to space constraints is included for the sake of completeness.

\begin{proposition}
    % \label{Scalar2LTL}
    Given an MG $G$ and LTL objectives $\{\varphi^j\}_{1 \leq j \leq m}$ (each equivalent to an LDBA $B^j$), let $G_B = G \otimes B^1 \otimes \cdots \otimes B^m$ be the resulting product MG with newly defined reward functions $R^j_\otimes$ and state-dependent discount functions $\Gamma^j$ given by (2). Then there exists some $0 < \gamma_V < 1$ such that (1) is satisfied by the patient value function $V_\pi \coloneqq \sum_j w[j] V^j_\pi$.
\end{proposition}
\begin{proof}
    We begin by observing that:
    \begin{align*}
        V^{*}_{\pi}(s^\otimes) 
        &= \sum_j w[j] V^{j*}_{\pi}(s^\otimes)\\
        &= \sum_j w[j] \mathbb{E}_{\pi} \Big[ \sum^\infty_{t=0} \Gamma^j_{1:t} R^j_\otimes(s^\otimes_{t+1}) ~\Big\vert~ s^\otimes \Big]\\
        &= \sum_j w[j] \sum_\rho \Big[Pr^\pi_{G_B}(\rho \vert s^\otimes) \sum^\infty_{t=0} \Gamma^j_{1:t} R^j_\otimes(s^\otimes_{t+1}) \Big]
    \end{align*} 
    We denote the number of times a path $\rho$ in $G_B$ passes through the accepting set $F^j$ of automaton $B^j$ by $F^j(\rho)$, and number of times $\rho$ passes through any accepting set by $F(\rho)= \sum_j F^j(\rho)$. Then clearly when $F^j(\rho) = \infty$ we have that $\sum^\infty_{t=0} \Gamma^j_{1:t} R^j_\otimes(s^\otimes_{t+1}) = \frac{1}{1 - \gamma_V}$ and when $F^j(\rho) = f^j$ for some finite number $f^j$ then we have that $\sum^\infty_{t=0} \Gamma^j_{1:t} R^j_\otimes(s^\otimes_{t+1}) = \frac{1 - \gamma_V^{f^j}}{1 - \gamma_V}$.
    Now consider any state $s^\otimes = (s, q^1_0, \ldots, q^m_0) \in S^\otimes$. We wish to show that there exists some $0 < \gamma_V < 1$ such that $\argmax_{\pi}V^*_{\pi}(s^\otimes) \subseteq \argmax_{\pi} \sum_j w[j] \Pr^{\pi}_G(s \models \varphi^j)$. Note that if $\sum_j w[j] \Pr^{\pi}_G(s \models \varphi^j) = 0$ for every joint strategy $\pi$ then there is nothing to prove, so suppose this is not the case.
    We prove the contrapositive. Suppose that $\pi \notin \argmax_{\pi} \sum_j w[j] \Pr^{\pi}_G(s \models \varphi^j)$. We show that there exists $0 < \gamma_V < 1$ such that $\pi \notin \argmax_{\pi}V^*_{\pi}(s^\otimes)$ by proving that for some $\pi' \in \argmax_{\pi} \sum_j w[j] \Pr^{\pi}_G(s \models \varphi^j)$, we have $V^*_{\pi'}(s^\otimes) > V^*_{\pi}(s^\otimes)$.
    Define the sets $\finn^j(s^\otimes) = \{\rho : \rho[0] = s^\otimes \wedge F^j(\rho) \neq \infty\}$ and $\inff^j(s^\otimes) = \{\rho : \rho[0] = s^\otimes \wedge F^j(\rho) = \infty\}$. Then we have:
    \begin{align*}
        V^*_{\pi'}(s^\otimes) &= \sum_j w[j] \sum_\rho \Big[Pr^{\pi'}_{G_B}(\rho \vert s^\otimes) \sum^\infty_{t=0} \Gamma^j_{1:t} R^j_\otimes(s^\otimes_{t+1}) \Big]\\
        &= \sum_j w[j] \bigg[\sum_{\rho \in \inff^j(s^\otimes)} Pr^{\pi'}_{G_B}(\rho \vert s^\otimes) 
        \sum^\infty_{t=0} \Gamma^j_{1:t} R^j_\otimes(s^\otimes_{t+1})
        + \sum_{\rho \in \finn^j(s^\otimes)} Pr^{\pi'}_{G_B}(\rho \vert s^\otimes) 
        \sum^\infty_{t=0} \Gamma^j_{1:t} R^j_\otimes(s^\otimes_{t+1})
        \bigg]\\
        &= \sum_j w[j] \Bigg[\sum_{\rho \in \inff^j(s^\otimes)} Pr^{\pi'}_{G_B}(\rho \vert s^\otimes) \frac{1}{1- \gamma_V}
        + \sum_{\rho \in \finn^j(s^\otimes)} Pr^{\pi'}_{G_B}(\rho \vert s^\otimes) \frac{1 - \gamma_V^{F^j(\rho)}}{1 - \gamma_V} \Bigg]\\
        &\geq \sum_j w[j] \bigg[\sum_{\rho \in \inff^j(s^\otimes)} Pr^{\pi'}_{G_B}(\rho \vert s^\otimes) \frac{1}{1- \gamma_V} \bigg]\\
        &= \frac{a^\top w}{1- \gamma_V}
    \end{align*}
    where $a$ is a vector such that $a[j] = \sum_{\rho \in \inff^j(s^\otimes)} Pr^{\pi'}_{G_B}(\rho \vert s^\otimes) = Pr^{\pi'}_{G_B}\big(\inff^j(s^\otimes)\big)$. Similarly, we have:
    \begin{align*}
        V^*_{\pi}(s^\otimes) &= \sum_j w[j] \sum_\rho \Big[Pr^{\pi}_{G_B}(\rho \vert s^\otimes) \sum^\infty_{t=0} \Gamma^j_{1:t} R^j_\otimes(s^\otimes_{t+1}) \Big]\\
        &= \frac{b^\top w}{1- \gamma_V} + \sum_j w[j] \Bigg[ \sum_{\rho \in \finn^j(s^\otimes)} Pr^{\pi}_{G_B}(\rho \vert s^\otimes) \frac{1 - \gamma_V^{F^j(\rho)}}{1 - \gamma_V} \Bigg]\\
        &\leq \frac{b^\top w}{1- \gamma_V} + \sum_j w[j] \bigg[ \sum_{\rho \in \finn^j(s^\otimes)} Pr^{\pi}_{G_B}(\rho \vert s^\otimes) \frac{1 - \gamma_V^{f}}{1 - \gamma_V} \bigg]\\
        &= \frac{b^\top w}{1- \gamma_V} + (1 - b)^\top w \frac{1 - \gamma_V^{f}}{1 - \gamma_V}
    \end{align*}
    where $b$ is a vector such that $b[j] = \sum_{\rho \in \inff^j(s^\otimes)} Pr^{\pi}_{G_B}(\rho \vert s^\otimes) = Pr^{\pi}_{G_B}(\inff^j(s^\otimes))$ and $f = \max_j f^j$ where $f^j = \max_{\rho \in \finn^j(s^\otimes)}F^j(\rho)$. In the calculation above $(1 - b)$ represents vector subtraction.
    Notice that as $\pi \notin \argmax_{\pi} \sum_j w[j] \Pr^{\pi}_G(s \models \varphi^j)$ but $\pi' \in \argmax_{\pi} \sum_j w[j] \Pr^{\pi}_G(s \models \varphi^j)$ then we have $\sum_j w[j] \Pr^{\pi'}_G(s \models \varphi^j) > \sum_j w[j] \Pr^{\pi}_G(s \models \varphi^j)$ and hence, by a straightforward extension of Lemma \ref{LTL2Buchi}:
    $$1 \geq a^\top w = \sum_j w[j] Pr^{\pi'}_{G_B}(\inff^j(s^\otimes)) > \sum_j w[j] Pr^{\pi}_{G_B}(\inff^j(s^\otimes)) =  b^\top w \geq 0$$
    This in turn implies that $1 > \frac{1 - a^\top w}{1 - b^\top w} \geq 0$. Letting $\gamma_V > \sqrt[\uproot{5}f]{\frac{1-a^\top w}{1-b^\top w}}$ then we see that $a^\top w > 1 - (1-b^\top w)\gamma_V^f$ and therefore:
    $$V^*_{\pi'}(s^\otimes) \geq
    \frac{a^\top w}{1 - \gamma_V} = 
    \frac{1}{1 - \gamma_V}(a^\top w) > 
    \frac{1}{1 - \gamma_V}(1 - (1 - b^\top w)\gamma_V^f) =
    \frac{b^\top w}{1 - \gamma_V} + (1-b^\top w)\frac{1 - \gamma_V^f}{1 - \gamma_V} \geq
    V^*_{\pi}(s^\otimes)$$
\end{proof}
\begin{lemma}
    Let $G_B$ be some (product) MG. Then for any set of parameters $\{\theta^i\}_{i \in N}$ and any player $i \in N$, the natural policy gradient for player $i$ with respect to each $J^j(\theta) = \sum_{s^\otimes} \zeta^\otimes(s^\otimes) V^j_\theta(s^\otimes)$ is given by $\tilde{\nabla}_{\theta^i} J^j(\theta) = x^i_{V^j}$ where $x^i_{V^j}$ is a parameter satisfying $\psi^i_{\theta^i}(a^i\vert s^\otimes)^\top x^i_{V^j} = \nabla_{\theta^i} \log\pi^i(a^i\vert s^\otimes;\theta)^\top x^i_{V^j} = A^j_\theta(s^\otimes, a)$. 
\end{lemma}
\begin{proof}
    The proof follows along similar lines to those found in earlier work \cite{Sutton1999,Peters2008}. We begin by using the policy gradient theorem to derive the vanilla policy gradient with respect to $\theta^i$. We begin by noting that the action-value function $Q^j_\theta$ in the policy gradient can be replaced by the advantage function $A^j_\theta$. Then the following holds:
    \begin{align*}
        \nabla_{\theta^i} J^j(\theta)
        &= \sum_{s^\otimes} d^j_\zeta(s^\otimes ; \theta) \sum_a A^j_\theta(s^\otimes, a) \nabla_{\theta^i} \pi(a \vert s^\otimes; \theta)\\
        &= \sum_{s^\otimes} d^j_\zeta(s^\otimes ; \theta) \sum_a \pi(a \vert s^\otimes; \theta) A^j_\theta(s^\otimes, a)  \frac{\nabla_{\theta^i} \pi(a \vert s^\otimes; \theta)}{\pi(a \vert s^\otimes; \theta)}\\
        &= \sum_{s^\otimes} d^j_\zeta(s^\otimes ; \theta) \sum_a \pi(a \vert s^\otimes; \theta) A^j_\theta(s^\otimes, a)  \nabla_{\theta^i} \log \pi(a \vert s^\otimes; \theta)\\
        &= \sum_{s^\otimes} d^j_\zeta(s^\otimes ; \theta) \sum_a \pi(a \vert s^\otimes; \theta) A^j_\theta(s^\otimes, a) \Big[ \nabla_{\theta^i} \sum_{i\in N} \log \pi^i(a^i \vert s^\otimes; \theta^i) \Big]\\
        &= \sum_{s^\otimes} d^j_\zeta(s^\otimes ; \theta) \sum_a \pi(a \vert s^\otimes; \theta) A^j_\theta(s^\otimes, a) \psi^i_{\theta^i}(a^i\vert s^\otimes)
    \end{align*}
    Now let us substitute into the above using our assumption that $\psi^i_{\theta^i}(a^i\vert s^\otimes)^\top x^i = A^j_\theta(s^\otimes, a)$ to obtain the following:
    $$\sum_{s^\otimes} d^j_\zeta(s^\otimes) \sum_a \pi(a \vert s^\otimes; \theta) \psi^i_{\theta^i}(a^i\vert s^\otimes) \psi^i_{\theta^i}(a^i\vert s^\otimes)^\top x^i = F^j(\theta^i) x^i$$
    The natural policy gradient of $J^j$ with respect to $\theta^i$ is given by $\tilde{\nabla}_{\theta^i} J^j(\theta) =  G^j(\theta^i)^{-1} \nabla_{\theta^i} J^j(\theta)$ where $G^j(\theta^i)$ is the Fisher information matrix. Thus to complete the proof it suffices to show that $F^j(\theta^i) = G^j(\theta^i)$. Recall our induced Markov chain $\Pr^\theta_{G_B}(\cdot\vert s^\otimes)$ over the states in $G_B$ when starting from state $s^\otimes$ and, simplifying notation, let $\Pr(\rho) = \Pr^\theta_{G_B}(\rho\vert s^\otimes)\zeta^\otimes(s^\otimes)$. 
    We begin by noting that for any probability distribution $p(y;z)$ parameterised by $z$ that the following standard result holds when differentiating $\sum_y p(y) = 1$ twice with respect to $z$:
    $$\sum_y p(y) \nabla^2_{z} \log p(y) = - \sum_y p(y) \nabla_{z} \log p(y) (\nabla_{z} \log p(y))^\top$$
    In particular, this is the case for both $\Pr$ and each $\pi_{\theta}(\cdot\vert s^\otimes)$. Next, observe that:
    $$\Pr(\rho) = \zeta^\otimes(\rho[0]) \prod^\infty_{t=0} T^\otimes (\rho[t+1]\vert\rho[t],a) \pi_\theta(a \vert \rho[t])$$
    Taking logs and then differentiating twice with respect to $\theta^i$ we have:
    \begin{align*}
        \nabla^2_{\theta^i} \log \Pr(\rho) 
        &= \nabla^2_{\theta^i} \log \Big[ \zeta^\otimes(\rho[0]) \prod^\infty_{t=0} T^\otimes (\rho[t+1]\vert\rho[t],a) \pi_\theta(a \vert \rho[t]) \Big]\\
        &= \nabla^2_{\theta^i} \Big[ \log \zeta^\otimes(\rho[0]) + \sum^\infty_{t=0} [ \log T^\otimes (\rho[t+1]\vert\rho[t],a) + \log \pi_\theta(a \vert \rho[t])] \Big]\\
        &= \nabla^2_{\theta^i} \sum^\infty_{t=0} \log \pi_\theta(a \vert \rho[t])
    \end{align*}
    Then the final sequence of equalities follows straightforwardly from the results above:
    \begin{align*}
        G^j(\theta^i) 
        &= \sum_\rho \Pr(\rho) \nabla_{\theta^i} \log \Pr(\rho) (\nabla_{\theta^i} \log \Pr(\rho))^\top\\
        &= - \sum_\rho \Pr(\rho) \nabla^2_{\theta^i} \log \Pr(\rho)\\
        &= - \sum_\rho \Pr(\rho) \nabla^2_{\theta^i} \sum^\infty_{t=0} \log \pi_\theta(a \vert \rho[t])\\
        &= - \sum_{s^\otimes} d^j_\zeta(s^\otimes) \sum_a \pi(a \vert s^\otimes; \theta) \nabla^2_{\theta^i} \log \pi_\theta(a \vert s^\otimes)\\
        &= \sum_{s^\otimes} d^j_\zeta(s^\otimes) \sum_a \pi(a \vert s^\otimes; \theta) \nabla_{\theta^i} \log \pi_\theta(a \vert s^\otimes) (\nabla_{\theta^i} \log \pi_\theta(a \vert s^\otimes))^\top\\
        & = \sum_{s^\otimes} d^j_\zeta(s^\otimes) \sum_a \pi(a \vert s^\otimes; \theta) \nabla_{\theta^i} \sum_i \log \pi^i_{\theta^i}(a^i \vert s^\otimes) (\nabla_{\theta^i} \sum_i \log \pi^i_{\theta^i}(a^i \vert s^\otimes))^\top\\
        & = \sum_{s^\otimes} d^j_\zeta(s^\otimes) \sum_a \pi(a \vert s^\otimes; \theta) \psi^i_{\theta^i}(a^i\vert s^\otimes) \psi^i_{\theta^i}(a^i\vert s^\otimes)^\top\\ 
        &= F^j(\theta^i) 
    \end{align*}
\end{proof}
\begin{lemma}
    % \label{performancedifference}
    Suppose that $V_\theta(s^\otimes) \geq V_{\theta'}(s^\otimes)$ for some state $s^\otimes$ and two policies $\pi$ and $\pi'$ parametrised by $\theta$ and $\theta'$ respectively. Then we have:
    $$V_\theta(s^\otimes) - V_{\theta'}(s^\otimes) \leq \sum_j w[j] \bigg( \mathbb{E}_{\rho} \Big[\sum^\infty_{t=0} \Gamma^j_{0:t} A^j_{\theta'}(s^\otimes_{t},a_{t}) ~\Big\vert~ F^j(\rho) = \infty \Big] \bigg)$$
\end{lemma}
\begin{proof}
    Note that for any trajectory $\rho$ we have $\rho[0] = s^\otimes_0 = s^\otimes$. To ease notation we write $r^j_t = R^j(s^\otimes_t)$ and $\mathbb{E}_{\rho}$ instead of $\mathbb{E}_{\rho \sim \Pr^\theta_{G_B} (\cdot\vert s^\otimes)}$. As $V_\theta(s^\otimes)$ is bounded above by $\frac{1}{1 - \gamma_V}$ then, by our assumption that $V_\theta(s^\otimes) \geq V_{\theta'}(s^\otimes)$, we have:
    \begin{equation}
        \label{helper1}
        V_\theta(s^\otimes) - V_{\theta'}(s^\otimes) \leq \frac{1}{1 - \gamma_V} - V_{\theta'}(s^\otimes)
    \end{equation}
    Next, observe that for each specification $\varphi^j$ we have the following equality due to the definitions of our reward and state-dependent discount functions:
    \begin{equation}
        \label{helper2}
        \mathbb{E}_{\rho} \Big[\sum^{\infty}_{t=0} \Gamma^j_{0:t} r^j_{t+1} ~\Big\vert~ F^j(\rho) = \infty \Big] = \frac{1}{1 - \gamma_V}
    \end{equation}
    Using these two facts, we complete the proof as follows:
    \begin{align}
        &V_\theta(s^\otimes) - V_{\theta'}(s^\otimes) \nonumber\\
        &\leq \frac{1}{1 - \gamma_V} - \sum_j w[j] V^j_{\theta'}(s^\otimes) \label{eq:1}\\
        &= \sum_j w[j] \bigg( \frac{1}{1 - \gamma_V} - V^j_{\theta'}(s^\otimes) \bigg) \nonumber\\
        &= \sum_j w[j] \bigg( \mathbb{E}_{\rho} \Big[\sum^\infty_{t=0} \Gamma^j_{0:t} r^j_{t+1} ~\Big\vert~ F^j(\rho) = \infty \Big] - V^j_{\theta'}(s^\otimes) \bigg) \label{eq:2}\\
        &= \sum_j w[j] \bigg( \mathbb{E}_{\rho} \Big[ \Gamma^j_{0:0} r^j_1 + \sum^\infty_{t=1} \Gamma^j_{0:t} r^j_{t+1} ~\Big\vert~ F^j(\rho) = \infty \Big] - V^j_{\theta'}(s^\otimes) \bigg) \nonumber \\
        &= \sum_j w[j] \bigg( \mathbb{E}_{\rho} \Big[ \Gamma^j_{0:0} r^j_1 + \sum^\infty_{t=1} \Gamma^j_{0:t} \big(r^j_{t+1} + V^j_{\theta'}(s^\otimes_t) - V^j_{\theta'}(s^\otimes_t) \big) ~\Big\vert~ F^j(\rho) = \infty \Big] - V^j_{\theta'}(s^\otimes) \bigg) \nonumber \\
        &= \sum_j w[j] \bigg( \mathbb{E}_{\rho} \Big[ \Gamma^j_{0:0} \big( r^j_1 - V^j_{\theta'}(s^\otimes_0) \big) + \sum^\infty_{t=1} \Gamma^j_{0:t} \big(r^j_{t+1} + V^j_{\theta'}(s^\otimes_t) - V^j_{\theta'}(s^\otimes_t) \big) ~\Big\vert~ F^j(\rho) = \infty \Big] \bigg) \label{eq:3}\\
        &= \sum_j w[j] \bigg( \mathbb{E}_{\rho} \Big[\sum^\infty_{t=0} \Gamma^j_{0:t} \big(r^j_{t+1} - V^j_{\theta'}(s^\otimes_t) \big) + \sum^\infty_{t=1} \Gamma^j_{0:t} V^j_{\theta'}(s^\otimes_t) ~\Big\vert~ F^j(\rho) = \infty \Big] \bigg) \nonumber \\
        &= \sum_j w[j] \bigg( \mathbb{E}_{\rho} \Big[\sum^\infty_{t=0} \Gamma^j_{0:t} \big(r^j_{t+1} - V^j_{\theta'}(s^\otimes_t) \big) + \sum^{F^j(\rho)-1}_{t=0} \Gamma^j_{0:t} \Gamma^j(s^\otimes_{t+1}) V^j_{\theta'}(s^\otimes_{t+1}) ~\Big\vert~ F^j(\rho) = \infty \Big] \bigg) \nonumber \\
        &\leq \sum_j w[j] \bigg( \mathbb{E}_{\rho} \Big[\sum^\infty_{t=0} \Gamma^j_{0:t} \big(r^j_{t+1} - V^j_{\theta'}(s^\otimes_t) \big) + \sum^\infty_{t=0} \Gamma^j_{0:t} \Gamma^j(s^\otimes_{t+1}) V^j_{\theta'}(s^\otimes_{t+1}) ~\Big\vert~ F^j(\rho) = \infty \Big] \bigg) \label{eq:4}\\
        &= \sum_j w[j] \bigg( \mathbb{E}_{\rho} \Big[\sum^\infty_{t=0} \Gamma^j_{0:t} \big(r^j_{t+1} - V^j_{\theta'}(s^\otimes_t) + \Gamma^j(s^\otimes_{t+1}) V^j_{\theta'}(s^\otimes_{t+1}) \big) ~\Big\vert~ F^j(\rho) = \infty \Big] \bigg) \nonumber\\
        &= \sum_j w[j] \bigg( \mathbb{E}_{\rho} \Big[\sum^\infty_{t=0} \Gamma^j_{0:t} \big(r^j_{t+1} - V^j_{\theta'}(s^\otimes_t) + \Gamma^j(s^\otimes_{t+1}) \mathbb{E}[ V^j_{\theta'}(s^\otimes_{t+1}) \vert s^\otimes_t, a_t] \big) ~\Big\vert~ F^j(\rho) = \infty \Big] \bigg) \label{eq:5}\\
        &= \sum_j w[j] \bigg( \mathbb{E}_{\rho} \Big[\sum^\infty_{t=0} \Gamma^j_{0:t} A^j_{\theta'}(s^\otimes_{t},a_{t}) ~\Big\vert~ F^j(\rho) = \infty \Big] \bigg) \nonumber
    \end{align}
    In the deduction above: (\ref{eq:1}) follows from the inequality (\ref{helper1}) and the definition of $V_\theta$; (\ref{eq:2}) follows from (\ref{helper2}); (\ref{eq:3}) uses the facts that $\Gamma^j_{0:0} = 1$, that $V^j_{\theta'}(s^\otimes) = V^j_{\theta'}(s^\otimes_0)$ is a constant, and that each path $\rho$ where $F^j(\rho) = \infty$ is such that $\rho[0] = s^\otimes_0$; (\ref{eq:4}) follows from the fact that $\Gamma^j_{0:t+1} V^j_{\theta'}(s^\otimes_{t+1}) \geq 0$ for any $t$; and (\ref{eq:5}) follows from the tower property of conditional expectations.
\end{proof}
\begin{lemma}
    % \label{noregret}
    Consider a sequence of natural gradient updates $\{x^i_t\}_{0\leq t \leq T}$ found by \textsc{Almanac} such that $\Vert x^i_t \Vert_2 \leq X$ for all $t$. Let us write $\iota_{0:T} = \sum^T_{t=0} \iota_t$, and recall that $F^j(\rho)$ is the number of times a path $\rho$ in $G_B$ passes through the accepting set $F^j$ of automaton $B^j$. Let us write $\mathbb{E}_{\rho^*}$ instead of $\mathbb{E}_{\rho \sim \Pr^{\theta_*}_{G_B} (\cdot\vert s^\otimes), s^\otimes \sim \zeta^\otimes}$ and define $e^j_t$ by:
    $$e^j_t = \mathbb{E}_{\rho^*} \Big[ \sum^\infty_{\tau=0} \Gamma^j_{0:\tau} \Big( A^j_{\theta_t}(s^\otimes_\tau,a_\tau) - \psi^i_{\theta^i_t}(a^i_\tau\vert s^\otimes_\tau)^\top x^i_t \Big) ~\Big\vert~ F^j(\rho) = \infty \Big]$$
    where $\tau$ indexes $\rho$, i.e., $\rho[\tau] = s^\otimes_\tau$. Then we have:
    $$V_{\theta_*}(s^\otimes) - \lim_{T \rightarrow \infty} \mathbb{E}_{t \sim \iota_T} \big[ V_{\theta_t}(s^\otimes) \big] = \lim_{T \rightarrow \infty} \mathbb{E}_{t \sim \iota_T} \Big[ \sum_j w[j] e^j_t \Big]$$
    where we define the distribution $\iota_T$ over $t$ with $\iota_T(t) = \frac{\iota_t}{\iota_{0:T}}$.
\end{lemma}
\begin{proof}
    Our proof follows that of Agarwal et al. with adjustments made for the use of our state-dependent discount rates $\Gamma^j$ and non-constant learning rate $\iota$ \cite{Agarwal2019}. In what follows we denote by $\theta^i_t$ the value of parameter $\theta^i$ at time $t$ and $\pi^i_t$ the corresponding policy. We first observe that if a function $f(x)$ is $\sigma$-smooth then by Taylor's theorem we have:
    $$\big\vert f(x') - f(x) - (x' - x)^\top \nabla_x f(x) \big\vert \leq \frac{\sigma}{2} \Vert x ' - x \Vert^2_2$$
    Thus, by condition 6 we have:
        $$\log \frac{\pi^i_{t+1}(a^i \vert s^\otimes)}{\pi^i_t(a^i \vert s^\otimes)}
        \geq (\theta^i_{t+1} - \theta^i_{t})^\top \nabla_{\theta^i} \log \pi^i_t(a^i \vert s^\otimes) - \frac{\sigma}{2}\Vert \theta^i_{t+1} - \theta^i_{t} \Vert^2_2 = \iota_t \psi^i_{\theta^i_t}(a^i\vert s^\otimes)^\top x^i_t - (\iota_t)^2 \frac{\sigma}{2}\Vert x^i_t \Vert^2_2$$
    Let $\theta_*$ be a set of optimal parameters, then for any $\theta_t$ and any state $s^\otimes$ we have $V_{\theta_*}(s^\otimes) \geq V_{\theta_t}(s^\otimes)$. Using the above inequality and Lemma \ref{performancedifference} we obtain the following, where we simplify notation by writing $\mathbb{E}_{\rho^\zeta_*}$ in place of $\mathbb{E}_{\rho \sim \Pr^{\theta_*}_{G_B} (\cdot\vert s^\otimes), s^\otimes \sim \zeta^\otimes}$:
    \begin{align*}
        &\sum_j w[j] \bigg( \mathbb{E}_{\rho^\zeta_*} \Big[ \sum^\infty_{\tau=0} \Gamma^j_{0:\tau} \log \frac{\pi^i_{t+1}(a^i_\tau \vert s^\otimes_\tau)}{\pi^i_t(a^i_\tau \vert s^\otimes_\tau)} ~\Big\vert~ F^j(\rho) = \infty \Big] \bigg)\\
        &\geq \sum_j w[j] \bigg( \mathbb{E}_{\rho^\zeta_*} \Big[ \sum^\infty_{\tau=0} \Gamma^j_{0:\tau} \Big( \iota_t \psi^i_{\theta^i_t}(a^i_\tau\vert s^\otimes_\tau)^\top x^i_t - (\iota_t)^2 \frac{\sigma}{2}\Vert x^i_t \Vert^2_2 \Big) ~\Big\vert~ F^j(\rho) = \infty \Big] \bigg)\\
        &= \sum_j w[j] \bigg( \iota_t \mathbb{E}_{\rho^\zeta_*} \Big[ \sum^\infty_{\tau=0} \Gamma^j_{0:\tau} \psi^i_{\theta^i_t}(a^i_\tau\vert s^\otimes_\tau)^\top x^i_t \Big] - (\iota_t)^2 \frac{\sigma}{2}\Vert x^i_t \Vert^2_2 \mathbb{E}_{\rho^\zeta_*} \Big[ \sum^\infty_{\tau=0} \Gamma^j_{0:\tau} ~\Big\vert~ F^j(\rho) = \infty \Big] \bigg)\\
        &= - \sum_j w[j] \bigg( \iota_t \mathbb{E}_{\rho^\zeta_*} \Big[ \sum^\infty_{\tau=0} \Gamma^j_{0:\tau} \Big( A^j_{\theta_t}(s^\otimes_\tau,a_\tau) - \psi^i_{\theta^i_t}(a^i_\tau\vert s^\otimes_\tau)^\top x^i_t \Big) ~\Big\vert~ F^j(\rho) = \infty \Big]\\
        &- \iota_t \mathbb{E}_{\rho^\zeta_*} \Big[ \sum^\infty_{\tau=0} \Gamma^j_{0:\tau} A^j_{\theta_t}(s^\otimes_\tau,a_\tau) ~\Big\vert~ F^j(\rho) = \infty \Big]
        + (\iota_t)^2 \frac{\sigma}{2}\Vert x^i_t \Vert^2_2 \mathbb{E}_{\rho^\zeta_*} \Big[ \sum^\infty_{\tau=0} \Gamma^j_{0:\tau} ~\Big\vert~ F^j(\rho) = \infty \Big] \bigg)\\
        &\geq - \sum_j w[j] \bigg( \iota_t \mathbb{E}_{\rho^\zeta_*} \Big[ \sum^\infty_{\tau=0} \Gamma^j_{0:\tau} \Big( A^j_{\theta_t}(s^\otimes_\tau,a_\tau) - \psi^i_{\theta^i_t}(a^i_\tau\vert s^\otimes_\tau)^\top x^i_t \Big) ~\Big\vert~ F^j(\rho) = \infty \Big]\\
        &+ (\iota_t)^2 \frac{\sigma}{2}\Vert x^i_t \Vert^2_2 \mathbb{E}_{\rho^\zeta_*} \Big[ \sum^\infty_{\tau=0} \Gamma^j_{0:\tau} ~\Big\vert~ F^j(\rho) = \infty \Big] \bigg)
        + \iota_t \big( V_{\theta_*}(s^\otimes) - V_{\theta_t}(s^\otimes) \big)
    \end{align*}
    Rearranging the inequality above we have:
    \begin{align*}
        \iota_t \big( V_{\theta_*}(s^\otimes) - V_{\theta_t}(s^\otimes) \big)
        &\leq \sum_j w[j] \bigg(
        (\iota_t)^2 \frac{\sigma}{2}\Vert x^i_t \Vert^2_2 \mathbb{E}_{\rho^\zeta_*} \Big[ \sum^\infty_{\tau=0} \Gamma^j_{0:\tau} ~\Big\vert~ F^j(\rho) = \infty \Big]\\
        &+ \mathbb{E}_{\rho^\zeta_*} \Big[ \sum^\infty_{\tau=0} \Gamma^j_{0:\tau} \log \frac{\pi^i_{t+1}(a^i_\tau \vert s^\otimes_\tau)}{\pi^i_t(a^i_\tau \vert s^\otimes_\tau)} ~\Big\vert~ F^j(\rho) = \infty \Big]\\
        &+ \iota_t \mathbb{E}_{\rho^\zeta_*} \Big[ \sum^\infty_{\tau=0} \Gamma^j_{0:\tau} \Big( A^j_{\theta_t}(s^\otimes_\tau,a_\tau) - \psi^i_{\theta^i_t}(a^i_\tau\vert s^\otimes_\tau)^\top x^i_t \Big) ~\Big\vert~ F^j(\rho) = \infty \Big] \bigg)
    \end{align*}
    Next, we sum our inequality above over timesteps $t$ between $0$ and $T$ and divide by $\iota_{0:T} = \sum^T_{t=0} \iota_t$:
    \begin{align*}
        &V_{\theta_*}(s^\otimes) - \frac{1}{\iota_{0:T}} \sum^T_{t=0}\iota_t V_{\theta_t}(s^\otimes)\\
        &\leq \frac{1}{\iota_{0:T}} \sum^T_{t=0} (\iota_t)^2 \sum_j w[j] 
         \frac{\sigma}{2}\Vert x^i_t \Vert^2_2 \mathbb{E}_{\rho^\zeta_*} \Big[ \sum^\infty_{\tau=0} \Gamma^j_{0:\tau} ~\Big\vert~ F^j(\rho) = \infty \Big] \\
        &+ \frac{1}{\iota_{0:T}} \sum^T_{t=0} \sum_j w[j] \mathbb{E}_{\rho^\zeta_*} \Big[ \sum^\infty_{\tau=0} \Gamma^j_{0:\tau} \log \frac{\pi^i_{t+1}(a^i_\tau \vert s^\otimes_\tau)}{\pi^i_t(a^i_\tau \vert s^\otimes_\tau)} ~\Big\vert~ F^j(\rho) = \infty \Big] \\
        &+ \frac{1}{\iota_{0:T}} \sum^T_{t=0}\iota_t \sum_j w[j] \mathbb{E}_{\rho^\zeta_*} \Big[ \sum^\infty_{\tau=0} \Gamma^j_{0:\tau} \Big( A^j_{\theta_t}(s^\otimes_\tau,a_\tau) - \psi^i_{\theta^i_t}(a^i_\tau\vert s^\otimes_\tau)^\top x^i_t \Big) ~\Big\vert~ F^j(\rho) = \infty \Big]
    \end{align*}
    We conclude the proof by showing that the first two of the three terms on the right side of the inequality tend to 0 as $T$ tends to infinity. We consider these two terms in order. Let $M^j_\rho(k)$ be the number of steps along trajectory $\rho$ between the $k^\text{th}$ reward and preceding reward (or simply from the timestep 0 if $k = 1$). Then we have the following:
    \begin{align*}
        \mathbb{E}_{\rho^\zeta_*} \Big[ \sum^\infty_{\tau=0} \Gamma^j_{0:\tau} ~\Big\vert~ F^j(\rho) = \infty \Big]
        &= \mathbb{E}_{\rho^\zeta_*} \Big[ \sum^{\infty}_{k=0} M^j_\rho(k) \gamma^k_V ~\Big\vert~  F^j(\rho) = \infty \Big]\\
        &= \sum^{\infty}_{k=0} \mathbb{E}_{\rho^\zeta_*} \big[ M^j_\rho(k) ~\big\vert~ F^j(\rho) = \infty \big] \gamma^k_V\\
        &\leq \sum^{\infty}_{k=0} M^j \gamma^k_V = \frac{M^j}{1 - \gamma_V}
    \end{align*}
    where $M^j = \max_k \mathbb{E}_{\rho^\zeta_*} \big[ M^j_\rho(k) ~\big\vert~ F^j(\rho) = \infty \big]$ is a finite number. Next, note that in \textsc{Almanac} we bound the size of $\Vert x^i_t \Vert_2$ by a constant, $X$. Using the above facts and our choice of discount factor $\iota$ which obeys $\sum^\infty_{t=0} \iota_t = \infty$ and $\sum^\infty_{t=0} (\iota_t)^2 < \infty$ then we have:
    \begin{align*}
        \lim_{T \rightarrow \infty} \frac{1}{\iota_{0:T}} \sum^T_{t=0} (\iota_t)^2 \sum_j w[j] 
        \frac{\sigma}{2}\Vert x^i_t \Vert^2_2 \mathbb{E}_{\rho^\zeta_*} \Big[ \sum^\infty_{\tau=0} \Gamma^j_{0:\tau} ~\Big\vert~ F^j(\rho) = \infty \Big]
        &\leq \lim_{T \rightarrow \infty} \frac{1}{\iota_{0:T}} \sum^T_{t=0} (\iota_t)^2  
        \frac{\sigma}{2}X^2 \frac{\sum_j w[j] M^j}{1 - \gamma_V}\\
        &= \lim_{T \rightarrow \infty} \frac{\sum^T_{t=0} (\iota_t)^2}{\sum^T_{t=0} \iota_t} \frac{\sigma}{2}X^2 \frac{\sum_j w[j] M^j}{1 - \gamma_V}\\
        &= 0
    \end{align*}
    Moving to the second term, by telescoping the terms summing over $t$ we have the following:
    \begin{align*}
        &\frac{1}{\iota_{0:T}} \sum^T_{t=0} \sum_j w[j] \mathbb{E}_{\rho^\zeta_*} \Big[ \sum^\infty_{\tau=0} \Gamma^j_{0:\tau} \log \frac{\pi^i_{t+1}(a^i_\tau \vert s^\otimes_\tau)}{\pi^i_t(a^i_\tau \vert s^\otimes_\tau)} ~\Big\vert~ F^j(\rho) = \infty \Big]\\
        &= \frac{1}{\iota_{0:T}} \sum_j w[j] \mathbb{E}_{\rho^\zeta_*} \Big[ \sum^\infty_{\tau=0} \Gamma^j_{0:\tau} \sum^T_{t=0} \log \frac{\pi^i_{t+1}(a^i_\tau \vert s^\otimes_\tau)}{\pi^i_t(a^i_\tau \vert s^\otimes_\tau)} ~\Big\vert~ F^j(\rho) = \infty \Big]\\
        &= \frac{1}{\iota_{0:T}} \sum_j w[j] \mathbb{E}_{\rho^\zeta_*} \Big[ \sum^\infty_{\tau=0} \Gamma^j_{0:\tau} \log \frac{\pi^i_{T+1}(a^i_\tau \vert s^\otimes_\tau)}{\pi^i_0(a^i_\tau \vert s^\otimes_\tau)} ~\Big\vert~ F^j(\rho) = \infty \Big]\\
        &\leq \frac{1}{\iota_{0:T}} \sum_j w[j] \mathbb{E}_{\rho^\zeta_*} \Big[ \sum^\infty_{\tau=0} \Gamma^j_{0:\tau} \log \frac{1}{\pi^i_0(a^i_\tau \vert s^\otimes_\tau)} ~\Big\vert~ F^j(\rho) = \infty \Big]\\
        &= \frac{1}{\iota_{0:T}} \sum_j w[j] \mathbb{E}_{\rho^\zeta_*} \Big[ \sum^{\infty}_{\tau=0} \Gamma^j_{0:\tau} \log \vert A^i \vert \Big]
    \end{align*}
    where the final line follows from condition 8. Given our previous argument about the boundedness of $\mathbb{E}_{\rho^\zeta_*} \Big[ \sum^\infty_{\tau=0} \Gamma^j_{0:\tau} ~\Big\vert~ F^j(\rho) = \infty \Big]$ and the fact that $\iota_{0:\infty} = \sum^\infty_{t=0} \iota_t = \infty$ then we conclude our proof by observing that:
    $$\lim_{T \rightarrow \infty} \frac{1}{\iota_{0:T}} \sum^T_{t=0} \sum_j w[j] \mathbb{E}_{\rho^\zeta_*} \Bigg[ \sum^\infty_{\tau=0} \Gamma^j_{0:\tau} \log \frac{\pi^i_{t+1}(a^i_\tau \vert s^\otimes_\tau)}{\pi^i_t(a^i_\tau \vert s^\otimes_\tau)} ~\Bigg\vert~ F^j(\rho) = \infty \Bigg]
    \leq 
    \lim_{T \rightarrow \infty} \frac{1}{\iota_{0:T}} \frac{\sum_j w[j] M^j}{1 - \gamma_V} \log \vert A^i \vert = 0$$
\end{proof}

\begin{theorem}
    % \label{opt}
    Given an MG $G$ and LTL objectives $\{\varphi^j\}_{1 \leq j \leq m}$ (each equivalent to an LDBA $B^j$), let $G_B = G \otimes B^1 \otimes \cdots \otimes B^m$ be the resulting product MG with newly defined reward functions $R^j_\otimes$ and state-dependent discount functions $\Gamma^j$. Assume that $\gamma_V$ satisfies Proposition \ref{Scalar2LTL}, that the learning rates $\alpha, \beta^V, \beta^U, \eta, \iota$ are as in (9) and that conditions 1--8 hold. Then if each agent $i$ uses \emph{local} (\emph{global}) parameters $\theta^i$ with \emph{local} policy $\pi^i_{\theta^i}$ (\emph{global} policy $\pi^i_{\theta^i} = \pi_\theta$) then as $T \rightarrow \infty$,  \textsc{Almanac} converges to within 
    $$\lim_{T \rightarrow \infty} \mathbb{E}_{t \sim \iota_T} \left[ \sum_j w[j] \sqrt{e^j_{approx}} \frac{M^j}{(1 - \gamma_V)P^j} \right]$$
    of a \emph{local} (\emph{global}) optimum of $\sum_j w[j] \Pr^{\pi}_G(s \models \varphi^j)$, where $P^j$ and $M^j$ are constants.
\end{theorem}

\begin{proof}
    Our proof proceeds via a multi-timescale stochastic approximation analysis and is asymptotic in nature. We consider the convergence of three quantities in turn: state values, natural gradients, and finally the policy. In each step of the proof, we also divide our attention between the local and global settings, where required.

    \textbf{Step 1.} Let us first consider the behaviour of the critic. Our proof of convergence follows that of Tsitsiklis and Van Roy \cite{Tsitsiklis1997}. Note that as the rewards $r^j_t$ and states $s^\otimes_t$ are commonly observed by all agents, and we assume they share learning rate parameters $\alpha^V$ and $\alpha^U$, then each agent's estimate of $v^j$ and $u^j$ will be identical at each timestep $t$. Note also that as $v^j$ and $u^j$ are updated on a faster timescale than $\theta^i$ then we may consider $\theta^i$ fixed for the purpose of our analysis. Thus, the distinction between local and global policies has no impact on this step of our proof.
    
    In the most general setting it will not be possible to guarantee that the linear approximation of each $V^j_\theta$ and $U^j_\theta$ using features $\phi$ will have zero error, and thus we may only prove convergence to within a small neighbourhood of (lexicographic) local optimum. To elucidate, let us consider the hasty value function, the updates to which simply form the classic linear semi-gradient temporal difference algorithm \cite{Sutton2018}. Let $\bm{r}^j$ be the column vector with entries $R^j(s^\otimes_k)$, $Y_U^j$ the diagonal matrix with entries $\gamma_U$ on its diagonal, and $C_\theta$ be the diagonal matrix with entries $d(s^\otimes_k)$ on its diagonal. Let $\Phi$ be the matrix with row $k$ given by $\phi(s^\otimes_k)^\top$. Then the vector $\bm{u}^j$ with entries $\hat{U}^j_\theta(s^\otimes_k)$ can be written as $\bm{u}^j = \Phi u^j$. Finally, let $P_\theta$ be the matrix with entries given by $P_\theta[x,y] = \sum_a\pi(a\vert s_x^\otimes ; \theta)T(s_{y}^\otimes\vert s_x^\otimes,a)$ and $W^j_\theta$ be the function given by $W^j_\theta(\bm{x}) = P_\theta(\bm{r}^j + C_\theta \bm{x})$. Then we define $\widehat{\bm{u}^j} \coloneqq \Phi u^j_*(\theta)$ to be the (unique) \emph{TD fixed point} satisfying:
\begin{align}
    \label{eq:tdfp}
    \Phi^\top C_\theta \widehat{\bm{u}^j} = \Phi^\top C_\theta W^j_\theta (\widehat{\bm{u}^j})
\end{align}
for each $i \in \{1,\ldots,m\}$ and we write $\widehat{U^j_\theta(s)} = \phi(s)^\top u^j_*(\theta)$ as the linear approximation of the value of state $s^\otimes$ with respect to $K^j$. Similarly, we have $\widehat{K^j(\theta)} = \sum_{s^\otimes} \zeta(s^\otimes) \widehat{U^j_\theta (s^\otimes)}$. Akin to other results in the literature, we show below that \textsc{Alamnac} finds a lexicographic optimum with respect to the objectives $\widehat{J} = \sum_j w[j] \widehat{J^j}$ (where $\widehat{J^j}$ is defined analogously to $\widehat{K^j}$) and $\widehat{K} = \sum_j w[j] \widehat{K^j}$, which is in turn within some neighbourhood of a lexicographic local optimum with respect to $J$ and $K$. When $\max_{\theta} \sum_{s^\otimes} d^j(s^\otimes) \big\vert V^{j*}_\theta(s^\otimes) -  \widehat{V^j_\theta(s^\otimes)} \big\vert$ and $\max_{\theta} \sum_{s^\otimes} c(s^\otimes) \big\vert U^{j*}_\theta(s^\otimes) -  \widehat{U^j_\theta(s^\otimes)} \big\vert$ are small this neighbourhood will be small; see, e.g., Lemma 6 and Theorem 2 in \cite{Bhatnagar2009}. Equally, when one can ensure the use of a \emph{compatible} approximation then one can reduce this neighbourhood to just those lexicographic optima with respect to $J$ and $K$ \cite{Sutton1999}. 
So as to consider the most general case, in what remains we therefore consider the objectives $\widehat{J}$ and $\widehat{K}$ instead of $J$ and $K$, though to avoid clutter we drop the $~\widehat{}~$ notation. We emphasise, however, that we do not make any assumptions on the bounds of the approximation error. Thus, the size of the neighbourhood of solutions we guarantee convergence to relies on the accuracy of the approximation used for the critic.

As stated above, the updates to each $u^j$ simply form the classic linear semi-gradient temporal difference algorithm \cite{Sutton2018} which, under conditions 1--4, is known to converge with probability one to some $u^j_*(\theta)$ where $\widehat{\bm{u}^j} \coloneqq \Phi u^j_*(\theta)$ is the (unique) TD fixed point satisfying equation (\ref{eq:tdfp}). See \cite{Tsitsiklis1997} for the original proof (Theorem 1) and \cite{Sutton2018} for a more recent exposition (Section 9.4). All that remains is to provide an analogous result for each $v^j$. 
    Let $Y_Y^j$ be the diagonal matrix with entries $\Gamma^j(s^\otimes_k)$ on its diagonal, and $D_j$ be the diagonal matrix with entries $d^j(s^\otimes_k)$ on its diagonal. Then define $X^j_\theta$ to be the function given by $X^j_\theta(\bm{x}) = P_\theta(\bm{r}^j + Y_V^j \bm{x})$. Then we define $\widehat{\bm{v}^j} \coloneqq \Phi v^j_*(\theta)$ to be a TD fixed point satisfying:
\begin{align}
    \label{eq:patient_bellman}
    \Phi^\top D^j_\theta \widehat{\bm{v}^j} = \Phi^\top D^j_\theta X^j_\theta \widehat{\bm{v}^j} = \Phi^\top D^j_\theta (X^j_\theta)^k (\widehat{\bm{v}^j})
\end{align}
which is simply the Bellman equation given in (\ref{eq:tdfp}) written for $k$ steps. When using a state-dependent discount function whose value is sometimes 1, one may introduce extra fixed points of the standard Bellman equation (see Example \ref{ex:non-conv} for a worked example of how spurious fixed points may be reached if not using a patient update scheme). This is because $X^j_\theta$ no longer forms a contraction due to discount matrix $Y^j_V$ having some entries that are equal to 1. However, by using a patient update (after $k$ steps) then we effectively update each element of $\bm{v}^j$ using the operator $(X^j_\theta)^k$ where $k$ varies, which \emph{is} a contraction with respect to each state the due to discount being less than $1$. 

Let $m^j$ be maximum number of timesteps such that there exists a fully mixed joint policy $\pi$ (i.e., one in which every joint action is played with non-zero probability at every state) parametrised by $\theta$ with:
$$\sum_\rho \Pr^\theta_{G_B}(\rho) \sum^{m^j-1}_{t=0} R^j(\rho[t]) = 0$$
and $V^{j*}_\theta(\rho[m^j]) > 0$ for every path $\rho$ such that $\Pr^\pi_{G_B}(\rho) > 0$, where $V^{j*}_\theta$ is the \textit{true} patient value function for specification $\varphi^j$ that we are attempting to learn. In other words, $m^j$ is the maximum length of time we may possibly continue for in $G_B$ under any fully mixed policy without any chance of seeing a reward from $R^j_\otimes$ and whilst not entering a state from which no such rewards can possibly be observed in future (meaning that $V^{j*}_\pi(\rho[m^j]) = 0$). First, note that due to our policy parametrisation in \textsc{Almanac}, $\pi$ is in fact always fully mixed. Second, observe that each $m^j$ is always finite and in fact bounded by $\vert S^\otimes \vert$. Then the $(X^j_\theta)^{m^j}$ is a contraction with respect to $D^j_\theta$ (as opposed to an individual state), as can be seen by replacing the relevant matrices and vectors in Lemma 4 of \cite{Tsitsiklis1997}. Thus, following identical arguments as this previous work, the fixed point $\widehat{\bm{v}^j}$ for $k = m^j$ in (\ref{eq:patient_bellman}) is unique and is converged upon with probability one using the patient TD update scheme.

    \textbf{Step 2.} Due to the learning rates chosen according to (9) in the main manuscript, we may consider the more slowly updated parameters fixed for the purposes of analysing the convergence of more quickly updated parameters \cite{Borkar2008}. The inner loop of \textsc{Alamanac} terminates only when $x^i$ has converged and given that $x^i$ is updated more slowly than the critics we may assume that each $v^j$ and $u^j$ also converges within this loop to the TD fixed points described in the previous step. Let us now consider the convergence of $x^i$. Here, \textsc{Almanac} solves two optimisation problems simultaneously, on different timescales. On the faster timescale we update $x^i$ using the patient critic $V_\theta$ by minimising the loss given by:
    \begin{align*}
    L^i_V(x^i_V;\theta,\nu) \coloneqq
    \sum_j w[j] L^i_{V^j}(x^i;\theta,\nu^j_{\theta,\zeta})
    = \sum_j w[j] \mathbb{E}_{(s^\otimes,a) \sim \nu^j_{\theta,\zeta}} \Big[\big\vert\psi^i_{\theta^i}(a^i\vert s^\otimes)^\top x^i - A^j_\theta(s^\otimes, a)\big\vert\Big]
    \end{align*}
    On the slower timescale we update $x^i$ according to the Lagrangian given by:
    \begin{align*}
    \max_{\lambda^i \geq 0} \min_{x^i} ~~~~~ L^i_U(x^i;\theta,\mu) &+ \lambda^i \big[L^i_V(x^i;\theta,\nu) - l^i\big]
    \end{align*}
    where we have:
    \begin{align*}
    L^i_U(x^i;\theta,\mu) \coloneqq
    \sum_j w[j] L^i_{U^j}(x^i;\theta,\mu^j_{\theta,\zeta})
    = \sum_j w[j] \mathbb{E}_{(s^\otimes,a) \sim \mu^j_{\theta,\zeta}} \Big[\big\vert\psi^i_{\theta^i}(a^i\vert s^\otimes)^\top x^i - Z^j_\theta(s^\otimes, a)\big\vert\Big]
    \end{align*}
    and recall that for $U^j_\theta$ we use $Z^j_\theta$ instead of $A^j_\theta$ and $\mu^j$ instead of $\nu^j$. This allows us to minimise $x^i$ with respect to $L^i_U(x^i;\theta,\mu)$ while not increasing the current loss according to $L^i_V$. As these latter updates occur more slowly, we may consider $x^i$ fixed with respect to them for the purpose of analysing the former. In particular, note that the updates to $x^i$ given in the paper can in fact be split up into two separate updates:
    \begin{align}
    x^i &\gets \Omega_{x^i} \bigg[ x^i + \beta^V_t \sum_j w[j] \chi^{V^j}_t \psi^i_{\theta^i} (a^i_t \vert s^\otimes_t) \bigg] \label{eq:patgradup}\\
    x^i &\gets \Omega_{x^i} \bigg[ x^i + \beta^U_t \sum_j w[j] \big( \chi^{U^j}_t + \lambda^i \chi^{V^j}_t \big) \psi^i_{\theta^i} (a^i_t \vert s^\otimes_t) \bigg] \label{eq:hasgradup}
    \end{align}
    where recall that $\chi^{V^j}_t \coloneqq \Gamma^j_{1:t} \text{sgn}\big( \hat{\delta}^{V^j}_{t:t+1} - \psi^i_{\theta^i}(a^i\vert s^\otimes)^\top x^i\big)$ and $\chi^{U^j}_t \coloneqq \gamma^t_U \text{sgn}\big( \hat{\delta}^{U^j}_{t:t+1} - \psi^i_{\theta^i}(a^i\vert s^\otimes)^\top x^i\big)$. It is well-known\footnote{See, e.g., Lemma 3 in \cite{Bhatnagar2009} for a simple proof in the limit-average setting.} that the TD error $\hat{\delta}^{V^j}_{t:t+1}$ forms an unbiased estimate of $\hat{A}^j_\theta(s_t^\otimes, a_t)$ and given this, it can thus easily be seen that:
    \begin{align*}
        \mathbb{E}_\theta \Big[\sum_j w[j] \chi^{V^j}_t \psi^i_{\theta^i} (a^i_t \vert s^\otimes_t)\Big]
        &= \sum_j w[j] \mathbb{E}_\theta \Big[ \Gamma^j_{1:t} \text{sgn}\big( \hat{\delta}^{V^j}_{t:t+1} - \psi^i_{\theta^i}(a^i_t\vert s^\otimes_t)^\top x^i\big) \psi^i_{\theta^i} (a^i_t \vert s^\otimes_t) \Big]\\
        &= \sum_j w[j]  \mathbb{E}_\theta \Big[ \Gamma^j_{1:t} \text{sgn}\big( \hat{A}^j_\theta(s_t^\otimes, a_t) - \psi^i_{\theta^i}(a^i_t\vert s^\otimes_t)^\top x^i\big) \psi^i_{\theta^i} (a^i_t \vert s^\otimes_t) \Big]\\
        &= \sum_j w[j]  \mathbb{E}_{s,a \sim \nu_i^\theta} \Big[ \text{sgn}\big( \hat{A}^j_\theta(s^\otimes, a) - \psi^i_{\theta^i}(a^i\vert s^\otimes)^\top x^i\big) \psi^i_{\theta^i} (a^i \vert s^\otimes) \Big]\\
        &= - \sum_j w[j] \nabla_{x^i} L^i_{V^j}(x^i;\theta,\nu^j_{\theta,\zeta})\\
        & = - \nabla_{x^i} L^i_V(x^i;\theta,\nu)
    \end{align*}
    and hence that the update rule (\ref{eq:patgradup}) uses an unbiased estimate for gradient of $L^i_V(x^i;\theta,\nu)$, thus forming a discrete approximation of the ODE given by:
    $$\dot{x^i_t} = \Omega_{x^i} \big[ - \nabla_{x^i} L^i_V(x^i;\theta,\nu) \big]$$
    where $t$ indexes the values of $x$ over time and the projection operator $\Omega_{x^i}$ ensures that iterates governed by this ODE remain in a bounded, compact set $X^i$ (which we assume contains at least one solution). Standard stochastic approximation arguments and the fact that $L^i_V(x^i;\theta,\nu)$ is convex can then be employed to show that the iteration leads to the convergence of $x^i$ to some $x^i_V \in \argmin_{x^i} L^i_V(x^i;\theta,\nu)$ \cite{Borkar2008}. 
    
    What remains is to show convergence of the hasty updates to $x^i$ and to $\lambda^i$. We may assume, due to the separation of timescales, that $x^i$ has converged with respect to $L^i_V$, giving some $x^i_V$ and patient loss $l^i = L^i_V(x^i_V;\theta,\nu)$. Further, as each $\lambda^i$ is updated on a slower timescale then we can assume that $\lambda^i$ is also fixed. Following a similar chain of equalities to the above we have:
    \begin{align*}
        \mathbb{E}_\theta \Big[\sum_j w[j] \big( \chi^{U^j}_t + \lambda^i \chi^{V^j}_t \big) \psi^i_{\theta^i} (a^i_t \vert s^\otimes_t)\Big] = - \nabla_{x^i} \Big( L^i_U(x^i;\theta,\mu) &+ \lambda^i \big[L^i_V(x^i;\theta,\nu) - l^i\big] \Big).
    \end{align*}
    The update rule (\ref{eq:hasgradup}) therefore forms a discrete approximation of the ODE given by
    $$\dot{x^i_t} = \Omega_{x^i} \Big[ - \nabla_{x^i} \big( L^i_U(x^i;\theta,\mu) + \lambda^i \big(L^i_V(x^i;\theta,\nu) - l^i\big) \big) \Big]$$
    and so by similar arguments to above converges to some $x^i_*(\lambda^i) \in \argmin_{x^i} L^i_U(x^i;\theta,\mu) + \lambda^i \big[L^i_V(x^i;\theta,\nu) - l^i\big]$ where $\lambda^i$ is viewed as fixed. Finally, following precisely the same argument we have that the update rule for $\lambda^i$ given by:
    $$\lambda^i \gets \Omega_\lambda \bigg[\lambda^i + \eta_t \Big( \sum_j w[j] \Gamma^j_{1:t} \big\vert\psi^i_{\theta^i}(a^i_t\vert s_t^\otimes)^\top x^i_{V} - \hat{\delta}^{V^j}_{t:t+1}\big\vert - l^i \Big) \bigg]$$
    forms an unbiased estimate of the gradient $\nabla_{\lambda^i} \Big( L^i_U(x^i;\theta,\mu) + \lambda^i \big[L^i_V(x^i;\theta,\nu) - l^i\big] \Big)$ and thus a discrete approximation of the ODE given by:
    $$\dot{\lambda^i_\tau} = \Omega_{\lambda^i} \Big[ \nabla_{\lambda^i} \big( L^i_U(x^i(\lambda^i_\tau);\theta,\mu) + \lambda^i \big[L^i_V(x^i(\lambda^i_\tau);\theta,\nu) - l^i\big) \big) \Big]$$
    where $\tau$ indexes the values of $\lambda^i$ over time, $x^i(\lambda^i_\tau)$ is the limit of the $x^i$ recursion with static parameters $\lambda^i_\tau$, and $\Omega_{\lambda^i}$ ensures that $\lambda \geq 0$. Applying standard stochastic approximation arguments we see that $\lambda^i_\tau(x^i_*) \rightarrow \lambda^i_* \in \argmax_{\lambda^i \geq 0} L^i_U(x^i_;\theta,\mu) + \lambda^i \big[L^i_V(x^i_;\theta,\nu) - l^i\big]$ and so we have that $(x^i_*, \lambda^i_*)$ forms a saddle-point of the Lagrangian. Hence, by the Lexicographic KKT conditions $(x^i_*, \lambda^i_*)$ is a solution of our constrained optimisation problem; see Theorem 2 in \cite{Rentmeestersa1996}. Hence the natural gradient found by the inner loop of \textsc{Almanac} satisfies:
    \begin{equation*}
    x^i_* \in \argmin\hspace{0.01em}_{x^i_U \in \argmin_{x^i_V} L^i_V(x^i_V;\theta,\nu)}~L^i_U(x^i_U;\theta,\mu).
    \end{equation*}
    We conclude this step with a brief remark on the differences between the local and global cases.
    In the local case then $x^i_*$ is an approximation of the natural gradient $\tilde{\nabla}_{\theta^i} K(\theta)$ subject to minimising $\tilde{\nabla}_{\theta^i} J(\theta)$, but in the global case each agent has access to the full joint policy $\pi$ and parameters $\theta$, and so $x^i_*$ approximates $\tilde{\nabla}_{\theta} K(\theta)$ subject to minimising $\tilde{\nabla}_{\theta} J(\theta)$. It can be observed that, due to shared observations of the state space and reward signal, the updates made to $x^i$ by each agent $i$ in this latter case are identical, and so (given an identical initialisation) then the natural gradient $x^i_*$ converged upon is also identical.

    \textbf{Step 3.} Finally we consider the updates to the actor. We focus primarily on the proof of convergence with respect to $J$, as the proof for $K$ is the same as in Agarwal et al. \cite{Agarwal2019}. In particular, we show that error term given in Lemma \ref{noregret} tends to 0 in the global case (giving a global optimum), and tends to 0 with respect to the policies of other agents in the local case (giving a local optimum). Recall that the aforementioned error term is given by:
    \begin{align*}
        e^j_t 
        &= \mathbb{E}_{\rho^\zeta_*} \Big[ \sum^\infty_{\tau=0} \Gamma^j_{0:\tau} \Big( A^j_{\theta_t}(s^\otimes_\tau,a_\tau) - \psi^i_{\theta^i_t}(a^i_\tau\vert s^\otimes_\tau)^\top x^i_t \Big) ~\Big\vert~ F^j(\rho) = \infty \Big]\\
        &=  \mathbb{E}_{\rho^\zeta_*} \Big[ \sum^\infty_{\tau=0} \Gamma^j_{0:\tau} \Big( A^j_{\theta_t}(s^\otimes_\tau,a_\tau) - \psi^i_{\theta^i_t}(a^i_\tau\vert s^\otimes_\tau)^\top x^{i*}_t \Big) ~\Big\vert~ F^j(\rho) = \infty \Big] \\
        &+  \mathbb{E}_{\rho^\zeta_*} \Big[ \sum^\infty_{\tau=0} \Gamma^j_{0:\tau} \Big( \psi^i_{\theta^i_t}(a^i_\tau\vert s^\otimes_\tau)^\top \big( x^{i*}_t - x^i_t \big) \Big) ~\Big\vert~ F^j(\rho) = \infty \Big]
    \end{align*}
    where we write $\mathbb{E}_{\rho^\zeta_*}$ in place of $\mathbb{E}_{\rho \sim \Pr^{\theta_*}_{G_B} (\cdot\vert s^\otimes), s^\otimes \sim \zeta^\otimes}$, and recall that $x^{i*}_t \in \min_{x^i_t} L^i_V(x^i_t;\theta_t,\nu_t)$. Note that specifications $\varphi^j$ such that all paths with $F^j(\rho) = \infty$ have probability 0 can be ignored, as can been be seen in the derivation in Lemma \ref{performancedifference}. As such, we assume for the rest of the proof that $\Pr^\theta_{G_B}\big(F^j(\rho) = \infty\big) > 0$ for the specifications $\varphi^j$ in the remainder of our proof.
    
    Let us begin with the first sub-term identified in the equality above. In what follows we write $l^j_t(x^i; s^\otimes_\tau,a_\tau)$ for $A^j_{\theta_t}(s^\otimes_\tau,a_\tau) - \psi^i_{\theta^i_t}(a^i_\tau\vert s^\otimes_\tau)^\top x^i$ to simplify notation. Before proceeding we note the following fact. First, any path $\rho$ through the Markov chain $\Pr^{\theta_t}_{G_B}$ will eventually enter some bottom strongly connected component (BSCC) almost surely, hence almost all the probability mass of the set of paths with $F^j(\rho) < \infty$ is due to those paths that enter some BSCC that contains no rewarding state for $R^j_\otimes$. Let the union of the states in these BSCCs be denoted by $B^j_\otimes \subseteq S^\otimes$. As $\sum^\infty_{\tau=0} \Gamma^j_{0:\tau} = \infty$ along such a path and any state $s^\otimes \not\in B^j_\otimes$ is visited only finitely many times, then the discounted state distribution $d^j_\zeta(\cdot ; \theta_*)$ over states on paths $\rho$ with $F^j(\rho) < \infty$ is simply the steady state distribution over $B^j_\otimes$. Formally, we have:
    $$\mathbb{E}_{\rho^\zeta_*} \Bigg[ \frac{1}{\sum^\infty_{\tau=0} \Gamma^j_{0:\tau}} \sum^\infty_{\tau=0} \Gamma^j_{0:\tau} l^j_t(x^i; s^\otimes_\tau,a_\tau) ~\Bigg\vert~ F^j(\rho) < \infty \Bigg]
    = \mathbb{E}_{(s^\otimes, a) \sim \nu^j_*} \big[ l^j_t(x^i; s^\otimes,a) ~\big\vert~ s^\otimes \in B^j_\otimes \big]$$
    Note that $A^j_{\theta_t}(s^\otimes,a) = 0$ for any $s^\otimes \in B^j_\otimes$ and any joint action $a$. From this observation and the equality above we deduce the following:
    \begin{align*}
        &\mathbb{E}_{(s^\otimes, a) \sim \nu^j_*} \big[ l^j_t(x^i; s^\otimes_\tau,a_\tau) ~\big\vert~ s^\otimes \in B^j_\otimes \big]\\
        &= \mathbb{E}_{(s^\otimes, a) \sim \nu^j_*} \big[ A^j_{\theta_t}(s^\otimes_\tau,a_\tau) - \psi^i_{\theta^i_t}(a^i_\tau\vert s^\otimes_\tau)^\top x^i ~\big\vert~ s^\otimes \in B^j_\otimes \big]\\
        &= -\mathbb{E}_{(s^\otimes, a) \sim \nu^j_*} \big[ \psi^i_{\theta^i_t}(a^i\vert s^\otimes)^\top x^i ~\big\vert~ s^\otimes \in B^j_\otimes \big]\\
        &= -\mathbb{E}_{(s^\otimes, a) \sim \nu^j_*} \big[ \nabla_{\theta^i_t} \log\pi^i(a^i\vert s^\otimes;\theta^i_t) ~\big\vert~ s^\otimes \in B^j_\otimes \big]^\top x^i\\
        &= -\mathbb{E}_{(s^\otimes, a) \sim \nu^j_*} \Big[ \nabla_{\theta^i_t} \sum_i \log\pi^i(a^i\vert s^\otimes;\theta^i_t) ~\Big\vert~ s^\otimes \in B^j_\otimes \Big]^\top x^i\\
        &= -\mathbb{E}_{(s^\otimes, a) \sim \nu^j_*} \big[ \nabla_{\theta^i_t} \log\pi(a \vert s^\otimes;\theta_t) ~\big\vert~ s^\otimes \in B^j_\otimes \big]^\top x^i\\
        &= -\mathbb{E}_{(s^\otimes, a) \sim \nu^j_*} \Bigg[ \frac{\nabla_{\theta^i_t} \pi(a \vert s^\otimes;\theta_t)}{\pi(a \vert s^\otimes;\theta_t)} ~\Bigg\vert~ s^\otimes \in B^j_\otimes \Bigg]^\top x^i\\
        &= -\mathbb{E}_{s^\otimes \sim d^j_\zeta(\cdot ; \theta_*)} \Bigg[ \sum_{a \in A} \pi(a \vert s^\otimes; \theta_t) \frac{\nabla_{\theta^i_t} \pi(a \vert s^\otimes;\theta_t)}{\pi(a \vert s^\otimes;\theta_t)} ~\Bigg\vert~ s^\otimes \in B^j_\otimes \Bigg]^\top x^i\\
        &= -\mathbb{E}_{s^\otimes \sim d^j_\zeta(\cdot ; \theta_*)} \Big[ \nabla_{\theta^i_t} \sum_{a \in A} \pi(a \vert s^\otimes;\theta_t) ~\Big\vert~ s^\otimes \in B^j_\otimes \Big]^\top x^i\\
        &= -\mathbb{E}_{s^\otimes \sim d^j_\zeta(\cdot ; \theta_*)} \big[ \nabla_{\theta^i_t} 1 ~\big\vert~ s^\otimes \in B^j_\otimes \big]^\top x^i\\
        &= 0
    \end{align*}
    A similar line of reasoning applies when 
    Recall that by condition 5 we have that $\mathbb{E}_t \big[ L^i_{V^j}(x^{i*}_t; \theta_t, \nu^j_* ) \big] \leq e^j_{approx}$, where $e^j_{approx}$ is some constant, thus $\mathbb{E}_t \big[ L^i_{V}(x^{i*}_t; \theta_t, \nu_{*} ) \big] \leq e_{approx} \coloneqq \sum_j w[j] e^j_{approx}$. From this and our results immediately above we have:
    \begin{align}
        \sqrt{e^j_{\text{approx}}}
        &\geq \sqrt{L^i_{V^j}(x^{i*}_t;\theta_t,\nu^j_*)} \label{eq2:1}\\
        &= \sqrt{\mathbb{E}_{(s^\otimes, a) \sim \nu^j_*} \big[ l^j_t(x^{i*}_t; s^\otimes,a)^2 \big]} \nonumber \\
        &\geq \mathbb{E}_{(s^\otimes, a) \sim \nu^j_*} \Big[ \sqrt{l^j_t(x^{i*}_t; s^\otimes,a)^2} \Big] \nonumber \\
        &= \mathbb{E}_{\rho^\zeta_*} \Bigg[ \frac{1}{\sum^\infty_{\tau=0} \Gamma^j_{0:\tau}} \sum^\infty_{\tau=0} \Gamma^j_{0:\tau} \big\vert l^j_t(x^{i*}_t; s^\otimes_\tau,a_\tau) \big\vert \Bigg] \nonumber \\
        &= \mathbb{E}_{\rho^\zeta_*} \Bigg[ \frac{1}{\sum^\infty_{\tau=0} \Gamma^j_{0:\tau}} \sum^\infty_{\tau=0} \Gamma^j_{0:\tau} \big\vert l^j_t(x^{i*}_t; s^\otimes_\tau,a_\tau) \big\vert ~\Bigg\vert~ F^j(\rho) = \infty \Bigg] \Pr(F^j(\rho) = \infty) \nonumber \\
        &+ \mathbb{E}_{\rho^\zeta_*} \Bigg[ \frac{1}{\sum^\infty_{\tau=0} \Gamma^j_{0:\tau}} \sum^\infty_{\tau=0} \Gamma^j_{0:\tau} \big\vert l^j_t(x^{i*}_t; s^\otimes_\tau,a_\tau) \big\vert ~\Bigg\vert~ F^j(\rho) < \infty \Bigg] \Pr(F^j(\rho) < \infty) \nonumber \\
        &= \mathbb{E}_{\rho^\zeta_*} \Bigg[ \frac{1}{\sum^\infty_{\tau=0} \Gamma^j_{0:\tau}} \sum^\infty_{\tau=0} \Gamma^j_{0:\tau} \big\vert l^j_t(x^{i*}_t; s^\otimes_\tau,a_\tau) \big\vert ~\Bigg\vert~ F^j(\rho) = \infty \Bigg] \Pr(F^j(\rho) = \infty) \label{eq2:2}\\
        &\geq \frac{(1 - \gamma_V) \Pr(F^j(\rho) = \infty)}{M^j} \mathbb{E}_{\rho^\zeta_*} \Big[ \sum^\infty_{\tau=0} \Gamma^j_{0:\tau} \big\vert l^j_t(x^{i*}_t; s^\otimes_\tau,a_\tau) \big\vert ~\Big\vert~ F^j(\rho) = \infty \Big] \label{eq2:3}\\
        &\geq \frac{(1 - \gamma_V) \Pr(F^j(\rho) = \infty)}{M^j} \mathbb{E}_{\rho^\zeta_*} \Big[ \sum^\infty_{\tau=0} \Gamma^j_{0:\tau} l^j_t(x^{i*}_t; s^\otimes_\tau,a_\tau) ~\Big\vert~ F^j(\rho) = \infty \Big] \nonumber\\
        &= \frac{(1 - \gamma_V) \Pr(F^j(\rho) = \infty)}{M^j} \mathbb{E}_{\rho^\zeta_*} \Big[ \sum^\infty_{\tau=0} \Gamma^j_{0:\tau} \Big( A^j_{\theta_t}(s^\otimes_\tau,a_\tau) - \psi^i_{\theta^i_t}(a^i_\tau\vert s^\otimes_\tau)^\top x^{i*}_t \Big) ~\Big\vert~ F^j(\rho) = \infty \Big] \nonumber\\
        &= \frac{(1 - \gamma_V) P^j}{M^j} \mathbb{E}_{\rho^\zeta_*} \Big[ \sum^\infty_{\tau=0} \Gamma^j_{0:\tau} \Big( A^j_{\theta_t}(s^\otimes_\tau,a_\tau) - \psi^i_{\theta^i_t}(a^i_\tau\vert s^\otimes_\tau)^\top x^{i*}_t \Big) ~\Big\vert~ F^j(\rho) = \infty \Big] \nonumber
    \end{align}
    where $P^j \coloneqq \Pr(F^j(\rho) = \infty) = \sum_\rho \Pr^{\theta_*}_{G_B} \mathbb{I}(F^j(\rho) = \infty) > 0$ (by our reasoning at the beginning of this step) and $\Pr(F^j(\rho) < \infty)$ is defined similarly. In this deduction: (\ref{eq2:1}) is due to condition 5 (as explained above; (\ref{eq2:2}) follows from the previous deduction; and (\ref{eq2:3}) is a result of the same reasoning employed in Lemma \ref{noregret}. We note here that when using a tabular policy representation or over-parametrised neural network, for example, it can be shown that $e_{\text{approx}} = 0$ and hence each $e^j_{\text{approx}} = 0$ \cite{}. In general, however, we have that
    \begin{align}
        \sum_j w[j] \mathbb{E}_{\rho^\zeta_*} \Big[ \sum^\infty_{\tau=0} \Gamma^j_{0:\tau} \Big( A^j_{\theta_t}(s^\otimes_\tau,a_\tau) - \psi^i_{\theta^i_t}(a^i_\tau\vert s^\otimes_\tau)^\top x^{i*}_t \Big) ~\Big\vert~ F^j(\rho) = \infty \Big]
        \leq \sum_j w[j] \frac{\sqrt{e^j_{approx}}M^j}{(1 - \gamma_V)P^j}
        \label{the_one}
    \end{align}
    where the conditional observation in the latter term stems from our previous remarks at the beginning of this step of the proof. In light of this result, we move onto the second sub-term. Note firstly the following series of inequalities holds:
    \begin{align*}
        &\mathbb{E}_{\rho^\zeta_*} \Big[ \sum^\infty_{\tau=0} \Gamma^j_{0:\tau} \Big( \psi^i_{\theta^i_t}(a^i_\tau\vert s^\otimes_\tau)^\top \big( x^{i*}_t - x^i_t \big) \Big) ~\Big\vert~ F^j(\rho) = \infty \Big]\\
        &= \mathbb{E}_{\rho^\zeta_*} \Big[ \sum^\infty_{\tau=0} \Gamma^j_{0:\tau} \Big( l^j_t(x^i; s^\otimes_\tau,a_\tau) - l^j_t(x^i; s^\otimes_\tau,a_\tau) \Big) ~\Big\vert~ F^j(\rho) = \infty \Big]\\
        &\leq \mathbb{E}_{\rho^\zeta_*} \Big[ \sum^\infty_{\tau=0} \Gamma^j_{0:\tau} \big\vert \psi^i_{\theta^i_t}(a^i_\tau\vert s^\otimes_\tau)^\top \big( x^{i*}_t - x^i_t \big) \big\vert ~\Big\vert~ F^j(\rho) = \infty \Big]\\
        &\leq \frac{M^j}{1 - \gamma_V} \mathbb{E}_{\rho^\zeta_*} \Bigg[ \frac{1}{\sum^\infty_{\tau=0} \Gamma^j_{0:\tau}} \sum^\infty_{\tau=0} \Gamma^j_{0:\tau} \big\vert \psi^i_{\theta^i_t}(a^i_\tau\vert s^\otimes_\tau)^\top \big( x^{i*}_t - x^i_t \big) \big\vert ~\Bigg\vert~ F^j(\rho) = \infty \Bigg]\\
        &= \frac{M^j}{1 - \gamma_V} \mathbb{E}_{(s^\otimes,a) \sim \nu^j_*} \Big[ \big\vert \psi^i_{\theta^i_t}(a^i \vert s^\otimes)^\top \big( x^{i*}_t - x^i_t \big) \big\vert ~\Big\vert~ s^\otimes \not\in B^j_\otimes \Big]\\
        &\leq \frac{M^j}{\Pr(s^\otimes \not\in B^j_\otimes)(1 - \gamma_V)} \mathbb{E}_{(s^\otimes,a) \sim \nu^j_*} \Big[ \big\vert \psi^i_{\theta^i_t}(a^i \vert s^\otimes)^\top \big( x^{i*}_t - x^i_t \big) \big\vert \Big]\\
        &\leq \frac{M^j}{\Pr(s^\otimes \not\in B^j_\otimes)(1 - \gamma_V)} \sqrt{\mathbb{E}_{(s^\otimes,a) \sim \nu^j_*} \bigg[ \Big( \psi^i_{\theta^i_t}(a^i \vert s^\otimes)^\top \big( x^{i*}_t - x^i_t \big) \Big)^2 \bigg]}\\
        &= \frac{M^j}{\Pr(s^\otimes \not\in B^j_\otimes)(1 - \gamma_V)} \sqrt{\Vert x^{i*}_t - x^i_t \Vert^2_{\Sigma^t_{j*}}}
    \end{align*}
    where $\Pr(s^\otimes \not\in B^j_\otimes) = \sum_{s^\otimes \not\in B^j_\otimes} d^j_*(s^\otimes) > 0$, as we have $\Pr(F^j(\rho) = \infty) > 0$, and $\Sigma^t_{j*} = \Sigma^{\theta^i_t}_{\nu^j_*}$. Recall our definition of $\nu^j_{\theta,\xi}$ as given by:
    $$\nu^j_{\theta,\xi}(s^\otimes, a) = \sum_{(s_0^\otimes, a_0) \in S^\otimes \times A} \xi^\otimes(s_0^\otimes, a_0) \sum_{\rho} \Pr^{\theta}_{G_B} (\rho \vert s_0^\otimes, a_0)
    \Bigg[ \frac{1}{\sum_{t=0}^\infty \Gamma^j_{0:t} } \sum_{t=0}^\infty \Gamma^j_{0:t} \mathbb{I} \big( \rho[t, t + 0.5] = (s^\otimes, a) \big) \Bigg]$$
    We note, as in Agrawal et al., that for all state-action pairs $s^\otimes, a$ and any set of joint parameters $\theta$ we have
    $\nu^j_{\theta,\xi}(s^\otimes, a) \geq \frac{1 - \gamma_V}{M^j} \xi(s^\otimes, a)$. We conclude by using this fact and by following the proof of Theorem 6.2 in Agarwal et al., adjusted for our weighted sum of losses due to our multiple objectives:
    \begin{align*}
        \sum_j w[j] \Vert x^{i*}_t - x^i_t \Vert^2_{\Sigma^t_{j*}}
        &\leq \sum_j w[j] \big\Vert (\Sigma^t_\xi)^{-1/2} \Sigma^t_{j*} (\Sigma^t_\xi)^{-1/2} \big\Vert_2 
        \Vert x^{i*}_t - x^i_t \Vert^2_{\Sigma^t_{\xi}}\\
        &= \sum_j w[j] \kappa^j_t \Vert x^{i*}_t - x^i_t \Vert^2_{\Sigma^t_{\xi}}\\
        &\leq \sum_j w[j] \frac{M^j\kappa^j_t}{1 - \gamma_V} \Vert x^{i*}_t - x^i_t \Vert^2_{\Sigma^t_{j,\xi}}\\
        &\leq \frac{M\kappa_t}{1 - \gamma_V} \sum_j w[j] \Vert x^{i*}_t - x^i_t \Vert^2_{\Sigma^t_{j,\xi}}\\
        &\leq \frac{M\kappa_t}{1 - \gamma_V} \big( L^i_{V}(x^i_t;\theta_t,\nu_{t,\xi}) - L^i_{V}(x^i_{t*};\theta_t,\nu_{t,\xi}) \big)
    \end{align*}
    where $\Sigma^t_{j,\xi} = \Sigma^{\theta_t}_{\nu^j_{\theta_t,\xi}}$, $\kappa^j_t = \big\Vert (\Sigma^t_\xi)^{-1/2} \Sigma^t_{j*} (\Sigma^t_\xi)^{-1/2} \big\Vert_2$, and we set $M = \max_j M^j$ and $\kappa_t = \max_j \kappa^j_t$. Here, the final line follows from the first order optimality conditions for $x^i_{t*}$ with respect to $L^i_{V}(x^i_{t*};\theta_t,\nu_{t,\xi})$ which imply that $(x^i_t -  x^i_{t*})^\top \nabla_{x^i} L^i_{V}(x^i_{t*};\theta_t,\nu_{t,\xi}) \geq 0$ (we refer the reader to Agarwal et al. for the short proof \cite{Agarwal2019}). Thus, in expectation we have:
    \begin{align}
        \mathbb{E}_{t \sim \iota_T} \Big[ \sum_j w[j] \Vert x^{i*}_t - x^i_t \Vert^2_{\Sigma^t_{j*}} \Big]
        &\leq \mathbb{E}_{t \sim \iota_T} \bigg[ \frac{M\kappa_t}{1 - \gamma_V} \big( L^i_{V}(x^i_t;\theta_t,\nu_{t,\xi}) - L^i_{V}(x^i_{t*};\theta_t,\nu_{t,\xi}) \big) \bigg]\\
        &\leq \frac{M\kappa}{1 - \gamma_V} \mathbb{E}_{t \sim \iota_T} \Big[ L^i_{V}(x^i_t;\theta_t,\nu_{t,\xi}) - L^i_{V}(x^i_{t*};\theta_t,\nu_{t,\xi}) \Big]\\
        &= \frac{M\kappa}{1 - \gamma_V} \mathbb{E}_{t \sim \iota_T} \Big[ L^i_{V}(x^i_t;\theta_t,\nu_{t,\xi}) - L^i_{V}(x^i_{t*};\theta_t,\nu_{t,\xi}) ~\Big\vert~ \theta_t \Big]\\
        &\leq \frac{M\kappa}{1 - \gamma_V} e_{\text{stat}}
    \end{align}
    Now note that in \textsc{Almanac}, due to the use of a two timescale approach, then whenever we come to update $\theta^i_t$ we already have that the process of finding $x^i_t$ has converged and hence (via the second step of the proof) we have $e_{\text{stat}} = 0$. As, therefore, $\mathbb{E}_{t \sim \iota_T} \big[ \sum_j w[j] \Vert x^{i*}_t - x^i_t \Vert^2_{\Sigma^t_{j*}} \big] = 0$ then we also have that:
    \begin{align*}
        &\mathbb{E}_{t \sim \iota_T} \bigg[ \sum_j w[j] \mathbb{E}_{\rho^\zeta_*} \Big[ \sum^\infty_{\tau=0} \Gamma^j_{0:\tau} \Big( \psi^i_{\theta^i_t}(a^i_\tau\vert s^\otimes_\tau)^\top \big( x^{i*}_t - x^i_t \big) \Big) ~\Big\vert~ F^j(\rho) = \infty \Big] \bigg]\\
        &\leq \mathbb{E}_{t \sim \iota_T} \bigg[ \sum_j w[j] \frac{M^j}{\Pr(s^\otimes \not\in B^j_\otimes)(1 - \gamma_V)} \sqrt{\Vert x^{i*}_t - x^i_t \Vert^2_{\Sigma^t_{j*}}} \bigg]\\
        &= 0 
    \end{align*}
    Thus, we conclude Step 3. of our proof with the observation that using the inequality (\ref{the_one}) to show that:
    \begin{equation}
        V_{\theta_*}(s^\otimes) - \lim_{T \rightarrow \infty} \mathbb{E}_{t \sim \iota_T} \big[ V_{\theta_t}(s^\otimes) \big] \leq \lim_{T \rightarrow \infty} \mathbb{E}_{t \sim \iota_T} \left[\sum_j w[j] \frac{\sqrt{e^j_{approx}}M^j}{(1 - \gamma_V)P^j} \right].
    \end{equation}
    In particular, in the global setting then updates to $\theta^i$ are the same as updates to the global parameters $\theta$, and so a (neighbourhood of a) global optimum is reached. In the local setting, then each actor updates their policy given the policies of the other agents, and hence $V_{\theta_*}(s^\otimes)$ is defined with respect to these other policies. The final joint strategy may thus not be globally optimal, but no agent may deviate to improve their policy, and so it will at least be a (neighbourhood of a) MPE and thus a (neighbourhood of a) local optimum.
\end{proof}
\begin{lemma}
    \label{LTL2Buchi}
    Let $\varphi$ be an LTL formula. Given an LDBA $B$ (with accepting set $F$ and alphabet $\Sigma$) representing $\varphi$, a (finite) MG $G$ (with state space $S$), and a labelling function $L : S \rightarrow \Sigma$, let the resulting product MG $G_B$ and let $F_\otimes = S \times F$. We denote by $F(\rho)$ the number of times that $\rho$ passes through $F_\otimes = S \times F$. Recalling our notion of satisfaction probability from Definition 6 and letting $\pi$ be a joint policy in $G$, then for any extension $\pi'$ of $\pi$ to the product game $G_B$ we have that
    $\Pr^\pi_G(s \models \varphi) \geq \Pr^{\pi'}_{G_B}(\{\rho : F(\rho)= \infty \} \vert s^\otimes)$
    where $s^\otimes = (s,q_0)$. Further, there exists a canonical extension of $\pi$ (which for simplicity we also denote by $\pi$) such that $\Pr^\pi_G(s \models \varphi) = \Pr^{\pi}_{G_B}(\{\rho : F(\rho) = \infty\} \vert s^\otimes)$.
\end{lemma}
\begin{proof}
    Note that by an `extension' of $\pi$ to the product game we refer to a policy that takes the same actions as $\pi$ at all states $s^\otimes$, apart from when $\E$-transitioning from product states over $Q_I$ to those in $Q_A$ (as such actions do not exist in $G$). By viewing the product action space $A^1 \times \cdots \times A^n$ in $G$ as a single joint action space $A$ then a joint policy $\pi$ over $G_B$ can be viewed as a policy of a single agent in a product MDP instead. The result is then an immediate consequence of Theorem 3 from Sickert et al. \cite{Sickert2016}.
\end{proof}

\section{Examples}
\label{examples}

In this section we provide a small set of examples, complete with diagrams, in order to ease understanding of the underlying constructions and to highlight some of the subtleties of our work.

\subsection{Conjunctions and Weighted Combinations of Specifications}
\label{ex:conjunction}

We note here that it is \textit{not} the case that, for some set of LTL goals $\{\varphi^j\}_{0\leq j \leq m}$ and vector of weights $w$ such that each $w[j] > 0$ and $\sum_j w[j] = 1$, that:
$$\argmax_\pi \sum_j w[j] \Pr^{\pi}_G(s \models \varphi^j)~~~=~~~\argmax_\pi \Pr^{\pi}_G(s \models \bigwedge^j \varphi^j)$$
To see this, consider the objectives $\varphi_1 = \ltlF\chi$ and $\varphi_2 = \ltlF\psi$ in the context of the MG given in Figure \ref{conj}. Clearly, for any choice of $w_1$ and $w_2$, the optimal policy for maximising $\sum_j w[j] \Pr^{\pi}_G(s_0 \models \varphi^j)$ takes action $b$ and the optimal policy for maximising $\Pr^{\pi}_G(s_0 \models \bigwedge^j \varphi^j)$ takes action $a$. We do not claim that one of these policies is more natural than another, merely that by also considering weighted combinations of LTL formulae, then we may consider a strictly larger class of goals than if considering simply logical combinations of goals (which are themselves LTL formulae). Whether it is more natural to form a logical or weighted combination will no doubt depend on the goals, agents, and environment in question.

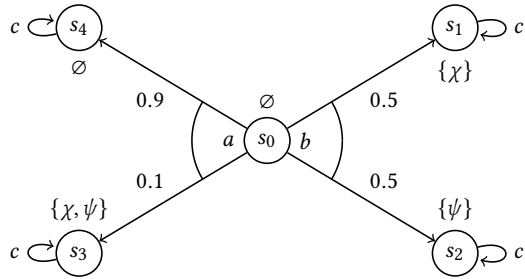
\begin{figure}[h]
    \centering
    \begin {tikzpicture}[
        -latex ,
        auto ,
        node distance =1.5cm and 2.5cm ,
        on grid ,
        semithick ,
        -> ,
        initial text = {},
        state/.style ={circle , draw, 
        }
        ]
            \node[state, label=above:{$\varnothing$}, label={[label distance=0.005cm]360:{$b$}}, label={[label distance=0.005cm]180:{$a$}}] (A) {$s_0$};
            \draw [domain=0:31,-] plot ({cos(\x)}, {sin(\x)});
            \draw [domain=329:360,-] plot ({cos(\x)}, {sin(\x)});
            \draw [domain=149:211,-] plot ({cos(\x)}, {sin(\x)});
            \node[state, label=below:{$\{\chi\}$}] (B) [above right=of A] {$s_1$};
            \node[state, label=above:{$\{\psi\}$}] (C) [below right=of A] {$s_2$};
            \node[state, label=above:{$\{\chi,\psi\}$}] (D) [below left=of A] {$s_3$};
            \node[state, label=below:{$\varnothing$}] (E) [above left=of A] {$s_4$};
            \path (A) edge  node[below right] {$0.5$} (B);
            \path (A) edge  node[above right] {$0.5$} (C);
            \path (A) edge  node[above left] {$0.1$} (D);
            \path (A) edge  node[below left] {$0.9$} (E);
            \path (B) edge [loop right] node[right] {$c$} (B);
            \path (C) edge [loop right] node[right] {$c$} (C);
            \path (D) edge [loop left] node[left] {$c$} (D);
            \path (E) edge [loop left] node[left] {$c$} (E);
        \end{tikzpicture}
    \caption{A small MG $G$ with joint actions $a$, $b$, and $c$.}
    \label{conj}
\end{figure}

\subsection{Product Markov Game}
\label{ex:prodgame}

Here we provide an example of forming a product MG from a small two-player MG $G$ and two LDBAs $B^1$ and $B_2$, as can be seen in Figure \ref{preprod}. We assume that the goal $\ltlF \ltlG \varphi$ is associated with agent 1 in $G$, i.e., $N^{B^1} = \{1\}$ and so $A^1_\otimes = A^1 \cup \{\E_{q^1_1}\}$. As there is only one $\E$ transition in $B^1$ then we simplify notation and write $\E_{q^1_1}$ as $\E$. We further simplify the notation in the diagrams by denoting joint actions with single lowercase letters $a$, $b$, $c$, etc. In the product MG pictured in Figure \ref{postprod} the transition $\E$ thus represents the joint action $(\E,\_)$ where player 1 performs action $\E$ and player 2 performs any action in $A^2_\otimes = A^2$.

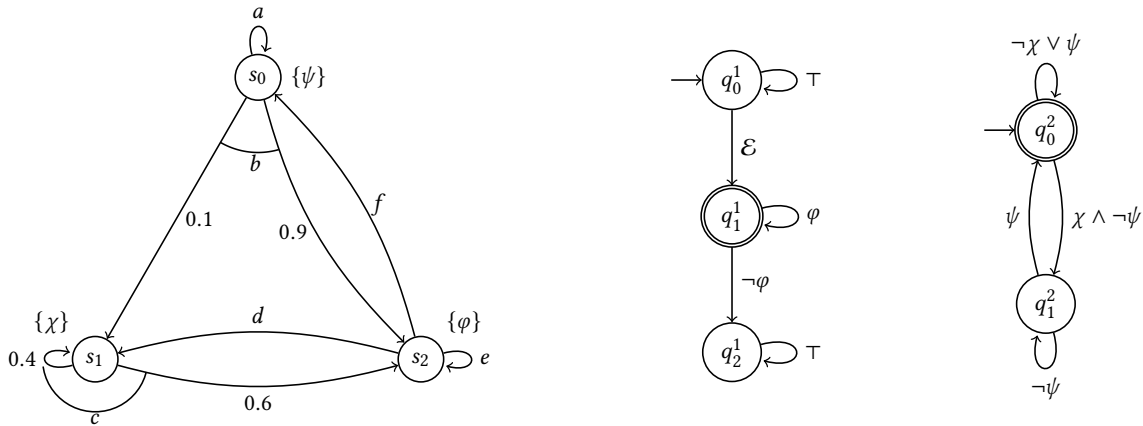
\begin{figure}[h]
    \centering
    \begin{minipage}[c]{0.45\linewidth}
        \flushleft
        \begin {tikzpicture}[
        -latex ,
        auto ,
        on grid ,
        semithick ,
        -> ,
        initial text = {},
        state/.style ={circle , draw, 
        }
        ]
            \draw [domain=240:286,-] plot ({1*cos(\x)}, {2.5 + 1*sin(\x)});
            \draw [domain=188:346,-] plot ({-2.165+ 0.7*cos(\x)}, {-1.25 + 0.7*sin(\x)});
            \node[state, label=right:{$\{\psi\}$}, label={[label distance=0.6cm]268:{$b$}}] (A) at (90:2.5) {$s_0$};
            \node[state, label=above left:{$\{\chi\}$}, label={[label distance=0.3cm]270:{$c$}}] (B) at (210:2.5) {$s_1$};
            \node[state, label=above right:{$\{\varphi\}$}] (C) at (330:2.5) {$s_2$};
            \path (A) edge [loop above] node[above] {$a$} (A);
            \path (A) edge node[right] {$0.1$} (B);
            \path (A) edge [bend right =15] node[left] {$0.9$} (C);
            \path (C) edge [bend right =15] node[right] {$f$} (A);
            \path (B) edge [loop left] node[left] {$0.4$} (B);
            \path (C) edge [loop right] node[right] {$e$} (C);
            \path (B) edge [bend right =15] node[below] {$0.6$} (C);
            \path (C) edge [bend right =15] node[above] {$d$} (B);
        \end{tikzpicture}
    \end{minipage}
    \begin{minipage}[c]{0.2\linewidth}
        \centering
        \begin {tikzpicture}[
        -latex ,
        auto ,
        node distance =1cm and 1cm ,
        semithick ,
        -> ,
        initial text = {},
        state/.style ={circle , draw, 
        }
        ]
            \node[state, initial] (A) {$q^1_0$};
            \node[state, accepting] (B) [below=of A] {$q^1_1$};
            \node[state] (C) [below=of B] {$q^1_2$};
            \path (A) edge [loop right] node[right] {$\top$} (A);
            \path (B) edge [loop right] node[right] {$\varphi$} (B);
            \path (C) edge [loop right] node[right] {$\top$} (C);
            \path (A) edge node[right] {$\E$} (B);
            \path (B) edge node[right] {$\neg \varphi$} (C);
        \end{tikzpicture}
    \end{minipage}
    \begin{minipage}[c]{0.2\linewidth}
        \flushright
        \begin {tikzpicture}[
        -latex ,
        auto ,
        node distance =1.5cm and 1.5cm ,
        semithick ,
        -> ,
        initial text = {},
        state/.style ={circle , draw, 
        }
        ]
            \node[state, initial, accepting] (A) {$q^2_0$};
            \node[state] (B) [below=of A] {$q^2_1$};
            \path (A) edge [loop above] node[above] {$\neg \chi \vee \psi$} (A);
            \path (B) edge [loop below] node[below] {$\neg \psi$} (B);
            \path (A) edge [bend right =-15] node[right] {$\chi \wedge \neg \psi$} (B);
            \path (B) edge [bend right =-15] node[left] {$\psi$} (A);
        \end{tikzpicture}
    \end{minipage}
    \caption{From left to right: a small MG $G$ with two players and joint actions $a$, $b$, $c$, $d$, $e$, and $f$; an LDBA $B^1$ representing the LTL formula $\ltlF \ltlG \varphi$; an LDBA $B_2$ representing the LTL formula $\ltlG (\chi \rightarrow \ltlF \psi)$.}
    \label{preprod}
\end{figure}

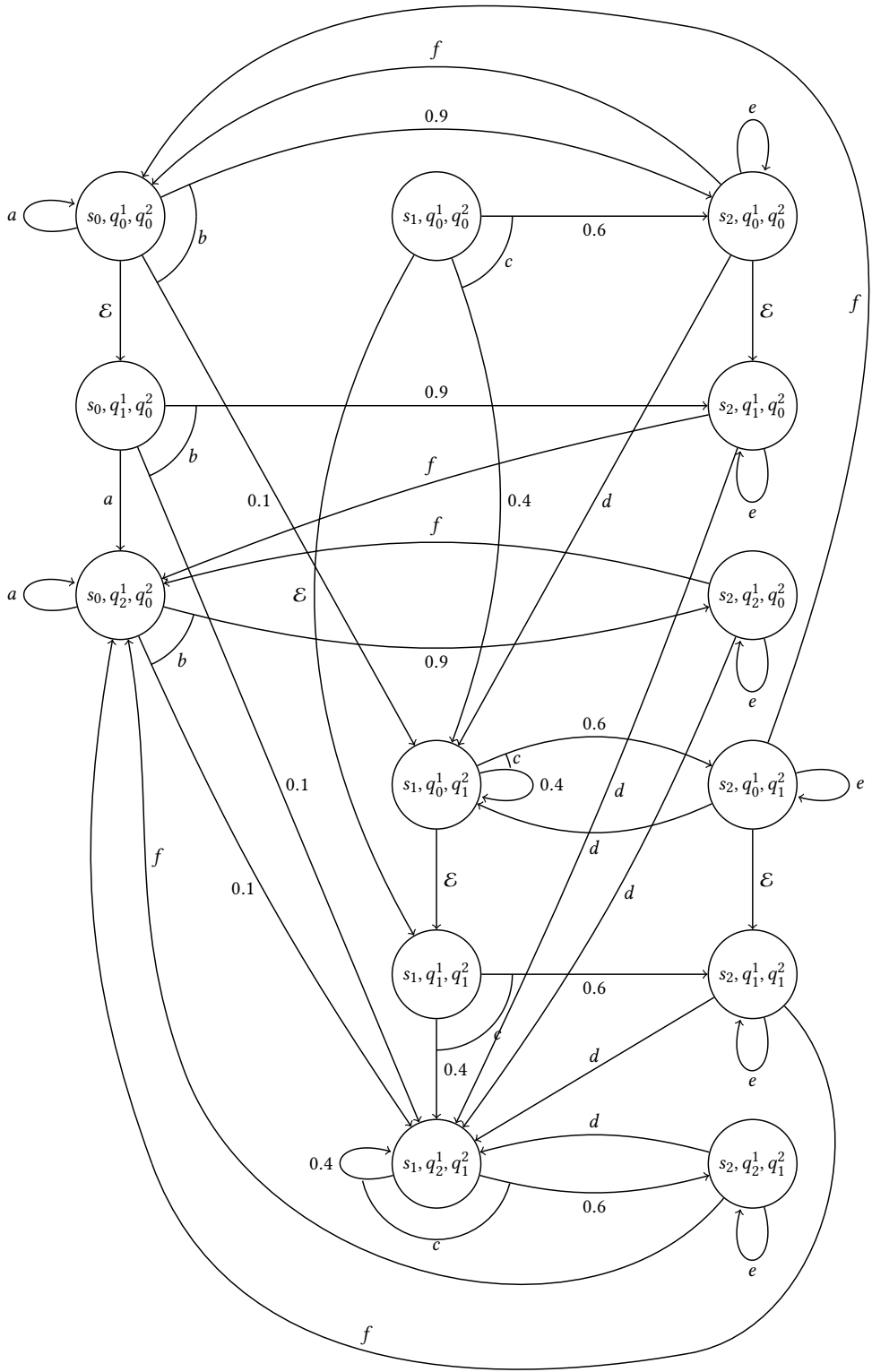
\begin{figure}
    \centering
    \resizebox{!}{0.95\textheight}{
        \begin {tikzpicture}[
        -latex ,
        auto ,
        on grid ,
        semithick ,
        -> ,
        initial text = {},
        state/.style ={circle , draw, 
        }
        ]
            \pgfmathsetmacro{\X}{5}
            \pgfmathsetmacro{\Y}{3}

            \node[state, label={[label distance=0.4cm]355:{$b$}}] (000) at (0*\X,0*\Y) {$s_0,q^1_0,q^2_0$};
            \draw [domain=0:25,-] plot ({1.2*cos(\x)}, {1.2*sin(\x)});
            \draw [domain=299:360,-] plot ({1.2*cos(\x)}, {1.2*sin(\x)});
            \node[state, label={[label distance=0.4cm]330:{$c$}}] (100) at (1*\X,0*\Y) {$s_1,q^1_0,q^2_0$};
            \draw [domain=290:360,-] plot ({\X + 1.2*cos(\x)}, {1.2*sin(\x)});
            \node[state] (200) at (2*\X,0*\Y) {$s_2,q^1_0,q^2_0$};
            \node[state, label={[label distance=0.4cm]330:{$b$}}] (010) at (0*\X,-1*\Y) {$s_0,q^1_1,q^2_0$};
            \draw [domain=292:360,-] plot ({1.2*cos(\x)}, {-\Y + 1.2*sin(\x)});
            \node[state] (210) at (2*\X,-1*\Y) {$s_2,q^1_1,q^2_0$};
            \node[state, label={[label distance=0.4cm]315:{$b$}}] (020) at (0*\X,-2*\Y) {$s_0,q^1_2,q^2_0$};
            \draw [domain=294:345,-] plot ({1.2*cos(\x)}, {-2*\Y + 1.2*sin(\x)});
            \node[state] (220) at (2*\X,-2*\Y) {$s_2,q^1_2,q^2_0$};
            \node[state, label={[label distance=0.4cm]11:{$c$}}] (101) at (1*\X,-3*\Y) {$s_1,q^1_0,q^2_1$};
            \draw [domain=13:24,-] plot ({\X + 1.2*cos(\x)}, {-3*\Y + 1.2*sin(\x)});
            \node[state] (201) at (2*\X,-3*\Y) {$s_2,q^1_0,q^2_1$};
            \node[state, label={[label distance=0.4cm]315:{$c$}}] (111) at (1*\X,-4*\Y) {$s_1,q^1_1,q^2_1$};
            \draw [domain=270:360,-] plot ({\X + 1.2*cos(\x)}, {-4*\Y + 1.2*sin(\x)});
            \node[state] (211) at (2*\X,-4*\Y) {$s_2,q^1_1,q^2_1$};
            \node[state, label={[label distance=0.4cm]270:{$c$}}] (121) at (1*\X,-5*\Y) {$s_1,q^1_2,q^2_1$};
            \draw [domain=193:345,-] plot ({\X + 1.2*cos(\x)}, {-5*\Y + 1.2*sin(\x)});
            \node[state] (221) at (2*\X,-5*\Y) {$s_2,q^1_2,q^2_1$};

            \path (000) edge [loop left] node[left] {$a$} (000);
            \path (000) edge [bend right =-25] node[above] {$0.9$} (200);
            \path (000) edge node[left] {$0.1$} (101);
            \path (000) edge node[left] {$\E$} (010);

            \path (100) edge [bend right =30] node[left] {$\E$} (111);
            \path (100) edge [bend right =-20] node[right] {$0.4$} (101);
            \path (100) edge node[below] {$0.6$} (200);

            \path (200) edge [bend right =45] node[above] {$f$} (000);
            \path (200) edge node[right] {$\E$} (210);
            \path (200) edge node[right] {$d$} (101);
            \path (200) edge [loop above] node[above] {$e$} (200);

            \path (010) edge node[left] {$a$} (020);
            \path (010) edge node[above] {$0.9$} (210);
            \path (010) edge node[right] {$0.1$} (121);

            \path (210) edge [loop below] node[below] {$e$} (210);
            \path (210) edge [bend right =5] node[above] {$f$} (020);
            \path (210) edge [bend right =-3] node[right] {$d$} (121);

            \path (020) edge [loop left] node[left] {$a$} (020);
            \path (020) edge [bend right =15] node[below] {$0.9$} (220);
            \path (020) edge [bend right =5] node[left] {$0.1$} (121);

            \path (220) edge [loop below] node[below] {$e$} (220);
            \path (220) edge [bend right =15] node[above] {$f$} (020);
            \path (220) edge [bend right =-7] node[right] {$d$} (121);

            \path (101) edge node[right] {$\E$} (111);
            \path (101) edge [loop right] node[right] {$0.4$} (101);
            \path (101) edge [bend right =-25] node[above] {$0.6$} (201);

            \path (201) edge [loop right] node[right] {$e$} (201);
            \path (201) edge node[right] {$\E$} (211);
            \path (201) edge [bend right =-25] node[below] {$d$} (101);
            % PATH to 000
            \node (helper) [minimum size=0mm,node distance=2mm, draw=none] at (1.8*\X,1*\Y) {};
            \path (201) edge [-, ,out=70,in=350] node[left] {$f$} (helper.center)
            (helper.center) edge[->, out=170,in=60] (000);

            \path (111) edge node[right] {$0.4$} (121);
            \path (111) edge node[below] {$0.6$} (211);

            \path (211) edge [loop below] node[below] {$e$} (211);
            \path (211) edge  node[above] {$d$} (121);
            % PATH to 020
            \node (helper3) [minimum size=0mm,node distance=2mm, draw=none] at (1.8*\X,-6*\Y) {};
            \node (helper4) [minimum size=0mm,node distance=2mm, draw=none] at (0.1*\X,-5*\Y) {};
            \path (211) edge [-, ,out=315,in=10] (helper3.center)
            (helper3.center) edge [-, out=190,in=290] node[above] {$f$} (helper4.center)
            (helper4.center) edge [->, out=110,in=260] (020);

            \path (121) edge [loop left] node[left] {$0.4$} (121);
            \path (121) edge [bend right =15] node[below] {$0.6$} (221);

            \path (221) edge [loop below] node[below] {$e$} (211);
            \path (221) edge [bend right =15] node[above] {$d$} (121);
            % PATH to 020
            \node (helper2) [minimum size=0mm,node distance=2mm, draw=none] at (0.2*\X,-4.5*\Y) {};
            \path (221) edge [-, ,out=230,in=290] (helper2.center)
            (helper2.center) edge[->, out=110,in=280] node[right] {$f$} (020);

        \end{tikzpicture}
    }
    \caption{The product game $G_{B^1,B_2}$ (unreachable states not pictured).}
    \label{postprod}
\end{figure}

\subsection{Non-Convergence of Q-Learning with State-Dependent Discounting}
\label{ex:non-conv}

Here we show that naively using a state-dependent discount rate with vanilla Q-learning, as has been suggested in earlier works \cite{Hahn2020,Hasanbeig2019}, can lead to convergence to the wrong Q values and hence the wrong policy. Consider the product game $G_B$ given in Figure \ref{nonconv}. We drop the product notation $\otimes$ here for clarity. The standard Q-learning update rule is given by:
$$Q(s_t,a_t) \leftarrow (1-\alpha_t)Q(s_t,a_t) + \alpha_t[R(s_{t+1}) + \gamma(s_{t+1}) \max_{a'} Q(s_{t+1}, a')]$$
where, in $G_B$, $R(s_1,q_1) = 1$ and is zero for all other states, and $\gamma(s_1,q_1) = \gamma_V$ and is one for all other states, where $\gamma_V \in (0, 1)$. The problem arises due to the zero-reward self loops from state $(s_0,q_0)$ to itself. It can be easily observed that, according to this discount scheme and reward function, the correct Q values in $G_B$ are given by:
\begin{equation*}
    Q(s,a) = 
    \begin{cases}
        c & \text{ if } s = (s_1,q_0) \text{ or } s = (s_1,q_1)\\
        0.1c & \text{ if } s = (s_0,q_0)\\
        0 & \text{ otherwise}
    \end{cases}
\end{equation*}
where $c = \frac{1}{1 - \gamma_V}$ is a constant. This means that at the only state where the agent has a choice about what to do, $(s_0,q_0)$, it is ambivalent between choosing $a$ or $b$. While this might seem counter-intuitive, there is in some sense `no rush' for the agent to choose $b$ if it is simply maximising the probability of achieving $\ltlF \psi$. When two Q values are equal, a policy can be chosen so as to randomise over these actions, or the tie can be broken by another scheme (such as via lexicographic RL).

\begin{figure}[h]
    \centering
    \begin{minipage}[c]{0.3\linewidth}
        \flushleft
        \begin {tikzpicture}[
        -latex ,
        auto ,
        node distance =2.5cm and 2.5cm ,
        on grid ,
        semithick ,
        -> ,
        initial text = {},
        state/.style ={circle , draw, 
        }
        ]
            \node[state, label=below left:{$\varnothing$}, label={[label distance=0.2cm]315:{$b$}}] (A) {$s_0$};
            \draw [domain=270:360,-] plot ({0.7*cos(\x)}, {0.7*sin(\x)});
            \node[state, label=below:{$\{\psi\}$}] (B) [right=of A] {$s_1$};
            \node[state, label=left:{$\varnothing$}] (C) [below=of A] {$s_2$};
            \path (A) edge [loop left] node[left] {$a$} (A);
            \path (A) edge  node[below] {$0.1$} (B);
            \path (A) edge  node[right] {$0.9$} (C);
            \path (B) edge [loop right] node[right] {$c$} (B);
            \path (C) edge [loop right] node[right] {$d$} (C);
        \end{tikzpicture}
    \end{minipage}
    \begin{minipage}[c]{0.2\linewidth}
        \centering
        \begin {tikzpicture}[
        -latex ,
        auto ,
        node distance =1.5cm and 1.5cm ,
        semithick ,
        -> ,
        initial text = {},
        state/.style ={circle , draw, 
        }
        ]
            \node[state, initial] (A) {$q_0$};
            \node[state, accepting] (B) [below=of A] {$q_1$};
            \path (A) edge [loop right] node[right] {$\psi$} (A);
            \path (B) edge [loop right] node[right] {$\top$} (B);
            \path (A) edge node[right] {$\neg \psi$} (B);
        \end{tikzpicture}
    \end{minipage}
    \begin{minipage}[c]{0.45\linewidth}
        \flushright
        \begin {tikzpicture}[
        -latex ,
        auto ,
        on grid ,
        semithick ,
        -> ,
        initial text = {},
        state/.style ={circle , draw, 
        }
        ]
            \node[state, label={[label distance=0.2cm]315:{$b$}}] (A) at (0,0) {$s_0,q_0$};
            \draw [domain=270:360,-] plot ({0.9*cos(\x)}, {0.9*sin(\x)});
            \node[state] (B) at (2.5,0) {$s_1,q_1$};
            \node[state] (C) at (0,-2.5) {$s_2,q_0$};
            \node[state] (D) at (5,-2.5) {$s_1,q_0$};
            \node[state] (E) at (2.5,-2.5) {$s_2,q_0$};
            \path (E) edge  node[above] {$d$} (C);
            \path (D) edge  node[above] {$c$} (B);
            \path (A) edge [loop left] node[left] {$a$} (A);
            \path (A) edge  node[below] {$0.1$} (B);
            \path (A) edge  node[right] {$0.9$} (C);
            \path (B) edge [loop right] node[right] {$c$} (B);
            \path (C) edge [loop left] node[left] {$d$} (C);
        \end{tikzpicture}
    \end{minipage}
    \caption{From left to right: a small MG $G$ with joint actions $a$, $b$, $c$, and $d$; a simple LDBA $B$ representing the LTL formula $\ltlF \psi$; the resulting product MG $G_B$ (unreachable states not pictured).}
    \label{nonconv}
\end{figure}
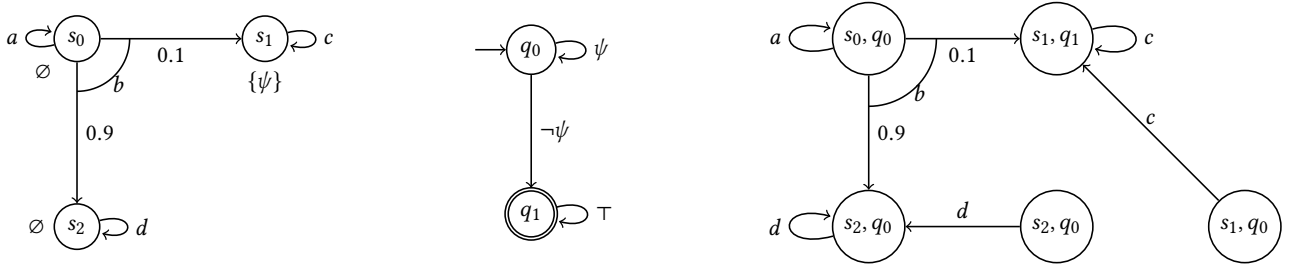

The problem arises as follows. Assume that, for the sake of argument, the Q values are all initialised to 0. This makes no real difference to this example, and in fact has been suggested in earlier work \cite{Hahn2020}. Now suppose that when learning in $G_B$ the agent makes several `lucky' transitions from $(s_0,q_0)$ to $(s_1,q_1)$ again and again in the initial trajectories (which terminate immediately afterwards due to reaching a sink state), updating $Q((s_0,q_0),b)$ to some value $d > 0.1c$. Letting the value of $Q((s_0,q_0),b)$ at time $t$ be given by $q_t$, then each of these updates is described by the rule:
$$q_{t+1} \leftarrow (1-\alpha_t)q_t + \alpha_t[1 + \gamma_V q_t]$$
It is straightforward to see that $q_t$ may quickly exceed $0.1c$. For example, starting from $t = 0$ and setting $\gamma_V = 0.9$ and $\alpha_t = \frac{1}{t + 1}$ then $q_2 = 1.45 > 0.1c = 1$. Now suppose that, starting a new trajectory from $(s_0,q_0)$, the agent repeatedly performs action $a$ and returns to $s_0,q_0$ for some number of timesteps, until $Q((s_0,q_0),a) > 0.1c$. For example, with the setting of the parameters above, then after five updates we reach $Q((s_0,q_0),a) = 1.036$. The problem now is that it will \textit{never} be possible to decrease the value of $Q((s_0,q_0),a)$, due to the option of updating this value in terms of itself:
\begin{align*}
    Q((s_0,q_0),a) &\leftarrow (1-\alpha_t)Q((s_0,q_0),a) + \alpha_t[R(s_0,q_0) + \gamma(s_0,q_0) \max_{a'} Q((s_0,q_0), a')]\\
    &= (1-\alpha_t)Q((s_0,q_0),a) + \alpha_t\max_{a'} Q((s_0,q_0), {a'})\\
    &\geq (1-\alpha_t)Q((s_0,q_0),a) + \alpha_t Q((s_0,q_0), a)\\
    &= Q((s_0,q_0),a)
\end{align*}
However, as the agent will never gain reward after reaching state $(s_2,q_0)$ then we will always have $Q((s_2,q_0),d) = 0$. Hence we can expect that, as $t$ tends to infinity, $Q((s_0,q_0),b)$ will eventually reach its true value of $0.1c$. This can be seen by observing that in this MG, the updates to the Q-values of this particular state-action pair would be identical if using a constant discount factor of $\gamma_V$. But then, once converged, we will have $Q((s_0,q_0), a) > Q((s_0,q_0), b)$. Not only are these Q values incorrect, when extracting a deterministic optimal policy via the standard approach of taking $\pi(s) = \argmax_{a'} Q(s,a')$ then the resulting policy will be incorrect with respect to the original specification $\ltlF\psi$.

\end{document}